\documentclass[11pt, letterpaper]{article}
\pdfoutput=1
 \usepackage{color,amsmath,latexsym,amssymb}
 \usepackage{bbm}   
 \usepackage[]{graphicx}
\usepackage{algorithm}%, algpseudocode}
\usepackage{enumitem}
\usepackage{algorithmic}

 \usepackage{geometry}
 \usepackage{mathtools}
 \usepackage[hidelinks]{hyperref}
 \usepackage{breakcites}
 \usepackage{siunitx}
 \usepackage[labelformat=simple]{subcaption}

\newcommand{\bpm}{\begin{pmatrix}}
\newcommand{\epm}{\end{pmatrix}}

\newcount\colveccount
\newcommand*\colvec[1]{
        \global\colveccount#1
        \begin{bmatrix}
        \colvecnext
}
\def\colvecnext#1{
        #1
        \global\advance\colveccount-1
        \ifnum\colveccount>0
                \\
                \expandafter\colvecnext
        \else
                \end{bmatrix}
        \fi
}

\newcommand {\real} {\mathbb{R}}

\newcommand {\cB} {\mathcal{B}}

\newcommand {\cD} {\mathcal{D}}

\newcommand {\cH} {\mathcal{H}}
\newcommand {\cI} {\mathcal{I}}

\newcommand {\cN} {\mathcal{N}}
\newcommand {\cO} {\mathcal{O}}

\newcommand {\cU} {\mathcal{U}}

\newcommand{\BC}{\ensuremath{\mathbf{C}} } %
\DeclareMathOperator*{\V}{\mathbb{V}}
\DeclareMathOperator*{\argmax}{arg\,max}%

\newcommand{\ba}{\ensuremath{\mathbf{a}}} %
\newcommand{\bb}{\ensuremath{\mathbf{b}}} %
\newcommand{\bc}{\ensuremath{\mathbf{c}}} %
\newcommand{\be}{\ensuremath{\mathbf{e}}} %
\newcommand{\bff}{\ensuremath{\mathbf{f}}} %
\newcommand{\bm}{\ensuremath{\mathbf{m}}} %
\newcommand{\bv}{\ensuremath{\mathbf{v}}} %
\newcommand{\bx}{\ensuremath{\mathbf{x}}} %
\newcommand{\by}{\ensuremath{\mathbf{y}}} %
\newcommand{\bz}{\ensuremath{\mathbf{z}}} %
\newcommand{\bzero}{\ensuremath{\mathbf{0}}} %

\newcommand {\BLambda} {\mbox{\boldmath $\Lambda$}}

\newcommand {\BTheta} {\boldsymbol{\Theta}}
\newcommand {\bpsi} {\mbox{\boldmath $\psi$}}

\newcommand {\btheta} {\mbox{\boldmath $\theta$}}

\newcommand {\BSigma} {\mbox{\boldmath $\Sigma$}}

 \newtheorem{theorem}{Theorem}[section]
 
 \newtheorem{proposition}[theorem]{Proposition}

 \newtheorem{definition}[theorem]{Definition}

 \newenvironment{proof}%
                {\noindent {\bf Proof:} }%
                {\hfill $\Box$ \\[1ex] }

               {\refstepcounter{theorem}\vspace{2ex}%
                \noindent{\bf Examples \thetheorem} }%
               {\hfill $\diamond$}

               {\refstepcounter{theorem}\vspace{2ex}%
                \noindent{\bf Example \thetheorem} }%
               {\hfill $\diamond$}

\newcounter{problem} %[section]
\renewcommand{\theproblem} {\arabic{problem}}
               {\refstepcounter{problem} \vspace{2ex}%
                \noindent{\bf Problem \theproblem} }%
               { }

               {\noindent{\bf Solution}}%
               {\hfill $\Box$ \\[1ex]}

\newlength{\boxwidth}

\newlength{\fullboxwidth}
\newlength{\fullinboxwidth}
\ifpdf
  \DeclareGraphicsExtensions{.eps,.pdf,.png,.jpg}
\else
  \DeclareGraphicsExtensions{.eps}
\fi

\newcommand{\TheTitle}{Adaptive Gaussian process surrogates for Bayesian inference} 
\newcommand{\TheAuthors}{T.\ Takhtaganov, J.\ M\"uller}

\ifpdf
\hypersetup{
  pdftitle={\TheTitle},
  pdfauthor={\TheAuthors}
}
\fi

\title{\TheTitle
	\thanks{This work was supported by the Lawrence Berkeley National Laboratory LDRD program under contract number DE-AC02005CH11231.}}

\newcommand{\email}[1]{\protect\href{mailto:#1}{#1}}

\author{T. Takhtaganov\thanks{Center for Computational Sciences and Engineering, Computational Research Division,
Lawrence Berkeley National Laboratory, Berkeley, CA 94720-8139, USA (\email{tat@lbl.gov}, \email{JulianeMueller@lbl.gov}, \url{https://ccse.lbl.gov/}).}
\and J. M\"uller\footnotemark[2]}

\begin{document}

\maketitle

\graphicspath{{./figures/}}

\begin{abstract}
We present an adaptive approach to the construction of Gaussian process surrogates for Bayesian inference with expensive-to-evaluate forward models. Our method relies on the fully Bayesian approach to training Gaussian process models and utilizes the expected improvement idea from Bayesian global optimization. We adaptively construct training designs by maximizing the \textit{expected improvement in fit} of the Gaussian process model to the noisy observational data. Numerical experiments on model problems with synthetic data demonstrate the effectiveness of the obtained adaptive designs compared to the fixed non-adaptive designs in terms of accurate posterior estimation at a fraction of the cost of inference with forward models.

% To demonstrate the effectiveness of the method, we apply it to model problems with synthetic data and contrast the obtained results with those obtained by using fixed non-adaptive designs. {\color{blue} can you add one sentence that gives the main takeaway from the results?}
\end{abstract}

\begin{keywords}
Bayesian inference, surrogate models, Gaussian process models, experimental designs, adaptive sampling, epistemic uncertainty, computationally expensive problems
% {\color{blue} anything about computationally expensive?}{\color{red} I don't know, we could include it but we don't really use expensive models here, and besides surrogate models term is likely to come up in a search for expensive models}{\color{blue} would you use this method for computationally cheap problems? -- what about multi-output?}
\end{keywords}

%\begin{AMS}
%62F15, 60G15, 62G08, 62K20, 62K86
%\end{AMS}

%!TEX root = mainAdaptiveGP.tex

\section{Introduction}

Computer simulations are used in many science domains to study complex physical phenomena for which targeted experiments to test hypotheses would otherwise be too time consuming or impossible to conduct. The simulation models generally have parameters whose values influence the output of the model. Experimental data are often used in order to assess the accuracy of the simulation model. However, these observations are often noisy, leading to an inverse problem to arise. 

Our work is concerned with the design of efficient tools for the solution of inverse problems encountered in science domains such as engineering, cosmology, or combustion. The goal of inference is the estimation of the parameters of interest that serve as inputs into the computational model from a set of observations. We consider applications in which the computational model is expensive (several minutes to hours per run on a modern supercomputer) and the observational data are noisy. The Bayesian approach provides a statistical framework for solving inverse problems with noisy and incomplete data. The solution of the inverse problem in the Bayesian framework is a posterior distribution that describes the degree of confidence about the parameters of interest. Typically, this distribution does not have an analytical form and is represented by the samples obtained with posterior sampling approaches such as Markov chain Monte Carlo (MCMC). These methods commonly require repeated evaluations of the forward model, which, in our setting, quickly becomes prohibitive from a computational point of view. Thus, our main challenge is to reduce the number of model evaluations that are required to find posterior distributions of the model's parameters. In this paper, we address this challenge by employing Gaussian process models as surrogates of the computationally expensive model.
%
%!TEX root = mainAdaptiveGP.tex

%\section{Literature review}

Approximate methods (also surrogates or metamodels) have been employed by many authors to accelerate inference tasks. Among the approximation methods used in the context of inference are projection-based model reduction \cite{NCNguyen_GRozza_DBPHuynh_ATPatera_2011a, CuiT_YMMarzouk_KEWillcox_2015a}, stochastic spectral methods \cite{YMMarzouk_HNNajm_LARahn_2007a, YMMarzouk_HNNajm_2009a}, and Gaussian process regression \cite{MCKennedy_AOHagan_2001a, SHabib_KHeitmann_DHigdon_CNakhleh_BWilliams_2007a, DHigdon_JGattiker_BWilliams_MRightley_2008a}. Typically, the surrogate model is built over the support of the prior distribution on the parameters of interest making it a ``global'' approximation. As argued in \cite{JLi_YMMarzouk_2014a}, it can be sufficient to have a ``localized'' surrogate---the one that is accurate only in the region of the posterior measure concentration. In cases where the prior distribution is ``broad'' and the posterior is highly concentrated, the localized surrogate approach can lead to a significant reduction in the number of forward model evaluations that are required to obtain it.

In \cite{JLi_YMMarzouk_2014a}, the authors perform Bayesian inference using Polynomial Chaos (PC) surrogates that are adaptively constructed over probability distributions chosen to approximate the posterior in the sense of Kullback--Leibler (K-L) divergence. Candidate distributions are chosen from a parameterized family by minimizing the approximate K-L divergence, and the localized PC surrogates are built with respect to the chosen distributions.
%and the K-L divergence is approximated using localized PC surrogates built with respect to the chosen distribution.
%The parameters for the next step distribution are obtained by maximizing the approximate K-L divergence with analytically computed derivatives. 

%{\color{blue}
In this paper, we pursue a similar idea of localized surrogates but with Gaussian process (GP) models as surrogates instead of PC. We treat the GP surrogate as a Bayesian surrogate described by a predictive distribution that encodes the available information from a limited number of forward model evaluations. We make use of the fully Bayesian formalism developed in \cite{IBilionis_NZabaras_2014a} to account for the uncertainty in the observational data as well as epistemic uncertainty arising from a limited number of simulations.

Our main contribution is the design of an adaptive algorithm to guide the selection of training inputs for the construction of the GP surrogate. Our algorithm aims at building a GP surrogate that is effective for the purpose of solving a specific inverse problem. In each step of the algorithm, we maximize an acquisition function that quantifies a potential improvement in the fit of the GP model to the observational data. This greedy approach explores the prior distribution of the parameters sufficiently in order to inform the surrogate model globally while emphasizing the regions that are most likely to be of interest for the construction of the posterior. 

The acquisition function in our algorithm is commonly used in Bayesian global optimization under the name of expected improvement. In that context, it is used within an algorithm called Efficient Global Optimization (EGO) \cite{DRJones_MSchonlau_WJWelch_1998a} to find the global optimum of an expensive-to-evaluate function. We, however, do not merely apply EGO to our objective. The difference between our approach of employing the expected improvement function and the EGO approach is explained in Section \ref{sec:difference_with_EGO}.%}

Previous work that applied Bayesian optimization and EGO to the solution of inverse problems is reported in \cite{PPerdikaris_GEKarniadakis_2016a}. The authors of this work apply the EGO algorithm directly to minimize the error between the model and the experimental data. However, the authors do not treat the problem in the fully Bayesian setting as we do here. A recent work \cite{MSinsbeck_WNowak_2017a} also considers a sequential design strategy for the solution of inverse problems with GP emulators. In this work, however, the forward model is assumed to be a realization of the GP model which is known completely (i.e., with fixed hyperparameters of the covariance function, see Section~\ref{sec:GPmodels}). We do not make such assumptions. Furthermore, our choice of acquisition function leads to a more tractable auxiliary problem.

Active learning and Bayesian optimization methods have also been  applied to the related problem of estimating the likelihood functions for Bayesian inference. Gaussian process models have been applied to directly approximate the likelihood, for example, in \cite{KKandasamy_JSchneider_BPoczos_2017a, HWang_JLi_2017a}. In these works, the training points for the GP model are chosen adaptively based on a measure of uncertainty such as predictive variance or entropy. A similar approach is developed in \cite{TBuiThanh_OGhattas_DHigdon_2012a} where the authors propose a method for approximating high-dimensional expensive-to-evaluate probability density functions (p.d.f.`s) with adaptive Gaussian approximations. The p.d.f.`s of interest are the ones arising from the Bayesian solution of inverse problems with Gaussian priors and likelihoods.
% The authors consider a Gaussian prior and likelihood, which lead to the posterior being an exponential function.
The proposed method is based on Gaussian processes with covariance functions utilizing the Hessian of the negative log-likelihood of the posterior density.
% non-stationary covariance functions; specifically, the authors construct the covariance function by utilizing the Hessian of the negative log-likelihood of the posterior density. 
% Their choice of covariance function is motivated by the Laplace approximation and leads to a predictor that has a form of a piecewise Laplace approximation. 
The training points are selected adaptively by maximizing the squared error between the true posterior density and the GP-based predictor. 
% Far from the training points, the mean of the GP predictor is close to the prior mean, which is zero. Therefore, the maximization problem becomes that of finding the modes of the posterior density. In order to render the maximization problem well-defined, the authors suggest tweaking the prior parameter distribution such that the inverse problem is well-posed. 
This approach requires derivatives of the forward model.
% {\color{blue} this paragraph can probably be shortened}

%{\color{blue}
Our approach differs from the methods described above in that we build the surrogate model for the forward simulation rather than for the likelihood function. While approximating the likelihood function directly might prove advantageous if it is multi-modal and the forward model is highly nonlinear, we argue that having a surrogate of the forward model has its own merits. In particular, once the surrogate is built, it can be used not only for estimating the parameter posterior but also for forward uncertainty propagation and prediction, albeit in a limited way due to the localized nature of the surrogate. %}
% {\color{blue} does this last sentence apply in general/globally? arent there caveats that the global approximation may suffer from our sampling approach here?}{\color{red} yes, globally it will not be great, but having uncertainty bounds gives us some idea how accurate our surrogate is globally}{\color{blue} should we mention something along these lines? ie global fit we dont care so much about, but when its about parameter recovery for explaining some complicated physics parameters, knowing what goes on locally around those most likely parameters and obtaining an idea of uncertainties is what we want.}

Regarding the choice of Gaussian process models as surrogates, we refer to the recent comparison of surrogate-based uncertainty quantification methods conducted in \cite{NEOwen_PChallenor_PPMenon_SBennani_2017a}. Gaussian process models have several advantages over polynomial chaos, for example, in terms of  their flexibility and the freedom in the choice of the design. Gaussian process models also prove to be more suitable for modeling nonlinear simulator behavior, and provide estimates of the prediction uncertainty. This last feature is particularly important for the method developed here.

The remainder of this paper is organized as follows. In Section~\ref{sec:BayesianInference}, we provide the definition of the inverse problem and a brief summary of Bayesian inference. In Section~\ref{sec:GPmodels}, we review Gaussian process models for the single and the multiple output cases and Bayesian inference with GP models. This section provides a review of the existing methodology for training GP models and motivates our modeling choices. It culminates in the derivation of the GP-based likelihood function used for inference. We describe our adaptive approach to constructing the GP models in Section~\ref{sec:adaptiveGP} and we show its performance in numerical experiments in Section~\ref{sec:numericalExp}. Here, we contrast our approach with the commonly used randomized designs that are not goal-oriented, such as Latin hypercube designs \cite[Section~5.2.2]{TJSantner_BJWilliams_WINotz_2003a}. We do not compare our results to any other adaptive or sequential experimental designs that aim to reduce the predictive errors in GP regression,
% . Examples of such methods are abundant in the literature, 
see, for example, \cite{JSacks_WJWelch_TJMitchell_HPWynn_1989a, AGorodetsky_YMMarzouk_2016a} for the summaries. The objectives of such designs are different from those considered in our work. In Section~\ref{sec:conclusions}, we draw conclusions and outline future research directions.  

%
%!TEX root = mainAdaptiveGP.tex

\section{Bayesian inference}\label{sec:BayesianInference}

We start by formulating the inference problem and introducing the notation. Let the vector of parameters of interest be denoted by $\btheta\in\real^p$. These parameters serve as an input into the simulation model (i.e., the computer code) that represents a given physical system. Let $\bff:\real^p \rightarrow \real^q$ denote a mapping from the inputs to the outputs of the deterministic forward model. The components of the output vector will be denoted by $f_i(\btheta)$: $\bff(\btheta) = (f_1(\btheta),\dots,f_q(\btheta))^T$. Multiple outputs arise, for example, if the forward model depends on an additional (deterministic) variable $x$ that takes on $q$ values; in this case, each component of the output represents the value for a fixed $x_i$: $f_i(\btheta) = f(\btheta, x_i)$, $i=1,\dots,q$. For example, $x$ could represent time in time-dependent problems.
%{\color{blue} mention an example here, like taking observations at 2 different times -- would help the reader to paint a picture in their head}

The goal of inference is to learn the parameters $\btheta$ from the direct observations of the physical system. We will denote such observations (experimental data) of the output quantities by a vector $\bz = (z_1,\dots,z_q)^T$. The measured quantities are never perfect and contain measurement noise that will be denoted by a vector $\be \in\real^q$. As in classical statistical inverse problems \cite{JKaipio_ESomersal_2005a}, we will view all the variables as random and use capital letters to represent them. Lower case letters will be reserved for their realizations.

\noindent
We assume the following statistical model for the measurements with additive noise:
\[
	Z_i = f_i(\BTheta) + E_i, \quad i=1,\dots,q.
\]
We further assume that the components of the measurement noise are normally distributed ($E_i\sim\cN(0,\sigma_{i}^2)$) and potentially correlated. The probability density of the measurement noise is thus given by a $q$-variate normal density:
\[
	p_{E}(\be) = \cN_q(\be \,|\, \bzero_q, \BSigma_{E}) \coloneqq (2\pi)^{-q/2} |\BSigma_E|^{-1/2} \exp\left[ -\frac{1}{2}\be^T\BSigma_{E}^{-1}\be \right]
\]
with known covariance $\BSigma_{E}$; $\bzero_q$ here denotes a vector of $q$ zeros. 

From the assumed statistical model it follows that $Z_i$ conditioned on $\BTheta=\btheta$ is distributed like $E_i$, which leads to the following measurement likelihood function:
\begin{equation}\label{eq:likelihood_z_full_model}
	L(\btheta | \bz) \coloneqq p_{E}(\bz - \bff(\btheta)) = \cN_q(\bz-\bff(\btheta) \,|\, \bzero_q, \BSigma_{E}).
\end{equation}
Here, we use the notation $L(\btheta | \bz) = p(\bz | \btheta)$ as in \cite[Section~6.3]{ELLehmann_GCasella_1998a}, i.e., the likelihood is the density of the data considered as a function of the parameters $\btheta$ for fixed $\bz$.

\noindent
Assuming the Bayesian framework, any prior information on the parameters $\btheta$ is encoded in the prior density function $p(\btheta)$. Given the prior and the observed measurements $\bz$, the solution of the inverse problem is the posterior density obtained by applying Bayes' rule:
\begin{equation} \label{eq:bayes_rule_for_theta}
	p(\btheta | \bz) \propto L(\btheta | \bz)p(\btheta).
\end{equation}
The posterior density usually does not have a closed form solution and is explored via sampling, for example, with MCMC methods \cite{JSLiu_2001a}. Applying MCMC requires repeated evaluations of the likelihood function $L(\btheta | \bz)$. Since these evaluations involve computing the forward model $\bff(\btheta)$, they are expensive, making direct application of MCMC methods infeasible. The goal of the next sections is to develop a method for approximating the forward model $\bff(\btheta)$ with a Bayesian surrogate that  allows an efficient and accurate computation of the likelihood function $L(\btheta|\bz)$. We use a Gaussian process (GP) model as the surrogate model. Before describing the proposed method, we review the standard GP methodology in the next section.

%
%!TEX root = mainAdaptiveGP.tex

\section{Gaussian Process models}\label{sec:GPmodels}
In this section, we describe a GP model for the single and the multiple output cases followed by the Bayesian inference with GP models. 
\subsection{Single output case} We start with a one-dimensional GP model for the case of a single output $f(\btheta)$. Formally, a Gaussian process model $f^{GP}(\btheta)$ is a collection of random variables such that any finite number of them has a joint Gaussian distribution \cite[Chapter~2]{CERasmussen_CKIWilliams_2006a}. This distribution is characterized by its mean and covariance functions. In the following, we take the mean function to be zero, and we let the covariance function be the squared exponential:
\begin{equation}\label{eq:squared_exponential_covariance}
	\text{cov}\big(f^{GP}(\btheta), f^{GP}(\btheta')\big) = c(\btheta, \btheta') \coloneqq \sigma_c^2 \exp\left[ - \sum_{i=1}^p \frac{(\theta_i - \theta_i')^2}{\ell_i^2} \right].
\end{equation}
This covariance function expresses prior information about $f^{GP}(\btheta)$: it prescribes the common variance $\sigma_c^2$ to the values at different $\btheta$ and expresses correlations through the distance between inputs weighted by characteristic length-scales $\ell_i$ in each input dimension $i=1,\dots,p$. We will denote the vector of the parameters of the covariance function by $\bpsi$,
\[
	\bpsi = (\sigma_c, \ell_1, \dots, \ell_p)^T \in \real^{p+1},
\]
and refer to them as \textit{hyperparameters}. 
%Specifics of the hyperparameters $\bpsi$ are discussed later. 
We will write the covariance function in the form $c(\btheta, \btheta'; \bpsi)$ to emphasize its dependence on the hyperparameters $\bpsi$.

The choice of the zero mean function does not affect our methodology, but is assumed for convenience. In practice, it is common to model the mean using a fixed basis which leads to the introduction of additional regression parameters, see, e.g., \cite{MCKennedy_AOHagan_2001a}. The modeling choice of the covariance function $c(\cdot, \, \cdot)$ is usually more important as it encodes certain assumptions on the smoothness of $f(\btheta)$. We choose the squared exponential covariance \eqref{eq:squared_exponential_covariance} for reasons of its interpretability and widespread use. However, depending on the application problem and the underlying physical process at hand, other choices might be more appropriate, see, e.g., \cite[Section~4.2]{CERasmussen_CKIWilliams_2006a}.

Besides specifying the mean and the covariance functions, constructing a GP model requires choosing a set of input parameter values for training: $\btheta^{(j)}_{train}$, $j=1,\dots, n_{train}$. Together with the corresponding values of the forward model,  we form the training set $\cD$:
\[
	\cD \coloneqq \big\{ \btheta^{(j)}_{train}, f\big(\btheta^{(j)}_{train}\big) \big\}_{j=1}^{n_{train}}.
\]
Given the training set $\cD$ and the hyperparameters $\bpsi$, the distribution of $f^{GP}(\btheta)$ at a test input $\btheta$ is given by
\begin{equation}\label{eq:gp_pred_posterior_single}
	p\big(f^{GP} \,\big|\, \btheta, \cD, \bpsi\big) = \cN\big(f^{GP} \,\big|\, m(\btheta; \cD, \bpsi), \V(\btheta; \cD, \bpsi)\big),
\end{equation}
with the mean and the variance given by
\begin{subequations}
\begin{align}
	m(\btheta; \cD, \bpsi) &= \bc_\psi^T (\BC_\psi)^{-1} \by, \label{eq:gp_pred_mean}\\ 
	\V(\btheta; \cD, \bpsi) &= c(\btheta, \btheta; \bpsi) - \bc_\psi^T(\BC_\psi)^{-1}\bc_\psi. \label{eq:gp_pred_var}
\end{align}
\end{subequations}
Here, $\bc_\psi = \big(c\big(\btheta, \btheta_{train}^{(1)}; \bpsi\big), \dots, c\big(\btheta, \btheta_{train}^{(n_{train})}; \bpsi\big)\big)^T \in\real^{n_{train}}$ is the vector of covariances between the test input $\btheta$ and the inputs in the training set given the hyperparameters $\bpsi$, $\by = \big( f\big(\btheta_{train}^{(1)}\big), \dots, f\big(\btheta_{train}^{(n_{train})}\big) \big)^T$, and $\BC_\psi\in\real^{{n_{train}}\times{n_{train}}}$ is the matrix of covariances between the inputs in the training set given the hyperparameters $\bpsi$:
 \[
 	(\BC_\psi)_{i,j} = c\big(\btheta_{train}^{(i)}, \btheta_{train}^{(j)}; \bpsi\big),\ i, j=1,\dots, n_{train}.
 \]
To train a GP model means to prescribe the hyperparameters $\bpsi$ using the training set $\cD$. This is commonly done using the evidence framework \cite{DJCMacKay_1999a}: hyperparameters $\bpsi$ are chosen by maximizing the logarithm of the marginal likelihood (or evidence) of the training values:
\begin{equation}\label{eq:max_log_marg_psi}
	\bpsi^* = \argmax\limits_{\ensuremath{\boldsymbol{\psi}}\in\cB_\psi} \,\log L(\bpsi | \cD),
\end{equation}
where $\cB_\psi$ is a compact subset of $\real^{p+1}$, and
\begin{equation} \label{eq:evidence_single_out}
	L(\bpsi | \cD) = \cN_{n_{train}}( \by \,|\, \bzero_{n_{train}}, \BC_\psi ).
\end{equation}
%\[
%	L(\psi | \cD) = \int p\big(\by \,\big|\, f^{GP}, \theta_{train}\big) p\big(f^{GP} \,\big|\, \theta_{train}\big) df^{GP}.
%\]

In practice, the maximization of the log-marginal likelihood in \eqref{eq:max_log_marg_psi} is performed by using a multi-start strategy to avoid getting trapped in local maxima. For a small number of training inputs, the slope of the log-marginal likelihood can be very low leading to multiple  hyperparameter values being consistent with the training data. In such cases, choosing the hyperparameter values by maximizing the log-marginal likelihood can become unreliable and produce estimates with high empirical variances \cite{DGinsbourger_DDupuy_ABadea_LCarraro_ORoustant_2009a}. Furthermore, the predictive variance of the GP model \eqref{eq:gp_pred_var} with the plug-in estimator \eqref{eq:max_log_marg_psi} is known to underestimate the true mean-squared prediction error of the model \cite{DLZimmerman_NCressie_1992a}.

An alternative way of training a GP model is to adopt a fully Bayesian perspective. Instead of taking the point-estimate of the hyperparameters as in \eqref{eq:max_log_marg_psi}, we can condition the predictive distribution \eqref{eq:gp_pred_posterior_single} on the hyperparameter distribution. 
% This approach requires specifying the prior on the hyperparameters. The misspecification of the prior for the correlation parameters can lead to improper posteriors (see the warning in \cite[page 69]{TJSantner_BJWilliams_WINotz_2003a}). It has been reported, however, that, in practice, the fully Bayesian approach gives wider confidence bounds than predictors based on plug-in estimators, thus, better accounting for the uncertainty about the covariance function \cite{MSHandcock_MLStein_1993a}. In the following we adopt the fully Bayesian approach to  the GP model training \cite{IBilionis_NZabaras_2014a}. That is, we specify a prior on the hyperparameters, $p(\bpsi)$, and use the likelihood function $L(\bpsi | \cD)$  from \eqref{eq:evidence_single_out} to obtain the hyperparameter posterior using MCMC methods:
It has been reported that the fully Bayesian approach gives wider confidence bounds than predictors based on plug-in estimators, thus, better accounting for the uncertainty about the covariance function \cite{MSHandcock_MLStein_1993a}. In the following we adopt the fully Bayesian approach to  the GP model training. %\cite{IBilionis_NZabaras_2014a}.
That is, we specify a prior on the hyperparameters, $p(\bpsi)$, and use the likelihood function $L(\bpsi | \cD)$  from \eqref{eq:evidence_single_out} to obtain the hyperparameter posterior using MCMC methods:
 \[
	p(\bpsi | \cD) \propto  L(\bpsi | \cD) p(\bpsi).
\]
 As a result, we obtain an ensemble of samples of the hyperparameter vector $\bpsi$ distributed according to $p(\bpsi| \cD)$: $\{\bpsi^{(j)}\}_{j=1}^{n_\psi}$.

The predictive distribution of the GP model at a test point $\btheta$ can then be obtained by marginalizing over (integrating out) the hyperparameters:
\[
	p\big(f^{GP} \,\big|\, \btheta, \cD\big) = \int p\big(f^{GP} \,\big|\, \btheta, \cD, \bpsi\big) p(\bpsi | \cD) d\bpsi,
\]
where $p\big(f^{GP} \,\big|\, \btheta, \cD, \bpsi\big)$ is given by \eqref{eq:gp_pred_posterior_single}. Using the samples of the hyperparameter posterior computed with MCMC, this integral can be discretized as follows:
\begin{equation} \label{eq:discrete_approx_gp_posterior_mcmc}
	p\big(f^{GP} \,\big|\, \btheta,\cD\big) \approx \frac{1}{n_\psi} \sum_{j=1}^{n_{\psi}} p\big(f^{GP} \,\big|\, \btheta, \cD, \bpsi^{(j)}\big).
\end{equation}
Thus, we obtain a Gaussian mixture model of the predictive distribution. The mean of this model is simply the average of the means from \eqref{eq:gp_pred_mean} with $\bpsi = \bpsi^{(j)}$,
\begin{equation}\label{eq:mean_mixture_single_out}
	m(\btheta; \cD) = \frac{1}{n_\psi}\sum_{j=1}^{n_\psi} m\big(\btheta; \cD, \bpsi^{(j)}\big),
\end{equation}
and the variance can be obtained using $\V(\btheta; \cD, \bpsi)$ with $\bpsi=\bpsi^{(j)}$ from \eqref{eq:gp_pred_var} as follows:
\begin{align}
\begin{split}\label{eq:var_mixture_single_out}
	\V(\btheta; \cD) &= \frac{1}{n_\psi} \sum_{j=1}^{n_\psi} \V\big(\btheta; \cD, \bpsi^{(j)}\big) \\
	&+ \frac{1}{n_\psi}\sum_{j=1}^{n_\psi} \big( m\big(\btheta; \cD, \bpsi^{(j)} \big) \big)^2 - \bigg( \frac{1}{n_\psi}\sum_{j=1}^{n_\psi} m\big(\btheta; \cD, \bpsi^{(j)}\big)  \bigg)^2.
\end{split}
\end{align}

\subsection{Multiple output case} Gaussian process models for the multi-output case $q > 1$ are discussed in detail in \cite{SConti_AOHagan_2010a}, where the authors employ a matrix-normal distribution for the training data to account for possible correlations between the outputs. 
%A similar approach is taken in \cite{IBilionis_NZabars_BAKonomi_GLin_2013a}.
In our work, we take a simplified approach from \cite{IBilionis_NZabaras_2012a} and \cite{IBilionis_NZabaras_2014a} that assumes that the outputs are conditionally independent given the covariance function. This treatment of the multi-output case is similar to \cite{SConti_AOHagan_2010a} if a diagonal correlation matrix and a constant mean are used, and is based on an assumption that the regularity of the outputs is approximately the same. In order to account for potentially different scales of the outputs, we normalize the training outputs as described below.

Let $f_i(\btheta)$ denote the $i$-th output of the forward model, $i=1,\dots, q$. Now, for the same set of training inputs $\btheta_{train}$, we have $q$ sets of output values which we will denote by $\cD_i$:
\[
	\cD_i = \big\{ \btheta_{train}^{(j)}, f_i\big(\btheta_{train}^{(j)}\big) \big\}_{j=1}^{n_{train}}, \quad i=1,\dots, q.
\]
We will write $\cD=\cup_{i=1}^{\,q} \cD_i$.
Assuming conditionally independent outputs, the marginal likelihood of the training outputs becomes
 \begin{equation}\label{eq:evidence_mult_out}
 	L(\bpsi | \cD) = \prod_{i=1}^q L(\bpsi | \cD_i ),
 \end{equation}
 with the one-dimensional likelihoods given by
 \[
 	L(\bpsi | \cD_i) = \cN_{n_{train}}(\by_i \,|\, \bzero_{n_{train}}, \BC_\psi ).
 \]
 Here, the training outputs $\by_i$ represent scaled responses:
 \begin{equation}\label{eq:scaled_outputs_mult_case}
 	\by_i = \big( \widehat{f}_i\big(\btheta_{train}^{(1)}), \dots, \widehat{f}_i\big(\btheta_{train}^{(n_{train})}\big) \big)^T
 \end{equation}
 with
 \[
 	\widehat{f}_i(\btheta) = \frac{f_i(\btheta) - m_i}{{\V_i}^{1/2}},
 \]
 where
 \[
 	m_i = \frac{1}{n_{train}}\sum_{j=1}^{n_{train}}f_i\big(\btheta_{train}^{(j)}\big), \quad \V\nolimits_i = \frac{1}{n_{train}} \sum_{j=1}^{n_{train}} \big(f_i\big(\btheta_{train}^{(j)}\big) - m_i\big)^2.
 \]
Similarly to the single-output case, the likelihood \eqref{eq:evidence_mult_out} is used to obtain the samples of the hyperparameter posterior $p(\bpsi | \cD)$: $\{\bpsi^{(j)}\}_{j=1}^{n_\psi}$. Under the standing assumption of conditional independence, 
\[
	p\big(\bff^{GP} \,\big|\, \btheta, \cD, \bpsi\big) = \prod_{i=1}^q p\big(f_i^{GP} \,\big|\, \btheta, \cD_i, \bpsi\big),
\]
which leads to the following predictive density of the combined output vector
\begin{align}\label{eq:pred_density_GP_vector}
\begin{split}
	p\big(\bff^{GP} \,\big|\, \btheta, \cD\big) &= \int p\big(\bff^{GP} \,\big|\, \btheta, \cD, \bpsi\big) p(\bpsi | \cD) d\bpsi \\
	&\approx \frac{1}{n_\psi} \sum_{j=1}^{n_\psi} \prod_{i=1}^q p\big(f_i^{GP} \,\big|\, \btheta, \cD_i, \bpsi^{(j)}\big).
\end{split}
\end{align}
In a compact form, \eqref{eq:pred_density_GP_vector} can be written as a mixture of $q$-variate Gaussians:
\begin{equation} \label{eq:pred_density_GP_vector_compact}
	p\big(\bff^{GP} \,\big|\, \btheta, \cD\big) \approx \frac{1}{n_\psi} \sum_{j=1}^{n_\psi} \cN_{q}\bigg( \bff^{GP} \,\big|\, \bm\big(\btheta; \cD, \bpsi^{(j)}\big), \BSigma_{GP}\big(\btheta; \cD, \bpsi^{(j)}\big)  \bigg),
\end{equation}
where for each $\bpsi=\bpsi^{(j)}$, $j=1,\dots, n_\psi$, (applying the necessary re-scaling)
\begin{equation}\label{eq:gp_vec_pred_mean}
	\bm(\btheta; \cD, \bpsi) = \big( {\V\nolimits_1}^{1/2}\cdot m(\btheta; \cD_1, \bpsi) + m_1, \dots, {\V\nolimits_q}^{1/2}\cdot m(\btheta; \cD_q, \bpsi) + m_q \big)^T,
\end{equation}
with $m(\btheta; \cD_i, \bpsi)$ as in \eqref{eq:gp_pred_mean}, and
% \begin{equation}\label{eq:gp_vec_pred_cov}
% 	\BSigma_{GP}(\btheta; \cD, \bpsi) = \text{diag}\big[ \V\nolimits_1\cdot\V(\btheta; \cD_1, \bpsi), \dots, \V\nolimits_q\cdot\V(\btheta; \cD_q, \bpsi) \big] \in \real^{q\times q}.
% \end{equation}
% Since the variances $\V(\btheta; \cD_i, \bpsi)$ do not depend on the output values $\by_i$ (see \eqref{eq:gp_pred_var}), we can write:
\begin{equation}\label{eq:gp_vec_pred_cov}
	\BSigma_{GP}(\btheta; \cD, \bpsi) = \V(\btheta; \cD, \bpsi)\cdot\text{diag}\big[ \V\nolimits_1, \dots, \V\nolimits_q \big]\in \real^{q\times q}
\end{equation}
with $\V(\btheta; \cD, \bpsi)$ as in \eqref{eq:gp_pred_var}.
The mean of the mixture distribution \eqref{eq:pred_density_GP_vector_compact} is given by
\begin{equation}\label{eq:mean_mixture_vector}
	\bm(\btheta; \cD) = \frac{1}{n_\psi} \sum_{j=1}^{n_\psi} \bm\big(\btheta; \cD, \bpsi^{(j)}\big),
\end{equation}
and the covariance is given by
\begin{align*}
	\BSigma_{GP}(\btheta; \cD) &= \frac{1}{n_\psi} \sum_{j=1}^{n_\psi} \BSigma_{GP}\big(\btheta; \cD, \bpsi^{(j)}\big) + \frac{1}{n_\psi} \sum_{j=1}^{n_\psi} \bm\big(\btheta; \cD, \bpsi^{(j)}\big) \bm\big(\btheta; \cD, \bpsi^{(j)}\big)^T \\
	&- \bigg(\frac{1}{n_\psi}\sum_{j=1}^{n_\psi} \bm\big(\btheta; \cD, \bpsi^{(j)}\big)\bigg)\bigg(\frac{1}{n_\psi}\sum_{k=1}^{n_\psi} \bm\big(\btheta; \cD, \bpsi^{(k)}\big)\bigg)^T.
\end{align*}

%%%%%%%%%%%%%%%%%%%%%%%%%%%%%%%%%%%%%%%%%%%%%%%%%%%%%%%%%%%%%%%%%%%%%%%%

\subsection{Bayesian inference with GP models}

Now, we re-formulate problem \eqref{eq:bayes_rule_for_theta} using the multi-output GP surrogate $\bff^{GP}$ with the predictive distribution given by \eqref{eq:pred_density_GP_vector_compact}. While the simplest approach would be to substitute $\bff(\btheta)$ in the likelihood definition $L(\btheta | \bz)$ in \eqref{eq:likelihood_z_full_model} with the mean vector in \eqref{eq:mean_mixture_vector}, such an approach would ignore the uncertainty of the surrogate. The availability of uncertainty estimates is the strength of the GP model and should therefore be exploited. Hence, we follow the approach in \cite{IBilionis_NZabaras_2014a} and substitute $L(\btheta | \bz)$ in \eqref{eq:likelihood_z_full_model} with $L(\btheta | \bz, \cD)$---the so-called \textit{$\cD$-restricted likelihood} function defined as follows: 
\begin{equation}\label{eq:D_restricted_likelihood}
	L(\btheta | \bz, \cD) \coloneqq \int L\big(\btheta \big| \bz, \bff^{GP}\big) p\big(\bff^{GP} \big| \btheta, \cD\big) d\bff^{GP},
\end{equation}
where $L\big(\btheta \big| \bz, \bff^{GP}\big)$ is the likelihood from \eqref{eq:likelihood_z_full_model} evaluated with $\bff^{GP}(\btheta)$ instead of $\bff(\btheta)$:
\begin{equation}\label{eq:likelihood_z_gp_model}
	L\big(\btheta \big| \bz, \bff^{GP}\big) \coloneqq \cN_q\big(\bz-\bff^{GP}(\btheta) \big|\, \bzero_q, \BSigma_{E}\big).
\end{equation}
Next, we plug in the mixture approximation of $p\big(\bff^{GP} \big| \btheta, \cD\big)$ from \eqref{eq:pred_density_GP_vector_compact} and the likelihood $L\big(\btheta \big| \bz, \bff^{GP}\big)$ from \eqref{eq:likelihood_z_gp_model} into \eqref{eq:D_restricted_likelihood} and integrate the product of the two Gaussians:
\begin{align}\label{eq:D_restricted_likelihood_approx}
\begin{split}
	L&(\btheta | \bz, \cD)  \\
	&\approx \frac{1}{n_\psi} \sum_{j=1}^{n_\psi} \int \cN_q\big(\bff^{GP}\big|\, \bz, \BSigma_{E}\big) \cN_q\bigg( \bff^{GP} \big|\, \bm\big(\btheta; \cD, \bpsi^{(j)}\big), \BSigma_{GP}\big(\btheta; \cD, \bpsi^{(j)}\big)  \bigg) d\bff^{GP} \\
	&= \sum_{j=1}^{n_\psi} \frac{k^{(j)}}{n_\psi} \exp\bigg[ - \frac{(\bz - \bm(\btheta; \cD, \bpsi))^T (\BSigma_{E} + \BSigma_{GP}(\btheta; \cD, \bpsi) )^{-1} (\bz - \bm(\btheta; \cD, \bpsi) )}{2} \bigg]%_{\bpsi=\bpsi^{(j)}} 
	\end{split}
\end{align}
with $\bpsi = \bpsi^{(j)}$ and $k^{(j)}=(2\pi)^{-q/2}\big|\BSigma_{E} + \BSigma_{GP}\big(\btheta; \cD, \bpsi^{(j)}\big)\big|^{-1/2}.$
The obtained $\cD$-restricted likelihood function incorporates both the measurement errors and the uncertainty of the GP model. Once the posterior samples of the hyperparameters $\{\bpsi^{(j)}\}_{j=1}^{n_\psi}$ are obtained, the approximation of $L(\btheta | \bz, \cD)$ in \eqref{eq:D_restricted_likelihood_approx} at a given test input $\btheta$ requires $\cO(n_{train}^3)$ operations for each $\bpsi^{(j)}$ for computing the predictive means and covariances \eqref{eq:gp_vec_pred_mean}--\eqref{eq:gp_vec_pred_cov}, and $\cO(q^3)$ for inverting the sum of the noise and the GP covariances, bringing the total cost to $\cO(n_\psi\cdot(n_{train}^3 + q^3))$ operations. In the special case of uncorrelated measurement noise and with the assumption of conditional independence of the outputs that we make, the cost of inverting $(\BSigma_{E}+\BSigma_{GP})$ becomes $\cO(q)$ instead of $\cO(q^3)$. The dominant cost then becomes that of computing the predictive means and the variances. In our target  applications, this cost is negligible since the number of training inputs $n_{train}$ is small and the cost of the forward model evaluation is large. In the numerical examples, we consider $\cO(10)$ training inputs. In general, the choice of $n_{train}$ is motivated by design considerations and may depend on the smoothness of the forward model mapping and on the dimension $p$ of the input space. For cases in which  the number of training inputs is large, various approximations of the covariance matrix $\BC_\psi$ can be considered, see, for example \cite[Chapter~8]{CERasmussen_CKIWilliams_2006a}. 

%As GP model gets more accurate, the likelihood $p(\bz|\theta, \cD)$ ``converges'' to the true likelihood $p(\bz | \theta, \bff(\theta))$.

Analogously to \eqref{eq:bayes_rule_for_theta}, the $\cD$-restricted likelihood $L(\btheta | \bz, \cD)$ leads to the \textit{$\cD$-restricted posterior} $p(\btheta | \bz, \cD)$:
\begin{equation} \label{eq:bayes_rule_for_theta_GP}
	p(\btheta | \bz, \cD) \propto L(\btheta | \bz, \cD)p(\btheta).
\end{equation}

Next, we analyze the approximation of $L(\btheta|\bz, \cD)$ in \eqref{eq:D_restricted_likelihood_approx} and develop a sequential adaptive strategy for selecting training inputs based on the current data $\cD$.

%
%!TEX root = mainAdaptiveGP.tex

\section{Adaptive construction of GP models}\label{sec:adaptiveGP}

Given the current training set $\cD$, it is of interest  how to select additional training inputs in order to make the GP-based likelihood $L(\btheta | \bz, \cD)$ more accurately represent the unattainable (due to its cost) ``true'' likelihood $L(\btheta | \bz)$. In particular, we would like $L(\btheta | \bz, \cD)$ to correctly capture the modes of the true likelihood $L(\btheta | \bz)$. Therefore, we attempt to find the minima of the ``true'' misfit function
\begin{equation}\label{eq:misfit_function_true}
	g(\btheta) \coloneqq (\bz - \bff(\btheta))^T \BSigma_E^{-1}(\bz - \bff(\btheta)).
\end{equation}
We start by defining the misfit function of the $\cD$-restricted likelihood \eqref{eq:D_restricted_likelihood_approx}:
\begin{equation}\label{eq:misfit_function_GP}
	g(\btheta; \cD, \bpsi) \coloneqq (\bz - \bm(\btheta; \cD, \bpsi))^T (\BSigma_{E} + \BSigma_{GP}(\btheta; \cD, \bpsi) )^{-1} (\bz - \bm(\btheta; \cD, \bpsi) ).
\end{equation}
With this definition we can re-write \eqref{eq:D_restricted_likelihood_approx} as
\[
	L(\btheta |\bz, \cD) \approx \sum_{j=1}^{n_\psi}\frac{k^{(j)}}{n_\psi} \exp\left[ - \frac{1}{2} g\big(\btheta; \cD, \bpsi^{(j)}\big) \right].
\]
The important properties of $g(\btheta; \cD, \bpsi)$ are summarized in the following proposition.

%Our $\cD$-restricted likelihood has the form of a mixture of Gaussians. In general, for the case of nonisotropic, different component covariance matrices the mixture can have more modes than components \cite{MACarreira_CKIWilliams_2003a}. In our case, due to $\Sigma_E$ we can have potentially nonisotropic covariances, and due to $\Sigma_{GP}(\theta; \cD, \psi)$ -- heteroscedasticity (different variances for different components). We restrict our attention only to the modes of the mixture that are the modes of the components. Since our goal is to accurately emulate the true likelihood function, we want to make sure that the modes of the mixture components correspond to the true likelihood modes.

\begin{proposition}\label{prop:properties_of_GP_misfit}
The misfit function $g(\btheta; \cD, \bpsi)$ in \eqref{eq:misfit_function_GP} has the following properties:
\begin{enumerate}[label=\alph*)]
\item it interpolates the true misfit function $g(\btheta)$ at the inputs in the training set $\cD$;
\item it is continuously differentiable with respect to $\btheta$.
\end{enumerate}
\end{proposition}
\begin{proof}
\noindent
\begin{enumerate}[label=\alph*)]
\item Observe that for each output $i$, $m\big(\btheta^{(j)}_{train}; \cD_i, \bpsi\big) = \widehat{f}_i\big(\btheta^{(j)}_{train}\big)$ and $\V\big(\btheta^{(j)}_{train}; \cD, \bpsi\big)=0$  for $j=1,\dots,n_{train}$ due to the interpolative properties of the GP model (see \cite[Section~4.1]{TJSantner_BJWilliams_WINotz_2003a}). Then, from \eqref{eq:gp_vec_pred_mean} and \eqref{eq:gp_vec_pred_cov} we have $\bm\big(\btheta^{(j)}_{train}; \cD, \bpsi\big)=\bff\big(\btheta^{(j)}_{train}\big)$ and $\BSigma_{GP}\big(\btheta^{(j)}_{train}; \cD, \bpsi\big) \equiv \bzero_{q\times q}$. Hence,
\[
	g\big(\btheta^{(j)}_{train}; \cD, \bpsi\big) = g\big(\btheta^{(j)}_{train}\big), \quad \text{for } j=1,\dots,n_{train}, \text{ and all } \bpsi.
\]
\item The predictive mean and variance of a Gaussian process inherit their smoothness properties from the underlying covariance function $c(\btheta, \btheta')$. The squared exponential function \eqref{eq:squared_exponential_covariance} considered here is in fact infinitely differentiable.
\end{enumerate}
\end{proof}
Outside of the training set $\cD$, $g(\btheta; \cD, \bpsi)$ provides an estimate of the misfit between the GP model with the hyperparameter vector $\bpsi$ and the measurement data $\bz$.
By treating $g(\btheta; \cD, \bpsi)$ as a random function of $\bpsi$ for a given test input $\btheta$ with distribution induced by $p(\bpsi | \cD)$,  we can explore the minima of $g(\btheta)$ using an auxiliary ``acquisition function''. Specifically, we employ the expected improvement idea from Bayesian optimization \cite{DRJones_MSchonlau_WJWelch_1998a}. 

%\noindent
Denote the best (i.e., the smallest) misfit value for the points in the training set as
\begin{equation}\label{eq:g_min_value}
	g_{min} \coloneqq \min\big\{ g\big(\btheta^{(j)}_{train}\big) \,\big|\, j=1,\dots, n_{train} \big\}.
\end{equation}
Consider the following problem:
\begin{equation}\label{eq:expected_improvement_max_problem}
	\max\limits_{\ensuremath{\boldsymbol{\theta}}\in\cB_\theta} \cI(\btheta) \coloneqq \frac{1}{n_\psi} \sum_{j=1}^{n_\psi} \left[ g_{min} - g\big(\btheta; \cD, \bpsi^{(j)}\big) \right]^+,
\end{equation}
where $[\,\cdot\,]^+$ denotes the positive part function, $[\,\cdot\,]^+\coloneqq\max\{ \,\cdot\,, 0 \}$, and $\cB_\theta$ is a closed and bounded subset of $\real^p$.
The idea behind formulation \eqref{eq:expected_improvement_max_problem} is to find the input that offers the largest \textit{expected improvement in the fit} to the measurement data under the mixture GP model conditioned on the misfit being smaller than the current best  true misfit value. More specifically, consider the following \textit{relative improvement in fit} function:
\[
	\cI_{rel}(\btheta; \bpsi) \coloneqq \bigg[1 - \frac{g(\btheta; \cD, \bpsi)}{g_{min}}\bigg]^+.
\]
For fixed $\bpsi$ and $\btheta$, if $g(\btheta; \cD, \bpsi) \geq g_{min}$, the improvement in fit is zero---the misfit between the GP model corresponding to $\bpsi$ and the measurement data $\bz$ at a given input $\btheta$ is the same or larger than the current best misfit value. This is reflected in $\cI_{rel}(\btheta; \bpsi) = 0$. The maximum achievable relative improvement value is $\cI_{rel}(\btheta; \bpsi) = 1$, which corresponds to the case when the GP model with a fixed $\bpsi$ takes the exact value of the measurement data $\bz$ at a given $\btheta$. The fact that we have a mixture of GP models for a given $\btheta$ induced by $p(\bpsi | \cD)$ means that, for the same $\btheta$, $\cI_{rel}(\btheta; \bpsi)$ can take a range of values from $0$ to $1$. By taking the average of the relative improvement values at a given $\btheta$, we obtain an ``expected improvement'' under the current GP mixture model. This motivates the formulation in \eqref{eq:expected_improvement_max_problem} (with the  objective multiplied by $g_{min}$). By adding a maximizer of \eqref{eq:expected_improvement_max_problem} to the training set $\cD$, we strive to improve our GP model in a way that improves the likelihood approximation $L(\btheta | \bz, \cD)$. 

The strategy described above is a greedy one-step look-ahead strategy that is focused on finding the modes of the likelihood function $L(\btheta | \bz)$. This strategy is designed to explore the parameter space globally just enough to make sure that the GP-based likelihood $L(\btheta | \bz, \cD)$ does not have modes in the regions where the true likelihood is flat, and to generate a sufficient number of training inputs locally around the regions of the modes of $L(\btheta | \bz)$ where the GP accuracy is  needed the most. Thus, our strategy balances exploration and exploitation---the global search  reduces the uncertainty in the GP model while  the local search samples in regions where the data misfit is likely to be minimized. The pseudocode of the full algorithm is described in Algorithm \ref{algo:adaptive_GP}.
 
%%\begin{subequations}\label{eq:expected_improvement_max_constrained_version}
%\begin{equation}\label{eq:expected_improvement_max_constrained_version}
%	\max_{\theta} \quad \frac{1}{n_\psi} \sum_{j=1}^{n_\psi}  \left[ \cI(\theta, \psi^{(j)}) \right]^+
%%	\text{s.t. } &\cI(\theta, \psi^{(j)}) \geq \delta^{(j)}, \\
%%	                     & \delta^{(j)} \geq 0.
%\end{equation}
%%\end{subequations}
%It is easy to see that the two problems \eqref{eq:expected_improvement_max_problem} and \eqref{eq:expected_improvement_max_constrained_version} are equivalent with the constraint $g(\theta; \cD, \psi) \leq g_{min}$ in \eqref{eq:expected_improvement_max_problem} being implicit. By adding the maximizers of either of this problems to the training set and updating our GP models we strive to balance a local and global search in the parameter space.

\begin{algorithm}[!htb]
\caption{Adaptive construction of GP surrogate for likelihood estimation}
\begin{algorithmic}[1]\label{algo:adaptive_GP}
	\REQUIRE Initial design $\big\{\btheta^{(j)}_{train}\big\}_{j=1}^{n_{train}}$, threshold value $\epsilon_{thresh}$, search space $\cB_\theta$,\newline maximum number of forward model evaluations $n_{max}$.
	\ENSURE Surrogate-based $\cD$-restricted likelihood $L(\btheta | \bz, \cD)$.
	
	\STATE Evaluate $\bff(\btheta)$ for $\btheta\in\big\{\btheta^{(j)}_{train}\big\}_{j=1}^{n_{train}}$ to obtain $\cD = \big\{\btheta^{(j)}_{train}, \bff\big(\btheta^{(j)}_{train}\big)\big\}_{j=1}^{n_{train}}$.
	\FOR{$k$ from $1$ to $n_{max}$}
	\STATE using $L(\bpsi | \cD)$ from \eqref{eq:evidence_mult_out}, run MCMC to obtain hyperparameter samples $\{\bpsi^{(j)}\}_{j=1}^{n_\psi}$ from the posterior distribution $p(\bpsi | \cD)$;
    \STATE evaluate $g_{min}$ as in \eqref{eq:g_min_value};
	\STATE solve maximization problem \eqref{eq:expected_improvement_max_problem}, and let $\btheta^{(k)}=\argmax\limits_{\ensuremath{\boldsymbol{\theta}}\in\cB_\theta} \cI(\btheta)$;
	\IF{$\cI\big(\btheta^{(k)}\big) < \epsilon_{thresh}\cdot g_{min}$} 
	\STATE \textbf{break}
	\ENDIF
    \STATE evaluate forward model $\bff$ at $\btheta^{(k)}$ and augment training set: $\cD = \cD\cup\big\{\btheta^{(k)}, \bff\big(\btheta^{(k)}\big)\big\}$;
	\ENDFOR
	\RETURN $L(\btheta | \bz, \cD)$ as in \eqref{eq:D_restricted_likelihood_approx}.
\end{algorithmic}
\end{algorithm}

The relative improvement function $\cI_{rel}(\btheta; \bpsi)$ motivates a natural stopping criterion for our algorithm. Specifically, we terminate the algorithm if the maximum average relative improvement value is smaller than some threshold $\epsilon_{thresh}$. In the numerical examples, we set this threshold to be $1\%$. For the formulation \eqref{eq:expected_improvement_max_problem} that we use for the solution, this means that we terminate when the objective value $\cI(\btheta)$ is less than or equal to $\epsilon_{thresh}\cdot g_{min}$.

%\begin{enumerate}
%\item choose initial design with a small number of training points (e.g., LH design)
%
%\item build an ensemble of GP models using true function evaluations at the desing points by running MCMC on the hyperparameters
%
%\item use the posterior means and posterior variances of GP models to obtain distribution of misfit function $g(\theta, \psi)$
%
%\item employ one step of EGO on misfit function, i.e., identify the point (or points) that maximize expected improvement for the misfit function
%
%\item stop if expected improvement is $1\%$ of the current best fit value
%
%\item add selected point/s to the design and go to Step 2).
%
%\end{enumerate}

\subsection{Difference with the original expected improvement criterion}\label{sec:difference_with_EGO}

In the original paper \cite{DRJones_MSchonlau_WJWelch_1998a} that popularized the expected improvement idea,  this criterion was applied to the problem of finding the global minimum of a model function approximated by a Gaussian process. Since integration was performed with respect to a Gaussian variable, the expected improvement function could be derived in a closed form. It was then maximized using a branch-and-bound algorithm. The closed form expression from \cite{DRJones_MSchonlau_WJWelch_1998a} does not apply in our case since the distribution of the misfit function $g(\btheta; \cD, \bpsi)$ at a given $\btheta$ is not normal but determined by $p(\bpsi | \cD)$. In this sense, our expected improvement in fit criterion is different from the expected improvement in the Bayesian optimization literature where the hyperparameters are usually fixed prior to computing the expectation with respect to the random GP variable. A notable exception to that is approach taken in \cite{JSnoek_HLarochelle_RPAdams_2012a} where the original expected improvement criterion is marginalized over the hyperparameter distribution.

\subsection{Analysis of the proposed algorithm}

Problem \eqref{eq:expected_improvement_max_problem} is well-defined: the misfit function $g(\btheta; \cD, \bpsi)$ is continuous, see part (b) of Proposition \ref{prop:properties_of_GP_misfit}, and the maximization is performed over a compact set $\cB_\theta$. Thus, there exists a solution to the problem \eqref{eq:expected_improvement_max_problem}. In case of multiple maximizers, we select any one of them; we will consider the simultaneous selection of multiple maximizers in the future.

Furthermore, observe that, as Algorithm \ref{algo:adaptive_GP} progresses, the value of $g_{min}$ either decreases or stays the same. The optimal value $\cI\big(\btheta^{(k)}\big)$ might not decrease with every iteration; however, as our numerical experiments in Section \ref{sec:numericalExp} demonstrate, with the addition of new training points, it does decrease and eventually falls below the threshold value. The following simple observation ensures that new information about the forward model is obtained in every iteration.

\begin{proposition} At each iteration $k$, if the new training input $\btheta^{(k)}$ selected by Algorithm \ref{algo:adaptive_GP} is added to the training set $\cD$, it is necessarily distinct from the other points in $\cD$.
\end{proposition}
\begin{proof}
Suppose that at a $k$-th iteration the maximizer $\btheta^{(k)}$ of the problem \eqref{eq:expected_improvement_max_problem} is already in the training set $\cD$. Due to part (a) of Proposition \ref{prop:properties_of_GP_misfit} and \eqref{eq:g_min_value}, for $\btheta\in\cD$ we have $g(\btheta; \cD, \bpsi^{(j)}) = g(\btheta)\geq g_{min}$ for $j=1,\dots,n_\psi$. Therefore,  $\cI\big(\btheta^{(k)}\big) = 0$, which means that the stopping criterion in line 6 of Algorithm \ref{algo:adaptive_GP} has been satisfied. The algorithm terminates without adding $\btheta^{(k)}$ to the training set $\cD$.
\end{proof}

If the hyperparameters $\bpsi$ of the GP covariance function were known and fixed, the predictive mean would be the best linear unbiased predictor and the mean squared prediction error, given by the predictive variance, would decrease with the addition of each new training point. As more data would be accumulated, under certain assumptions on the generating process, the mean of the GP model would converge to the model function and the variance would go to zero, see, for example, \cite{AMStuart_ALTeckentrup_2018a} for the case of stationary covariance functions with known hyperparameters.

In the case of hyperparameters estimated from the data as considered here, the predictive mean estimator is neither linear nor unbiased, therefore, it becomes difficult to make statements about the mean squared prediction error and the accuracy of its estimates.

\cite{EVazquez_JBect_2010a} have analyzed the convergence of the standard expected improvement algorithm of \cite{DRJones_MSchonlau_WJWelch_1998a} for the fixed Gaussian process prior. The results in this paper, while interesting from a theoretical point of view, are not applicable in practice since the prior on the GP model cannot be known in advance. Convergence rates of the expected improvement strategy for finding the global minimum of a function modeled by a Gaussian process with a fixed prior have been derived in \cite{ADBull_2011a}. The authors also extended their results for a case when parameters of the GP prior were estimated from the data by maximizing the marginal likelihood and provided an automatic choice of the parameters that retains the convergence rate of a fixed prior. In both papers, convergence results were stated in the norm of the Reproducing Kernel Hilbert Space associated with the chosen covariance function of a Gaussian process \cite[Section~6.1]{CERasmussen_CKIWilliams_2006a}. As explained in Section \ref{sec:difference_with_EGO}, we apply the expected improvement to the misfit function $g(\btheta; \cD, \bpsi)$, which depends on the mean and the variance of a Gaussian process but is not a Gaussian process itself. Thus, the above mentioned results cannot be directly applied in our case and further detailed analysis of the proposed algorithm is needed.  

%As the number of training points increases, the MSPE of the GP decreases (for a fixed covariance function) \cite{TJSantner_BJWilliams_WINotz_2003a}.

%
%\input{posterior}
%
%!TEX root = mainAdaptiveGP.tex

\section{Numerical experiments}\label{sec:numericalExp}

In this section, we demonstrate our method on a one-dimensional model problem with a single output and on a two-dimensional source inversion problem with multiple outputs. Implementation details and additional experiments on a higher-dimensional problem are reported in the supplementary materials \ref{sec:implementation} and \ref{sec:permeability}.

%\input{implementation}
%
%!TEX root = mainAdaptiveGP.tex

\subsection{One-dimensional example}

We start by testing the proposed algorithm on a univariate scalar function. This allows us to illustrate the steps of the algorithm and to provide intuition behind it. We use the  forward model function
\begin{equation}\label{eq:1D_example_function}
	f(\theta) = \frac{\theta^2 - 5\theta + 6}{\theta^2 + 1}, \quad \theta\in[-6,+6].
\end{equation}
We generate measurement data $z$ by evaluating $f(\theta)$ at $\theta_{true}=2.41$ and adding zero-mean Gaussian noise with $\sigma = 0.01$. The function defined in \eqref{eq:1D_example_function} and the measurement $z$ are shown in Figure \ref{fig:1D_function}.

\begin{figure}[!htb]
    \centering
    \begin{subfigure}[t]{0.45\textwidth}
        \centering
        \includegraphics[width=\textwidth]{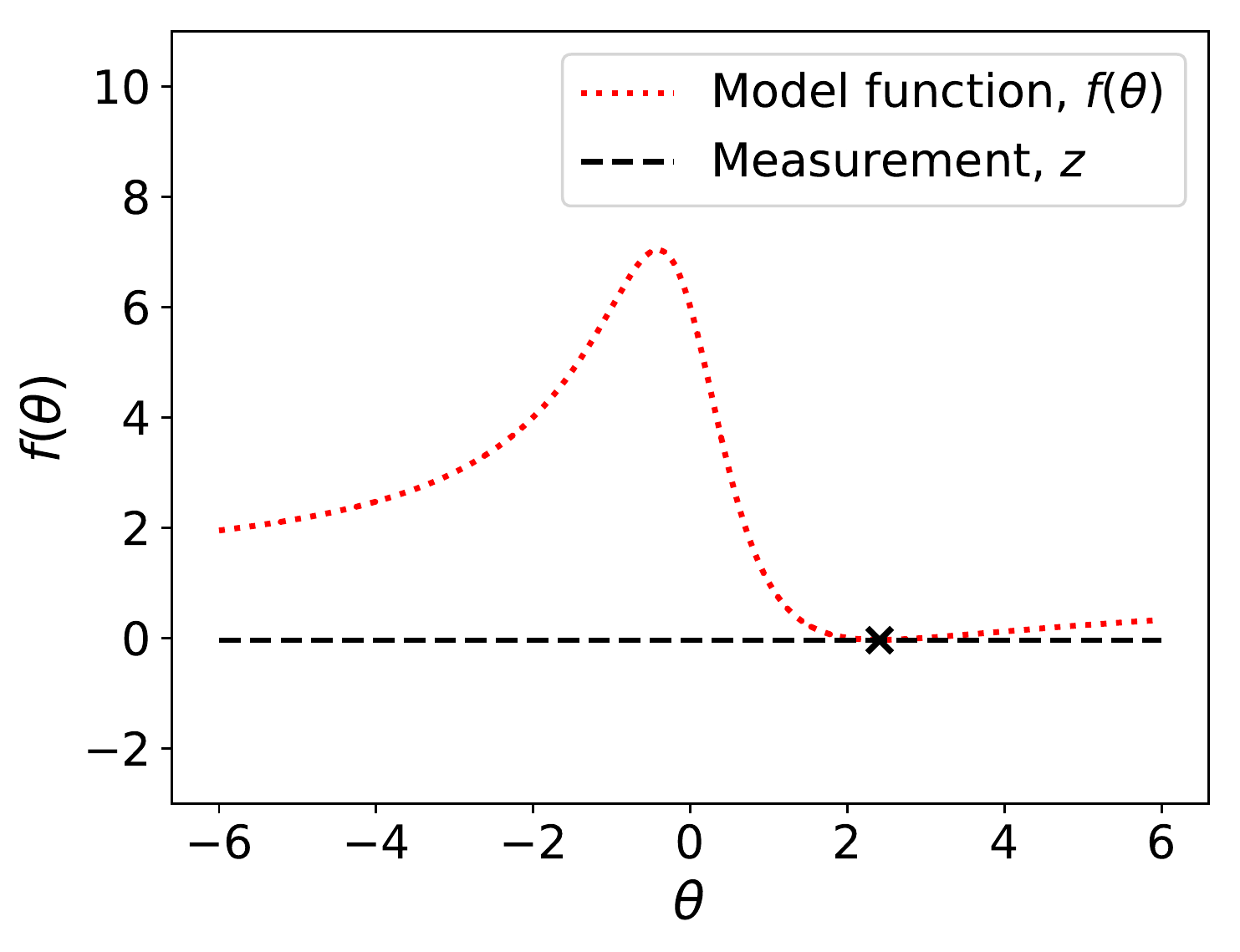}
	\vspace{-1.4\baselineskip}
    \end{subfigure}%
    \caption{Function \eqref{eq:1D_example_function}, $\theta_{true}$ (black cross), and measurement level $z$.}
    \setlength{\belowcaptionskip}{-10pt}
    \label{fig:1D_function}
\end{figure}

\begin{figure}[!htb]
    \centering
    \begin{subfigure}[t]{0.45\textwidth}
        \centering
        \includegraphics[width=\textwidth]{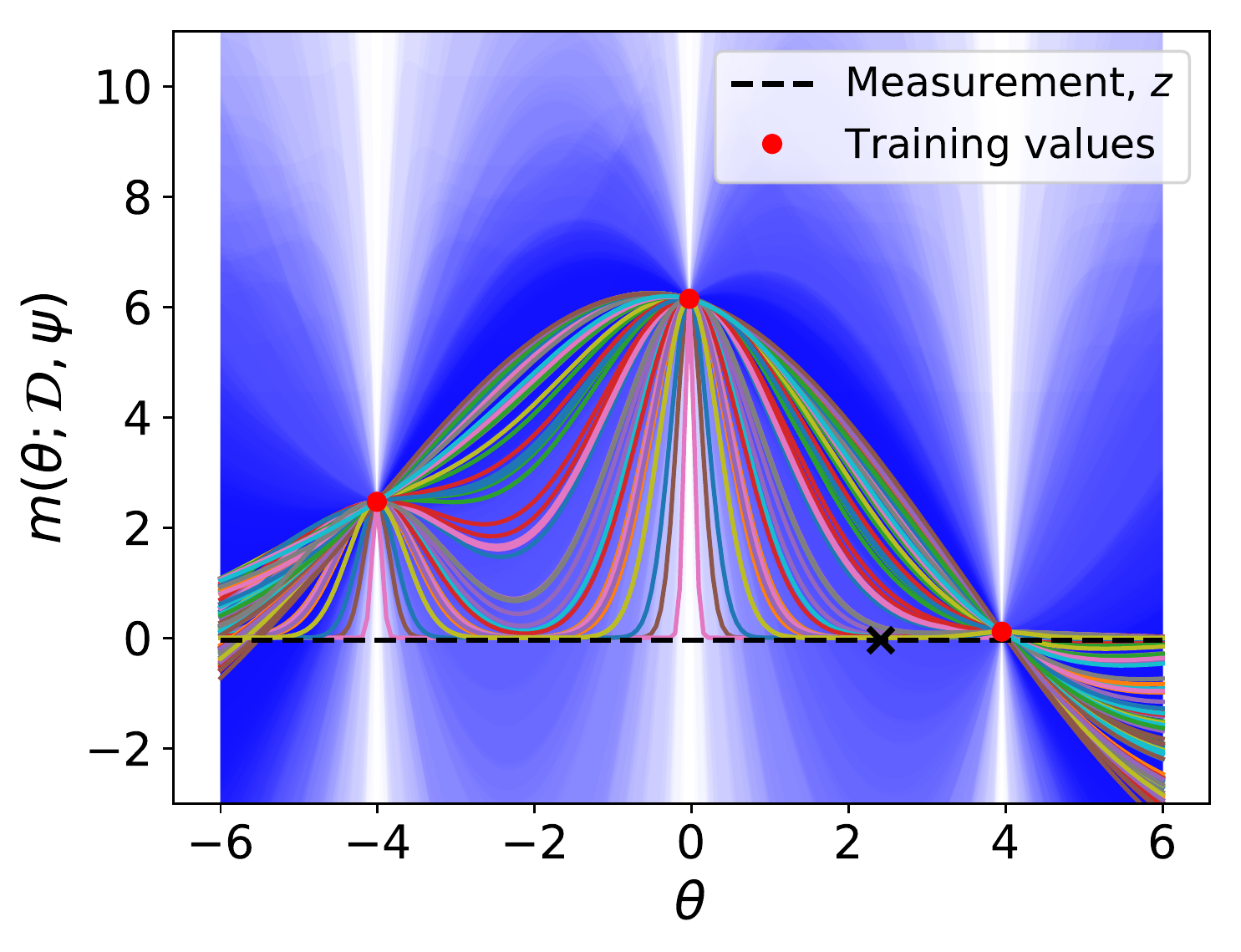}
	\vspace{-1.4\baselineskip}
        \caption{\small{Predictive means $m(\theta; \cD, \bpsi)$ (graphs)  and $95\%$ confidence regions (blue shaded areas) at iteration $k=1$.}}
        \label{fig:1D_means_step0}
    \end{subfigure}%
    ~~
    \begin{subfigure}[t]{0.47\textwidth}
        \centering
        \includegraphics[width=\textwidth]{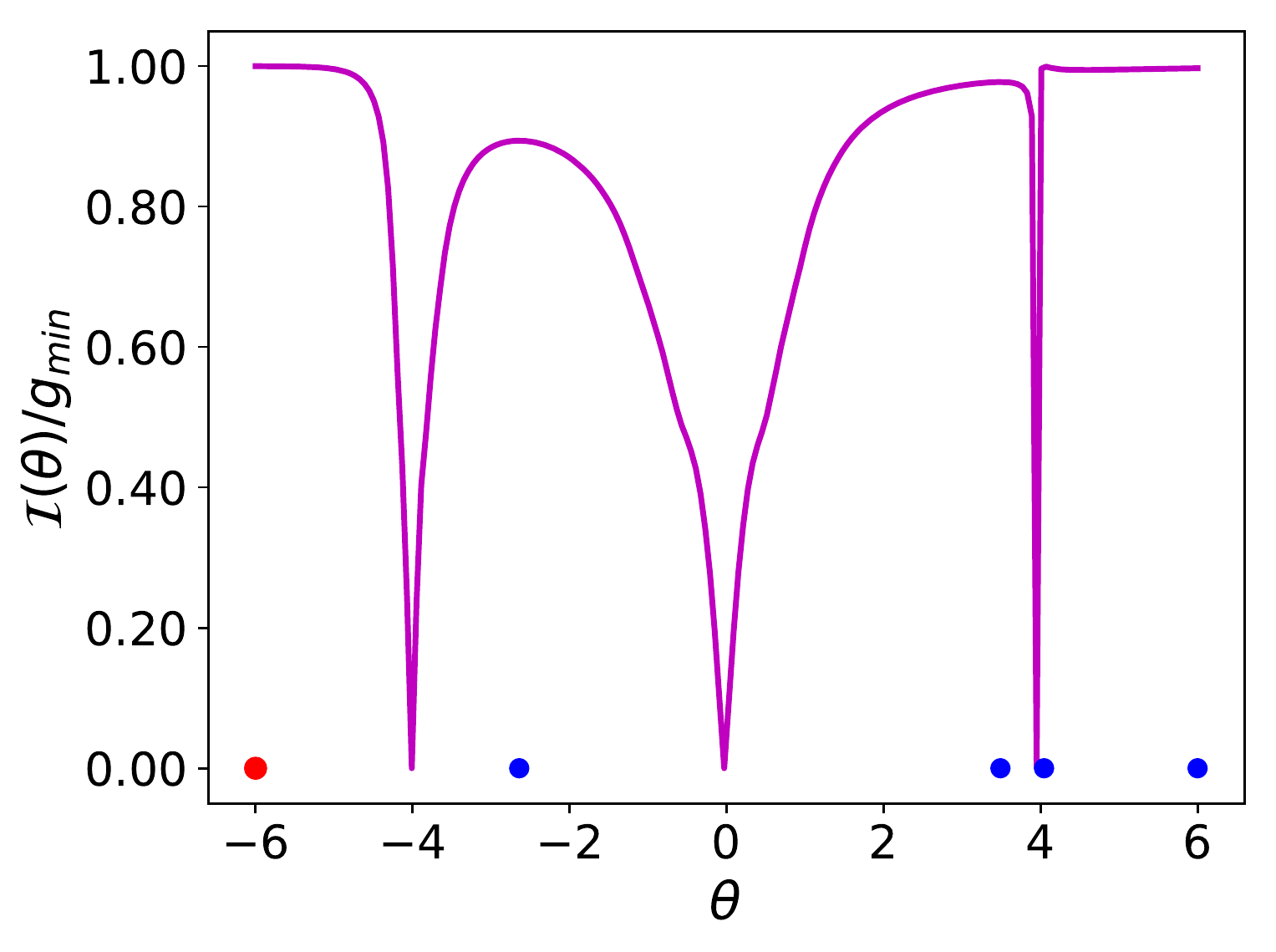}
	\vspace{-1.4\baselineskip}
        \caption{\small{Relative expected improvement function $\cI(\theta)/g_{min}$ at iteration $k=1$.}}
        \label{fig:1D_exp_imp_step0}
    \end{subfigure}%
    \\
    \begin{subfigure}[t]{0.45\textwidth}
        \centering
        \includegraphics[width=\textwidth]{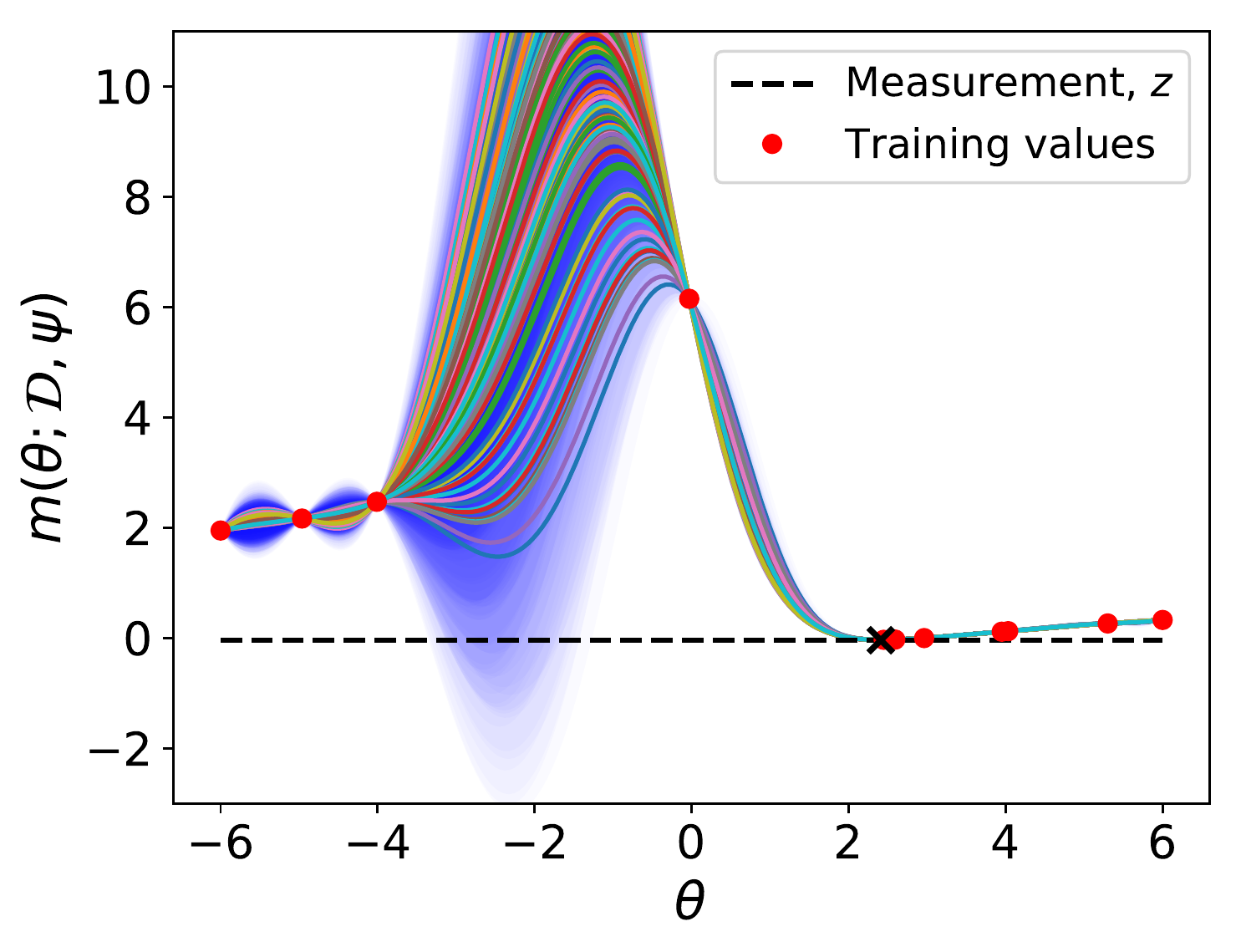}
	\vspace{-1.4\baselineskip}
        \caption{\small{Predictive means $m(\theta; \cD, \bpsi)$ (graphs)  and $95\%$ confidence regions (blue shaded areas) at iteration $k=10$.}}
        \label{fig:1D_means_step9}
    \end{subfigure}%
    ~~
    \begin{subfigure}[t]{0.48\textwidth}
        \centering
        \includegraphics[width=\textwidth]{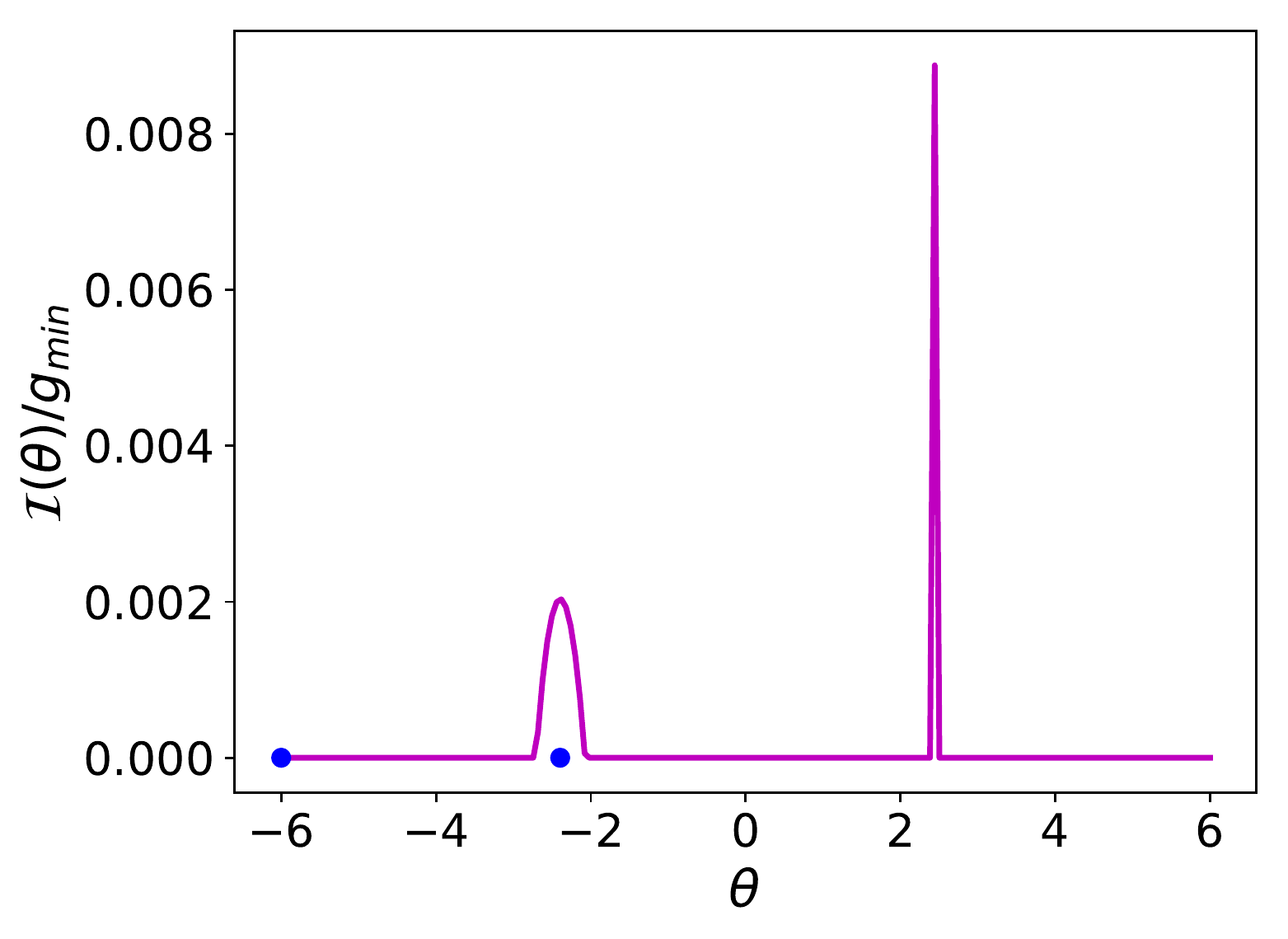}
	\vspace{-1.4\baselineskip}
        \caption{\small{Relative expected improvement function $\cI(\theta)/g_{min}$ at iteration $k=10$.}}
        \label{fig:1D_exp_imp_step9}
    \end{subfigure}%
\caption{Predictive means and expected improvement functions at the first and at the last iterations of Algorithm \ref{algo:adaptive_GP} applied to function \eqref{eq:1D_example_function}. In the left plots, red circles indicate training values and black cross indicates the location of $\theta_{true}$. In the right plots, blue circles indicate local maxima of the expected improvement function and the larger red circle indicates the selected input $\theta^{(k)}$.} %\color{blue}{should theta here not be bold since its a scalar? }}
\label{fig:1D_ex_means_and_exp_imp}
\end{figure}

We use  three equidistant input values $\{-4, 0, +4\}$ and the corresponding function values as initial design $\cD$. 
Next, we apply Algorithm \ref{algo:adaptive_GP} in the following way. For the hyperparameter estimation, we start with uniform priors on the covariance parameters $\sigma_c$ and $\ell_1$ (see \eqref{algo:adaptive_GP}): $p(\sigma_c)=\cU(10^{-8}, 12)$, $p(\ell_1)=\cU(10^{-8},5)$. 
% We initialize the \textit{emcee} sampler with $100$ walkers and run chains for $400$ steps; the final states of the chains are taken to be the posterior samples of the hyperparameters $\bpsi=(\sigma_c, \ell_1)^T$ resulting in $n_\psi=100$ posterior samples $\bpsi^{(j)}$.
Using MCMC we obtain $n_\psi=100$ posterior samples $\bpsi^{(j)}$ (see \ref{sec:implementation} for details).
Figure \ref{fig:1D_means_step0} shows the predictive means (see \eqref{eq:gp_pred_mean}) corresponding to different hyperparameter vectors together with the $95\%$ confidence regions around them based on the predictive variances \eqref{eq:gp_pred_var}. We observe that a small number of initial training inputs results in a broad hyperparameter posterior with a variety of corresponding GP models. The variances at the untested inputs appear to be quite large leading to wide confidence regions. This ensemble of GP models consistent with the initial training data gives a better idea of the uncertainty of the surrogate model than any single surface approximation would, for example, the one corresponding to the maximum likelihood estimate of the hyperparameters \eqref{eq:max_log_marg_psi}.

\begin{figure}[!htb]
    \centering
    \begin{subfigure}[t]{0.45\textwidth}
        \centering
        \includegraphics[width=\textwidth]{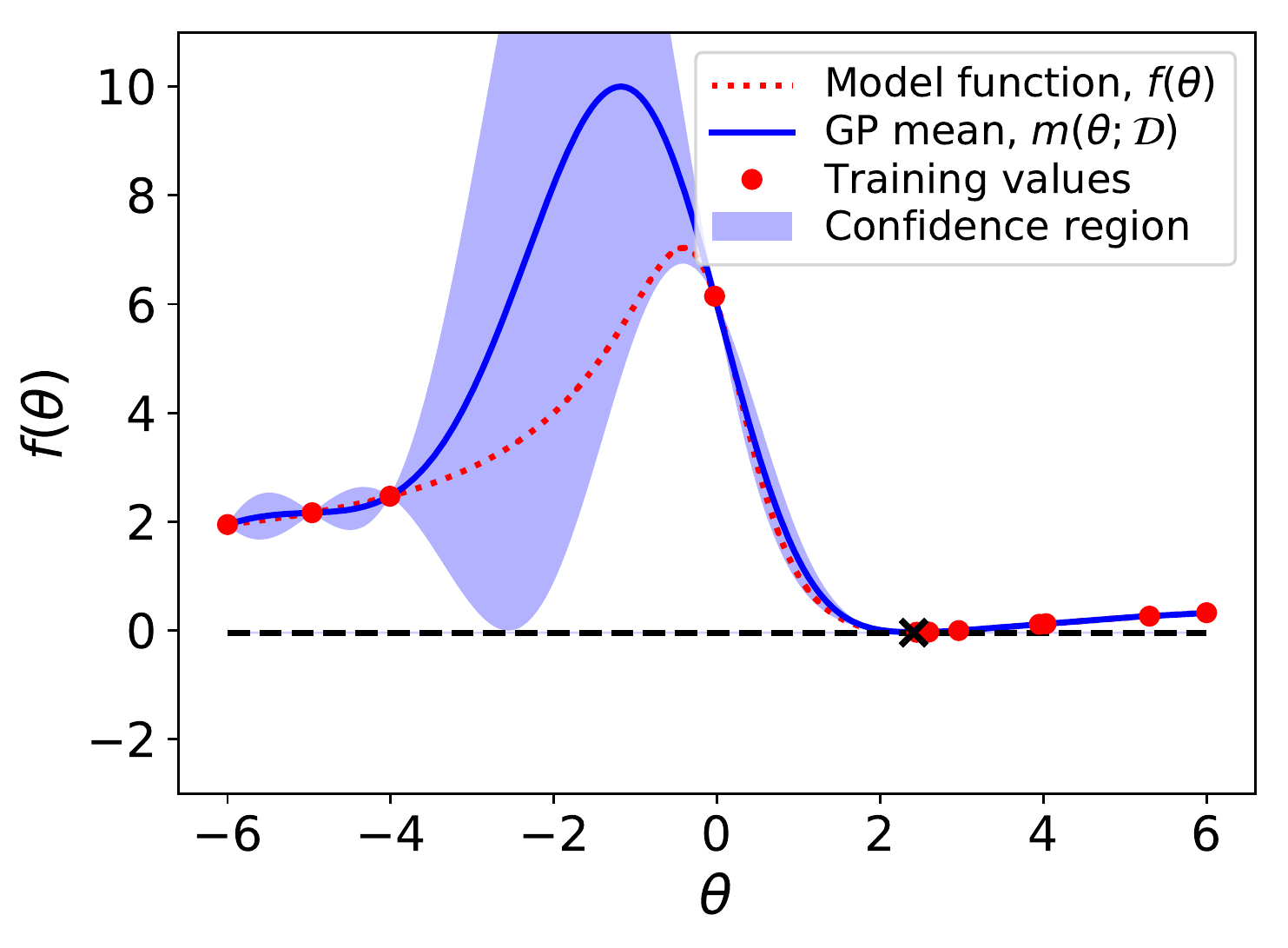}
	\vspace{-1.4\baselineskip}
        \caption{GP approximation with adaptive design.}
        \label{fig:1D_prediction_gp_adaptive}
    \end{subfigure}%
    ~
    \begin{subfigure}[t]{0.45\textwidth}
        \centering
        \includegraphics[width=\textwidth]{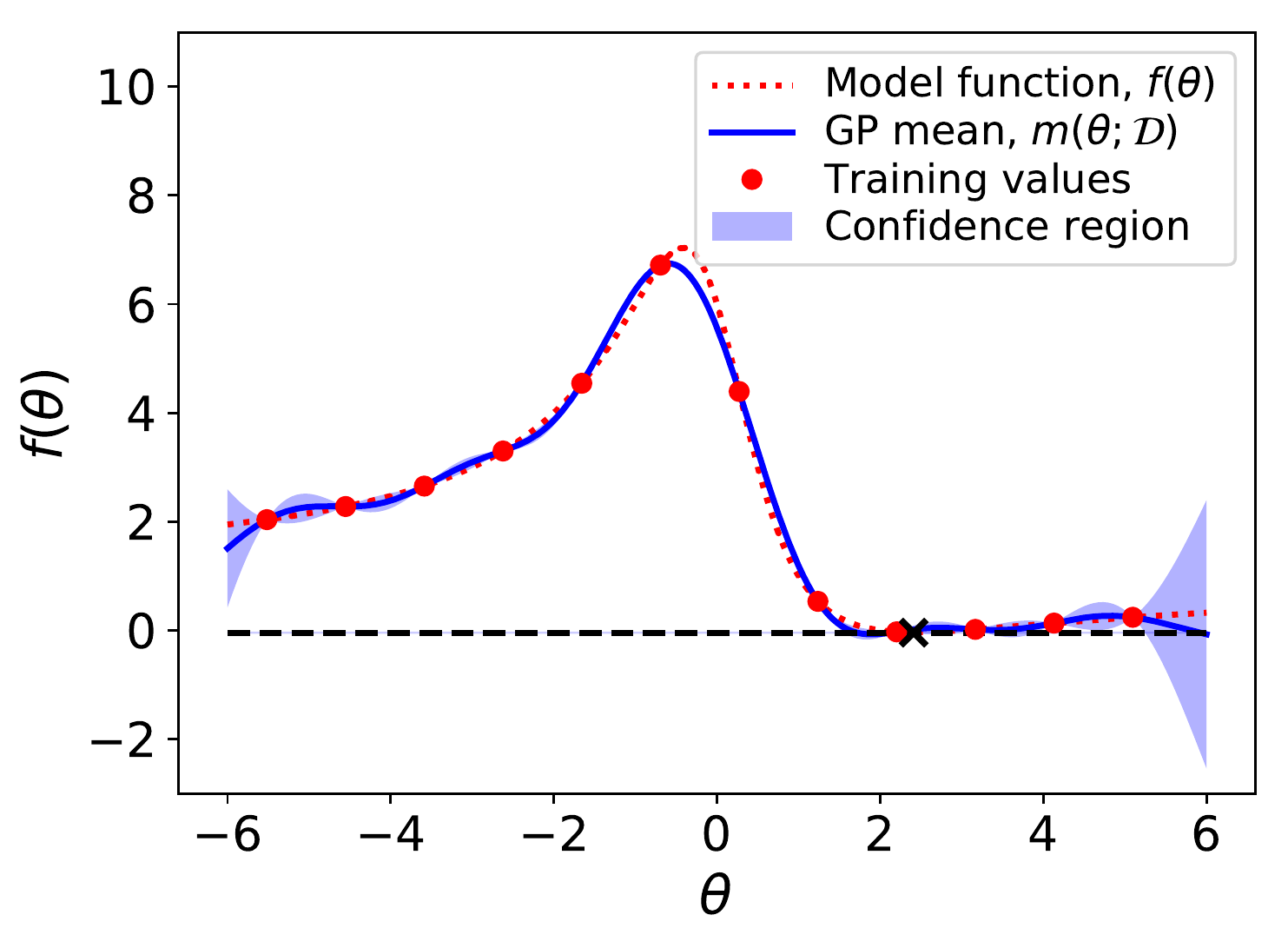}
	\vspace{-1.4\baselineskip}
        \caption{GP approximation with fixed design.}
        \label{fig:1D_prediction_LH}
    \end{subfigure}%
\caption{Predictive GP models corresponding to adaptive and non-adaptive designs. Means $m(\theta; \cD)$ are computed as in \eqref{eq:mean_mixture_single_out} and $95\%$ confidence regions are based on \eqref{eq:var_mixture_single_out}.}
\end{figure}

\begin{figure}[!htb]
    \centering
    \begin{subfigure}[t]{0.46\textwidth}
        \centering
        \includegraphics[width=\textwidth]{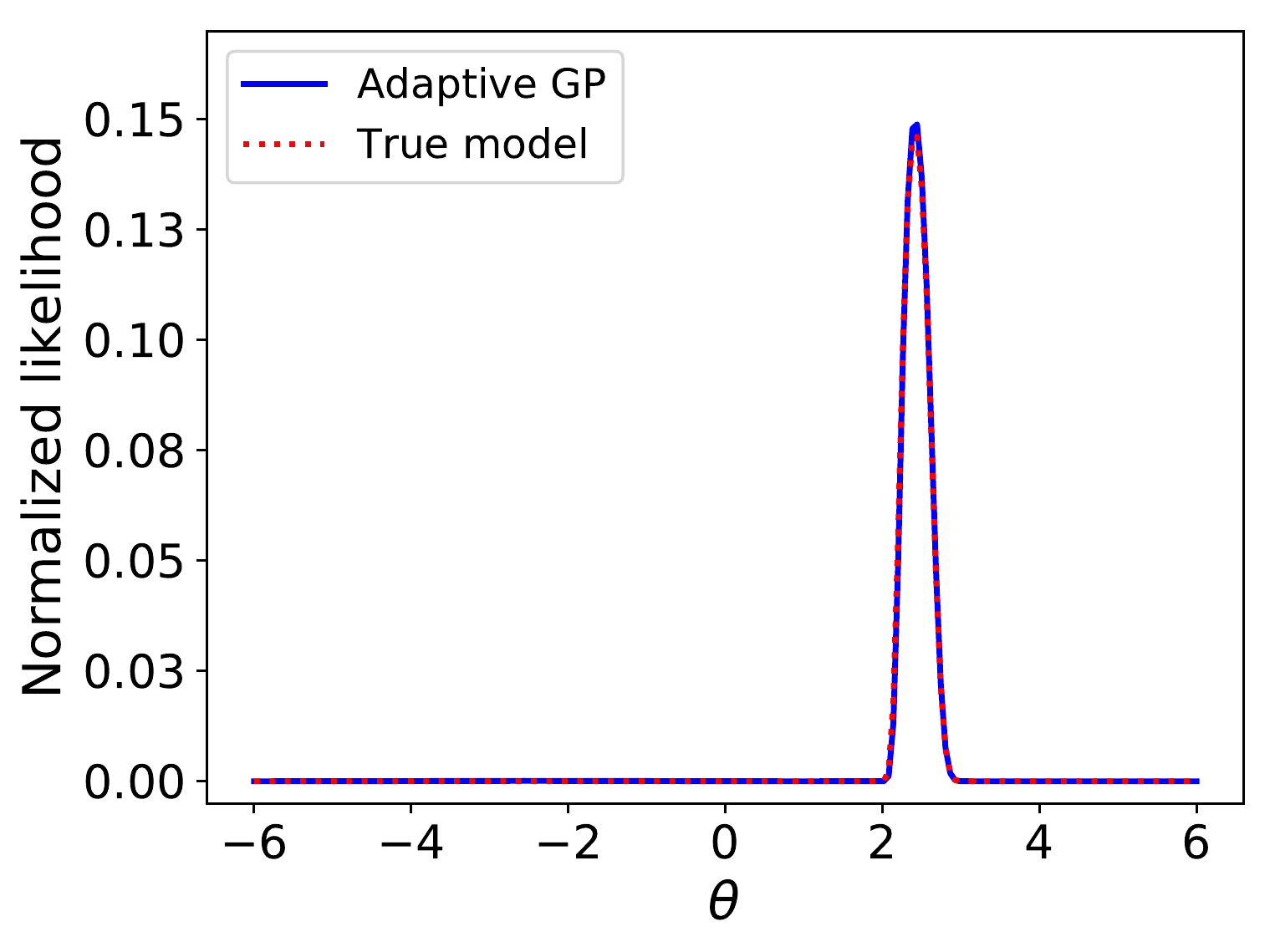}
	\vspace{-1.4\baselineskip}
        \caption{\small{Likelihood based on the adaptive design in Figure \ref{fig:1D_prediction_gp_adaptive}.}}
        \label{fig:1D_lik_gp_adaptive}
    \end{subfigure}%
    ~
    \begin{subfigure}[t]{0.46\textwidth}
        \centering
        \includegraphics[width=\textwidth]{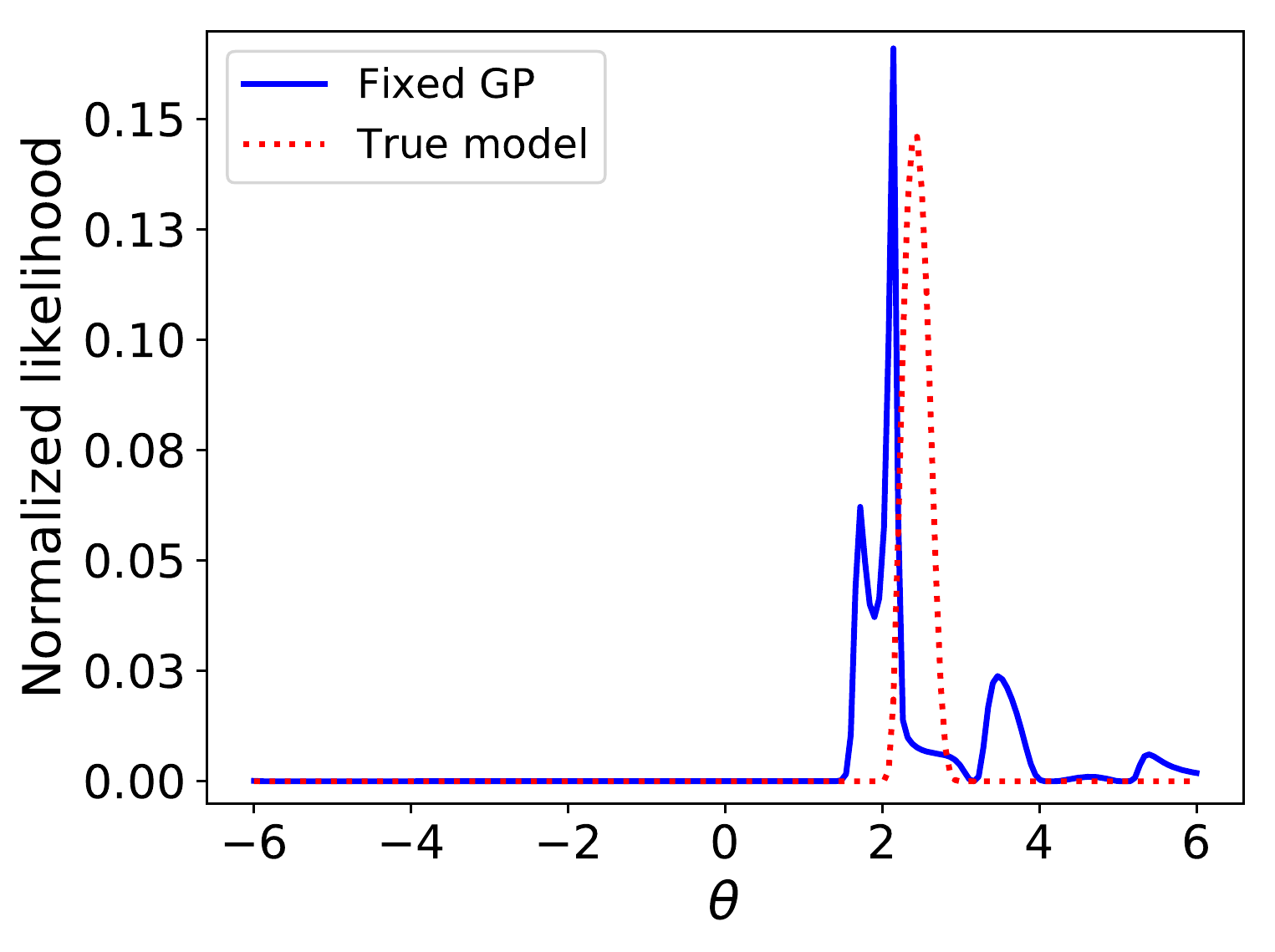}
	\vspace{-1.4\baselineskip}
        \caption{\small{Likelihood based on the fixed design in Figure \ref{fig:1D_prediction_LH}.}}
        \label{fig:1D_lik_LH}
    \end{subfigure}%
\caption{$\cD$-restricted likelihood functions built with adaptive and non-adaptive designs and compared to the likelihood computed with the true model \eqref{eq:1D_example_function}.}
\end{figure}

Using the ensemble of the GP models,  we now formulate the optimization problem \eqref{eq:expected_improvement_max_problem}. The objective function $\cI(\theta)$ (scaled by $1/g_{min}$) is presented in Figure \ref{fig:1D_exp_imp_step0}. By comparing Figures \ref{fig:1D_means_step0} and \ref{fig:1D_exp_imp_step0}, we observe that the expected improvement is largest in the regions where the majority of predictive means is close to the measurement value; it also grows with the distance to the points in the training set. This is the desired behavior: we want to explore the regions with highest uncertainty that are most likely to result in values at the level of the measurement $z$. 

We solve the problem \eqref{eq:expected_improvement_max_problem} using gradient-based optimization as described in the supplement \ref{sec:implementation} with $25$ initial points taken equidistantly on the interval $[-6, +6]$. The obtained maximizers are shown as blue circles in Figure \ref{fig:1D_exp_imp_step0}; the larger red circle at $\theta=-6$ corresponds to the largest maximum. We add this maximizer to the training set $\cD$ and proceed. We terminate the algorithm once the stopping criterion in line 6 of  Algorithm \ref{algo:adaptive_GP} is satisfied with $\epsilon_{thresh}=0.01$. At the final step ($k=10$), which corresponds to $12$ total training inputs, the means of the predictive distributions corresponding to the ensemble of the hyperparameters $\bpsi$ look as shown in Figure \ref{fig:1D_means_step9}. We observe that our adaptive algorithm has primarily added to the training set the inputs to the right of $\theta_{true}$ where the function \eqref{eq:1D_example_function} takes values closest to $z$. With the large number of training inputs in this region, the GP models are indistinguishable from the true function and the variances are low. On the other hand, to the left of $\theta=0$, the adaptive algorithm has added only two more inputs to the training set; as a result, there is still considerable uncertainty associated with the GP predictions in that region. However, most of the ensemble means in this region are far from the measurement value $z$, and the uncertainty of the ensemble values is not sufficiently large to imply that the measurement data could have originated in this region and to warrant further exploration. This is reflected in the maximum value of the relative expected improvement being below the threshold $\epsilon_{thresh}=0.01$ (see Figure \ref{fig:1D_exp_imp_step9}).

The obtained mixture GP approximation is presented in Figure \ref{fig:1D_prediction_gp_adaptive}. Here, we plot the mean of the ensemble as in \eqref{eq:mean_mixture_single_out} and base the $95\%$ confidence region on \eqref{eq:var_mixture_single_out}. The inputs in the final training set are shown as red circles. With this final design, we evaluate and plot the approximate $\cD$-restricted likelihood function \eqref{eq:D_restricted_likelihood_approx} (see Figure \ref{fig:1D_lik_gp_adaptive}). In the same figure, we plot the ``true'' likelihood $L(\theta | z)$ evaluated using the forward model function \eqref{eq:1D_example_function}. We observe that the two are identical.

Finally, we contrast our adaptively constructed GP-based likelihood with a GP-based likelihood constructed using a naive fixed design. The mixture GP model constructed with $12$ equidistant training inputs is shown in Figure \ref{fig:1D_prediction_LH} and the $\cD$-restricted likelihood based on this design is shown in Figure \ref{fig:1D_lik_LH}. We observe that the GP-based likelihood using the naive non-adaptive design is of considerably worse quality than the one that was built adaptively, even though the GP model corresponding to the naive design looks reasonably good and has narrow confidence regions.
%
%!TEX root = mainAdaptiveGP.tex

\subsection{Source inversion}
Next, we consider the source inversion problem studied in \cite{YMMarzouk_HNNajm_LARahn_2007a} and later in \cite{JLi_YMMarzouk_2014a}. The forward model is given by a diffusion equation in two dimensions:
\begin{subequations}\label{eq:heat}
\begin{align}
	\frac{\partial u}{\partial t} - \nabla^2 u &= s(\bx, t) &\bx\in \Omega\coloneqq[0,1]^2 \\
	\nabla u \cdot \mathbf{n} &= 0 &\bx\in\partial \Omega \\
	u(\bx, 0) &= 0 &\bx\in \Omega.
\end{align}
\end{subequations}
The source term $s(\bx, t)$ is given by
\[
	s(\bx, t) = \begin{cases}
                               \frac{a}{2\pi h^2} \exp\big( - \| \btheta - \bx\|^2/2h^2\big), &\quad 0\leq t\leq \tau, \\
                               0, &\quad t>\tau.
\end{cases}
\]
The following parameters have fixed values: $a=2$, $h=0.05$, $\tau=0.1$. The location of the source center is denoted by $\btheta$ and is the (two-dimensional) parameter of interest. We solve \eqref{eq:heat} in FEniCS \cite{ALogg_KAMardal_GNWells_2012a} using a $32\times 32$ uniform finite element mesh with piecewise-linear Lagrange elements, and backward Euler time discretization with a time step of $0.01$.

The measurements are taken at times $t=0.1$ and $t=0.2$ on a uniform $3\times 3$ grid covering $\Omega$ resulting in a total of $18$ measurements. The measurement noise is assumed to be a vector of independent zero-mean Gaussian random variables. Thus, the model is:
\[
	\bz = \bff(\btheta) + \be, \quad e_i \sim \cN(0, \sigma_i^2), \quad i=1,\dots,18.
\]
We fix $\sigma_i=0.1$ for all $i$. The measurement data $\bz$ is generated by solving the forward model with $\btheta_{true}=(0.25, 0.75)$ and adding noise. To avoid the obvious ``inverse crime'' \cite[Section~1.2]{JKaipio_ESomersal_2005a}, the measurement data is generated with a finer $128\times 128$ grid and a smaller time step of $0.0025$.

\begin{figure}[htb!]
    \centering
    \begin{subfigure}[t]{0.45\textwidth}
        \centering
        \includegraphics[width=\textwidth]{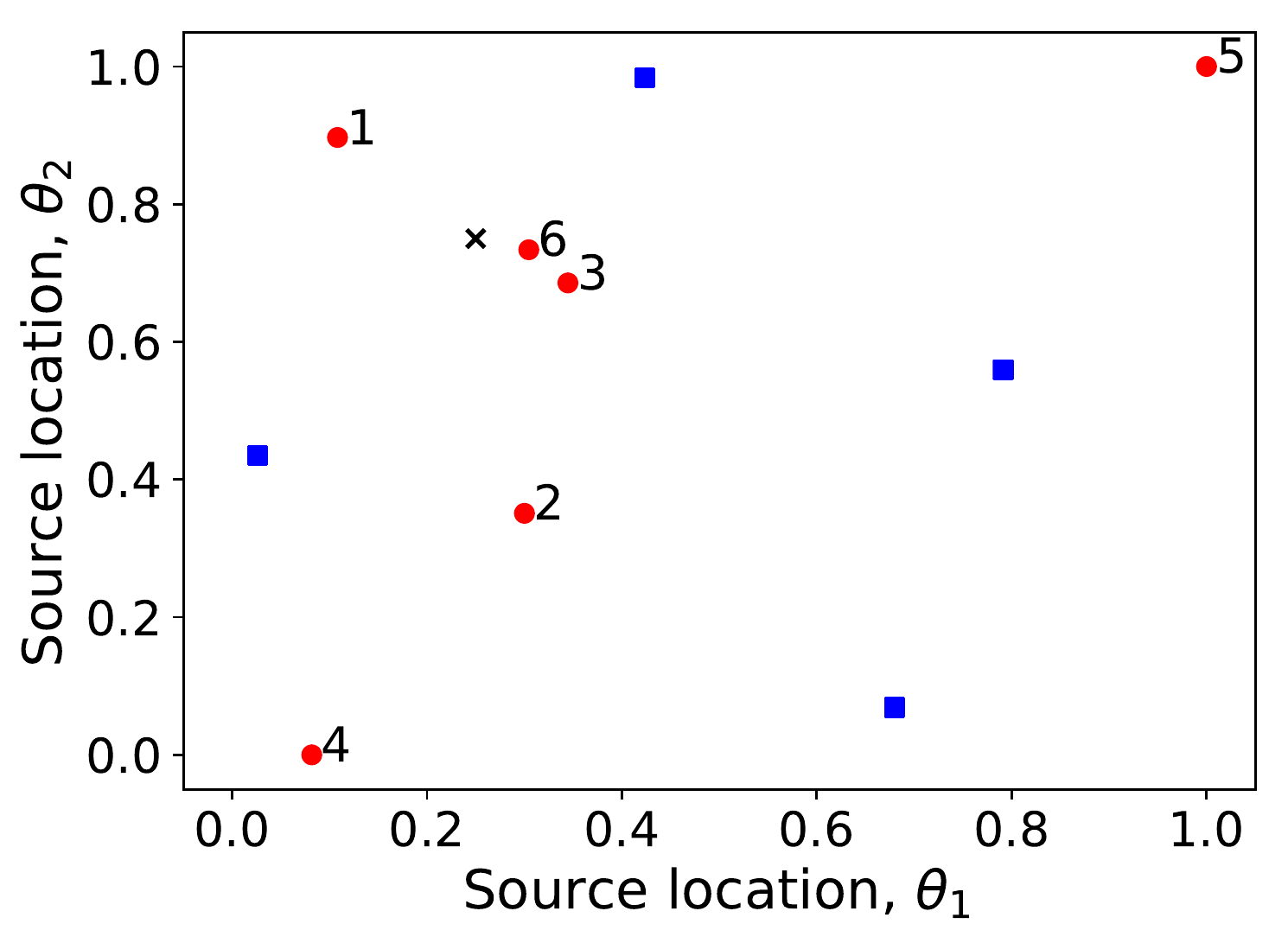}
	\vspace{-1.4\baselineskip}
        \caption{\small{Final adaptive design.}}
        \label{fig:heat_design_final}
    \end{subfigure}%
   ~
    \begin{subfigure}[t]{0.45\textwidth}
        \centering
        \includegraphics[width=\textwidth]{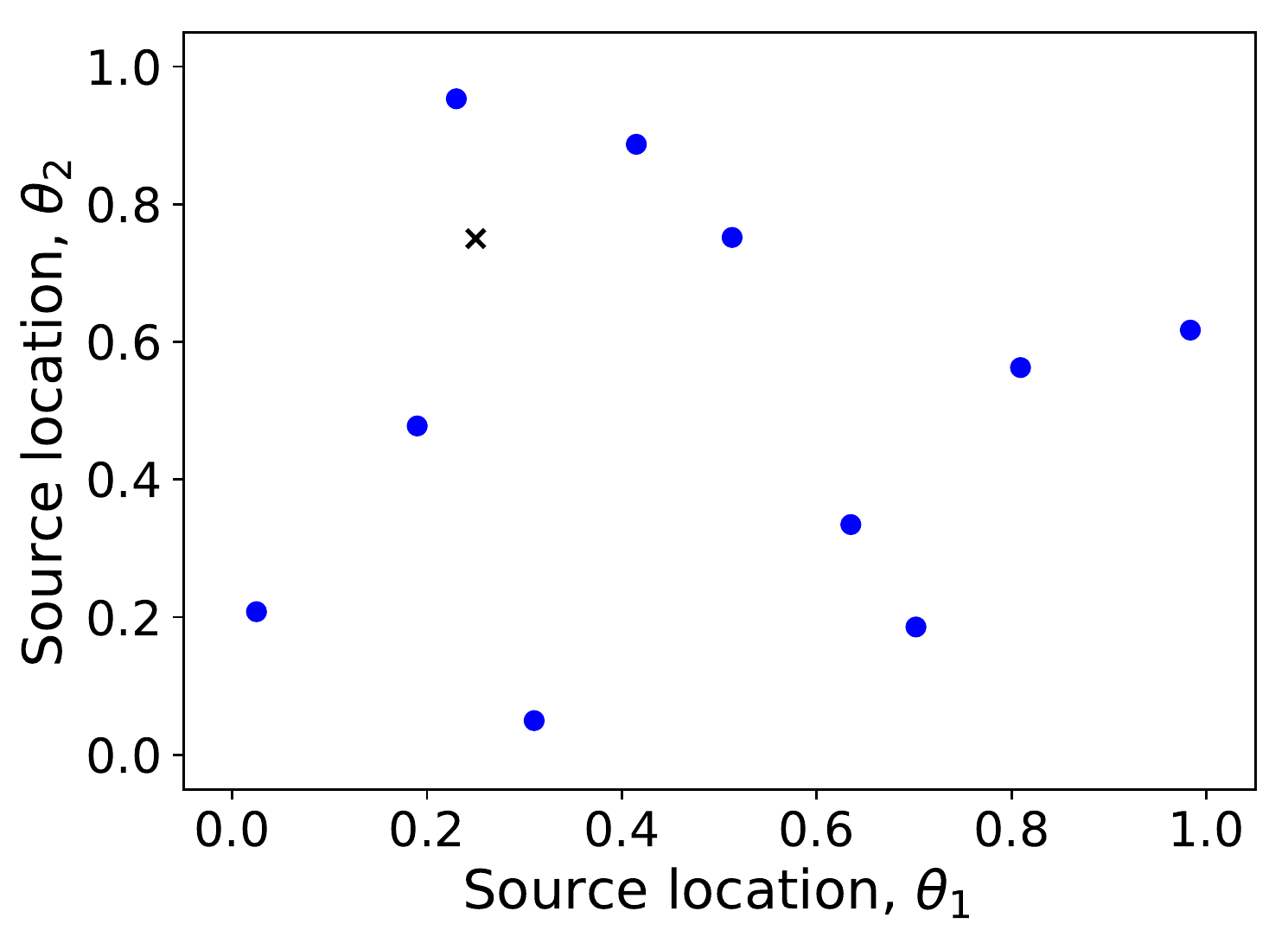}
	\vspace{-1.4\baselineskip}
        \caption{\small{Fixed Latin hypercube design.}}
        \label{fig:heat_design_LH}
    \end{subfigure}%
%\vspace{-1.\baselineskip}
\caption{Designs for the problem \eqref{eq:heat}:  \subref{fig:heat_design_final} final adaptive design with the blue squares being the initial inputs and  the red circles added by Algorithm \ref{algo:adaptive_GP} (with numbers indicating iterations $k$ at which they were added), and \subref{fig:heat_design_LH} non-adaptive Latin hypercube design  with the same total number of inputs. Black cross indicates $\btheta_{true}$.}
\end{figure}

We start with the initial design $\cD$ consisting of $4$ inputs $\btheta$ arranged in a Latin hypercube design (see the blue squares in Figure \ref{fig:heat_design_final}). The priors for the hyperparameters are taken as follows: $p(\sigma_c)=\cU(10^{-8},2)$, $p(\ell_1)=\cU(10^{-8},1)$, $p(\ell_2)=\cU(10^{-8},1)$. The hyperparameter posterior is obtained with MCMC using the likelihood function \eqref{eq:evidence_mult_out} with normalized outputs \eqref{eq:scaled_outputs_mult_case}.

We run Algorithm \ref{algo:adaptive_GP} with $n_{max}=11$, $\epsilon_{thresh}=0.01$, and $\cB_{\theta}=[0,1]^2$. To solve \eqref{eq:expected_improvement_max_problem}, we initialize the optimization algorithm with $50$ points from a two-dimensional Sobol sequence. If the resulting maximum expected improvement is less than the threshold, we perform another search initialized at an additional $100$ Sobol points.

\begin{figure}[htb!]
    \centering
    \begin{subfigure}[t]{0.45\textwidth}
        \centering
        \includegraphics[width=\textwidth]{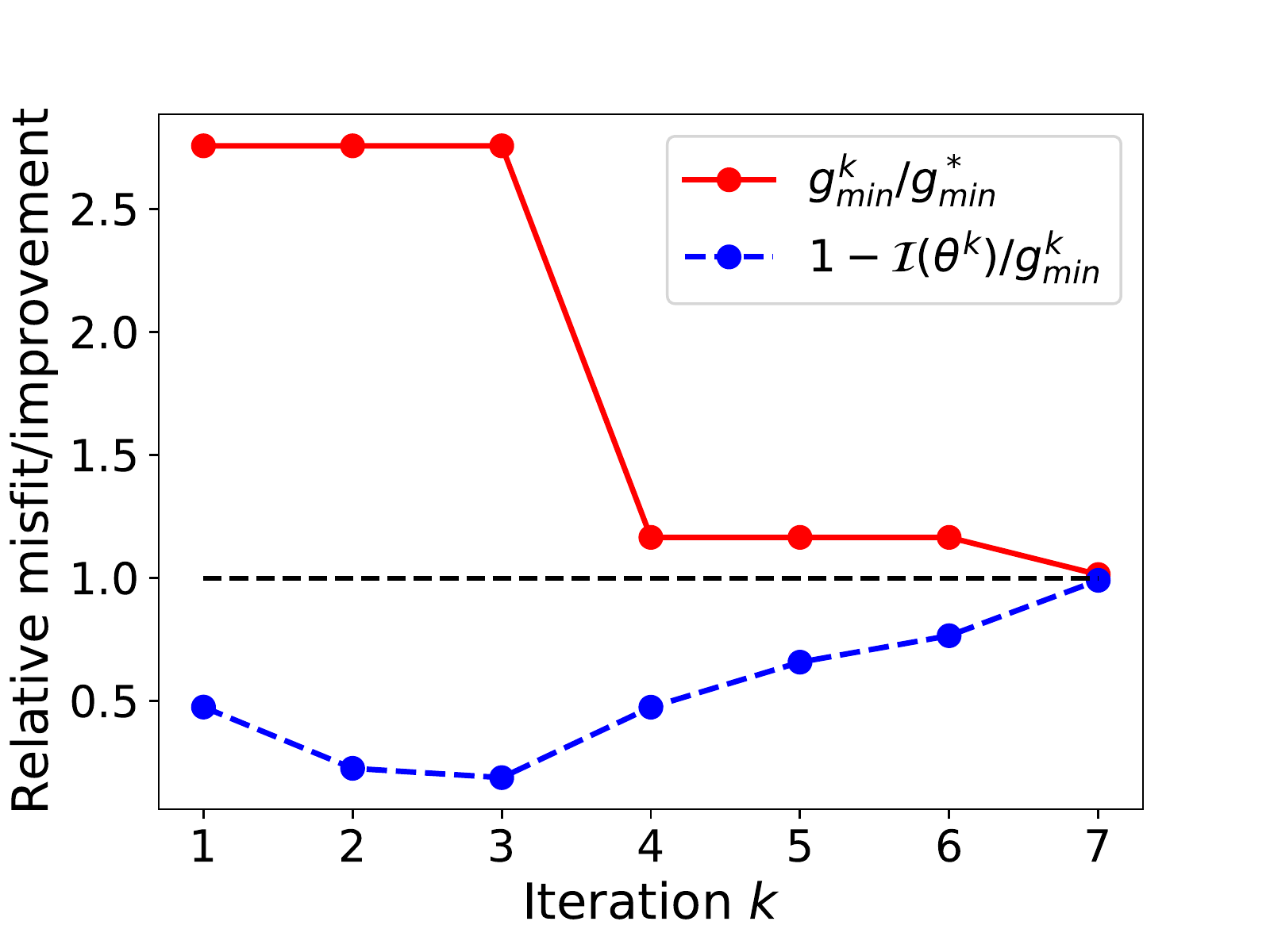}
	\vspace{-1.4\baselineskip}  
    \end{subfigure}%
    \caption{Iteration history of Algorithm \ref{algo:adaptive_GP} for the Problem \eqref{eq:heat}. Here $g_{min}^*=15.015$.}
    \label{fig:heat_misfits}
\end{figure}

Figure \ref{fig:heat_misfits} shows the iteration history of Algorithm \ref{algo:adaptive_GP}. The red solid line shows the values of $g_{min}^k/g_{min}^*$ over iterations, where $g_{min}^k$ is $g_{min}$ at iteration $k$, and $g_{min}^*=15.015$ is the minimum of $g(\btheta)$ that we find by exhaustively searching in the region around $\btheta_{true}$ and use here only as a reference value (this value is, of course, unknown in practice). The blue dashed line shows $1-\cI(\btheta^{(k)})/g_{min}^k$, i.e., one minus the relative expected improvement (recall, that $\btheta^{(k)}$ corresponds to the maximizer of problem \eqref{eq:expected_improvement_max_problem} at iteration $k$). As the algorithm progresses, we expect both lines to approach $1$.

Comparing Figure \ref{fig:heat_misfits} with the order in which the inputs were added to the training set (see the numbers in Figure \ref{fig:heat_design_final}), we can make a few observations. At the initial stages, $k=1,2,3$, the inputs that maximize the expected improvement are located  in the interior of $\cB_\theta$ and around $\btheta_{true}$. The relative expected improvement value is high---over $50\%$---and the $g_{min}$ value remains unchanged. Upon addition of input $3$, the value of $g_{min}$ drops, and the algorithm starts adding inputs corresponding to high variance---inputs $4$ and $5$. It is expected that these inputs lie on the boundary where the uncertainty is highest. At this time, the relative expected improvement steadily decreases. Finally, adding input $6$ leads to further reduction of $g_{min}$. This time, the  maximum relative expected improvement drops below the threshold value $\epsilon_{thresh}=0.01$ and the algorithm terminates. The final design in Figure \ref{fig:heat_design_final} contains $10$ training inputs, and so does the Latin hypercube design in Figure \ref{fig:heat_design_LH} that we use for comparison below.

\begin{figure}[!htb]
    \centering
    \begin{subfigure}[t]{0.45\textwidth}
        \centering
        \includegraphics[width=\textwidth]{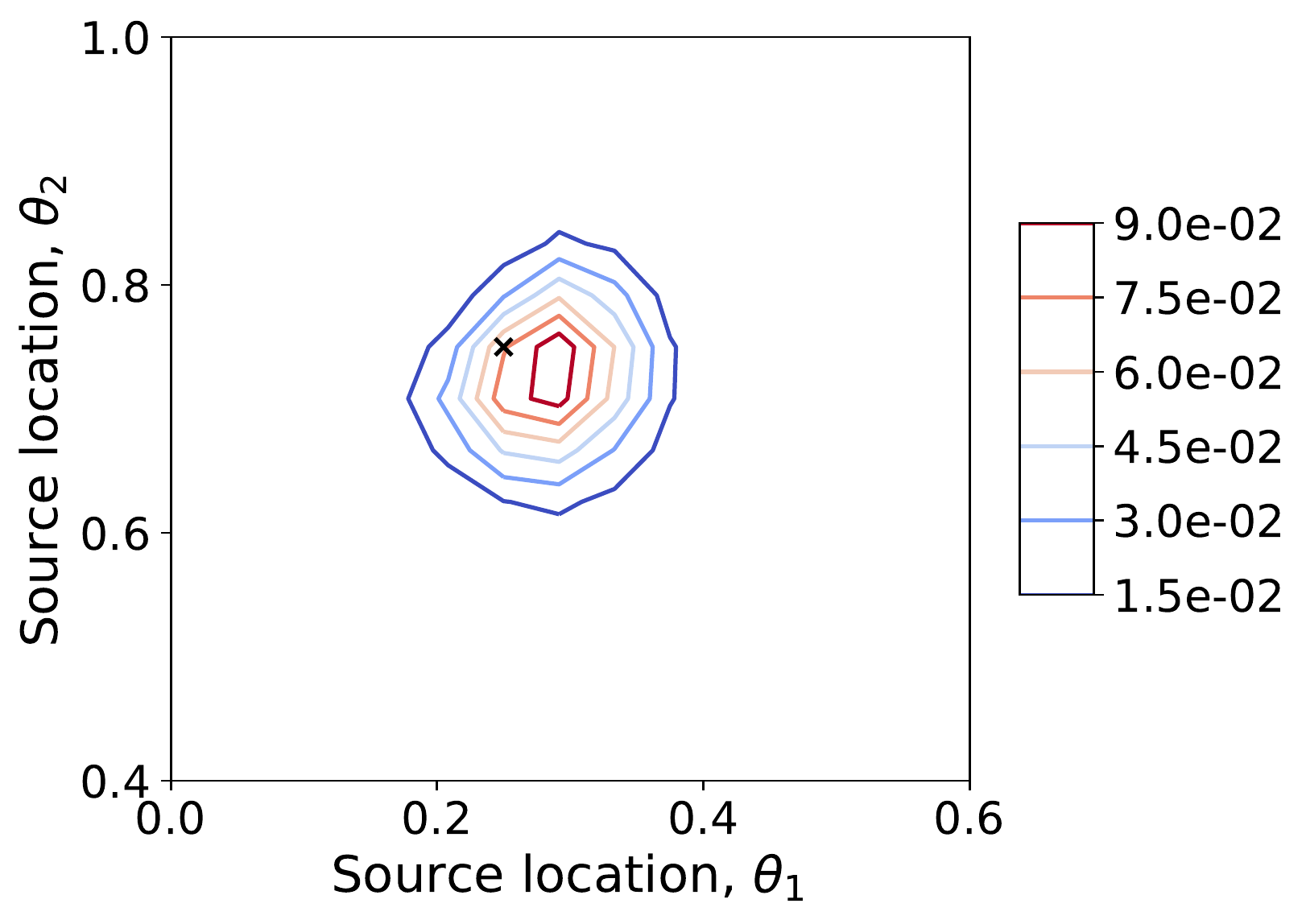}
	\vspace{-1.4\baselineskip}
        \caption{\small{Full model.}}
        \label{fig:heat_lik_true}
    \end{subfigure}%
   ~
    \begin{subfigure}[t]{0.45\textwidth}
        \centering
        \includegraphics[width=\textwidth]{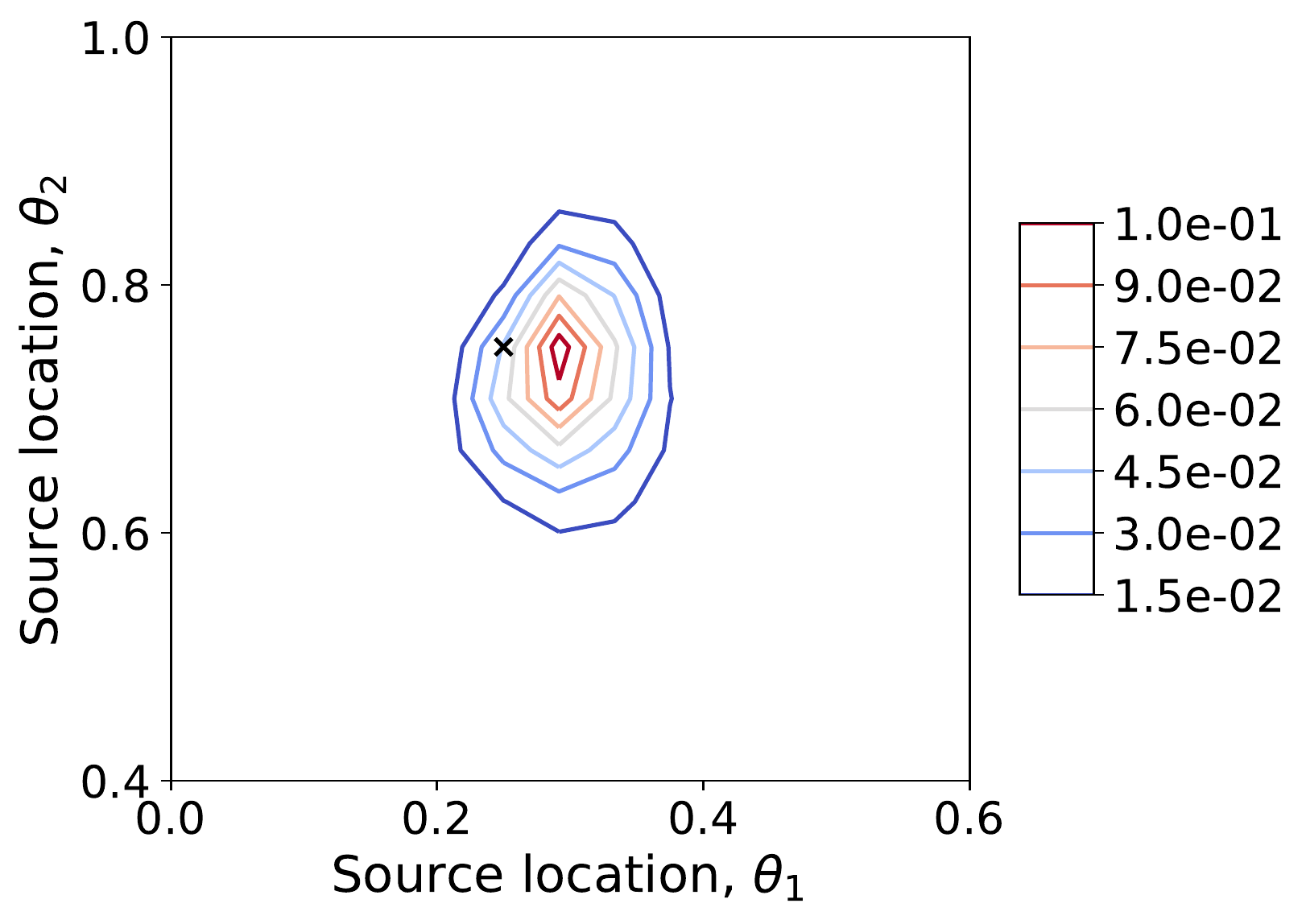}
	\vspace{-1.4\baselineskip}
        \caption{\small{GP (adaptive).}}
        \label{fig:heat_lik_gp_adaptive}
    \end{subfigure}%
    \\~
    \begin{subfigure}[t]{0.45\textwidth}
        \centering
        \includegraphics[width=\textwidth]{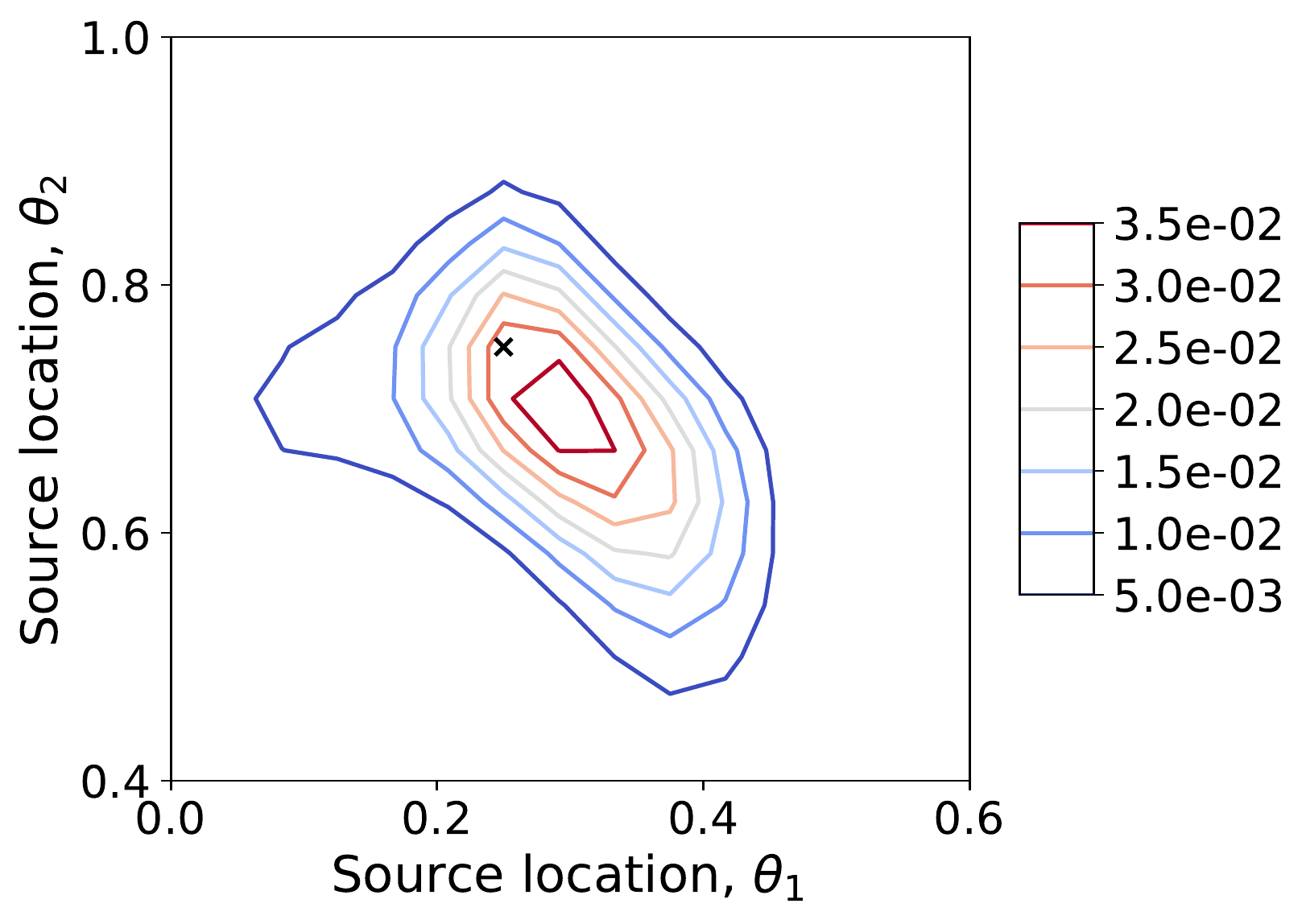}
	\vspace{-1.4\baselineskip}
        \caption{\small{GP (fixed).}}
        \label{fig:heat_lik_gp_fixed}
    \end{subfigure}%
\caption{Contours of normalized likelihoods constructed with the full model \protect\subref{fig:heat_lik_true}, the GP model with adaptive design \subref{fig:heat_lik_gp_adaptive}, and the GP model with a fixed design \subref{fig:heat_lik_gp_fixed}. Black cross indicates $\btheta_{true}$.}
\label{fig:heat_likelihoods}
\end{figure}

\begin{figure}[!htb]
    \centering
    \begin{subfigure}[t]{0.45\textwidth}
        \centering
        \includegraphics[width=\textwidth]{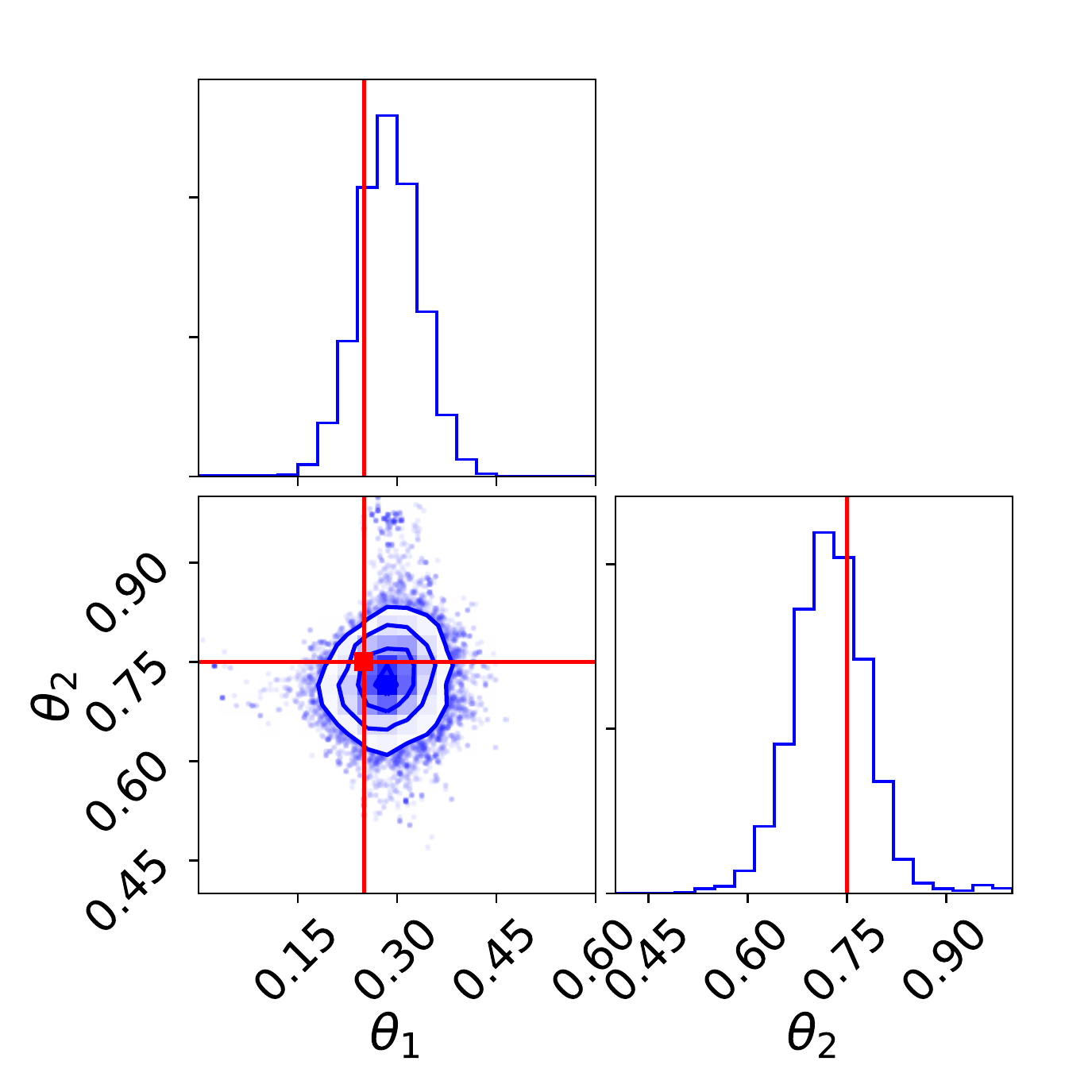}
	\vspace{-1.4\baselineskip}
        \caption{\small{Full model.}}
        \label{fig:heat_post_true}
    \end{subfigure}%
   ~
    \begin{subfigure}[t]{0.45\textwidth}
        \centering
        \includegraphics[width=\textwidth]{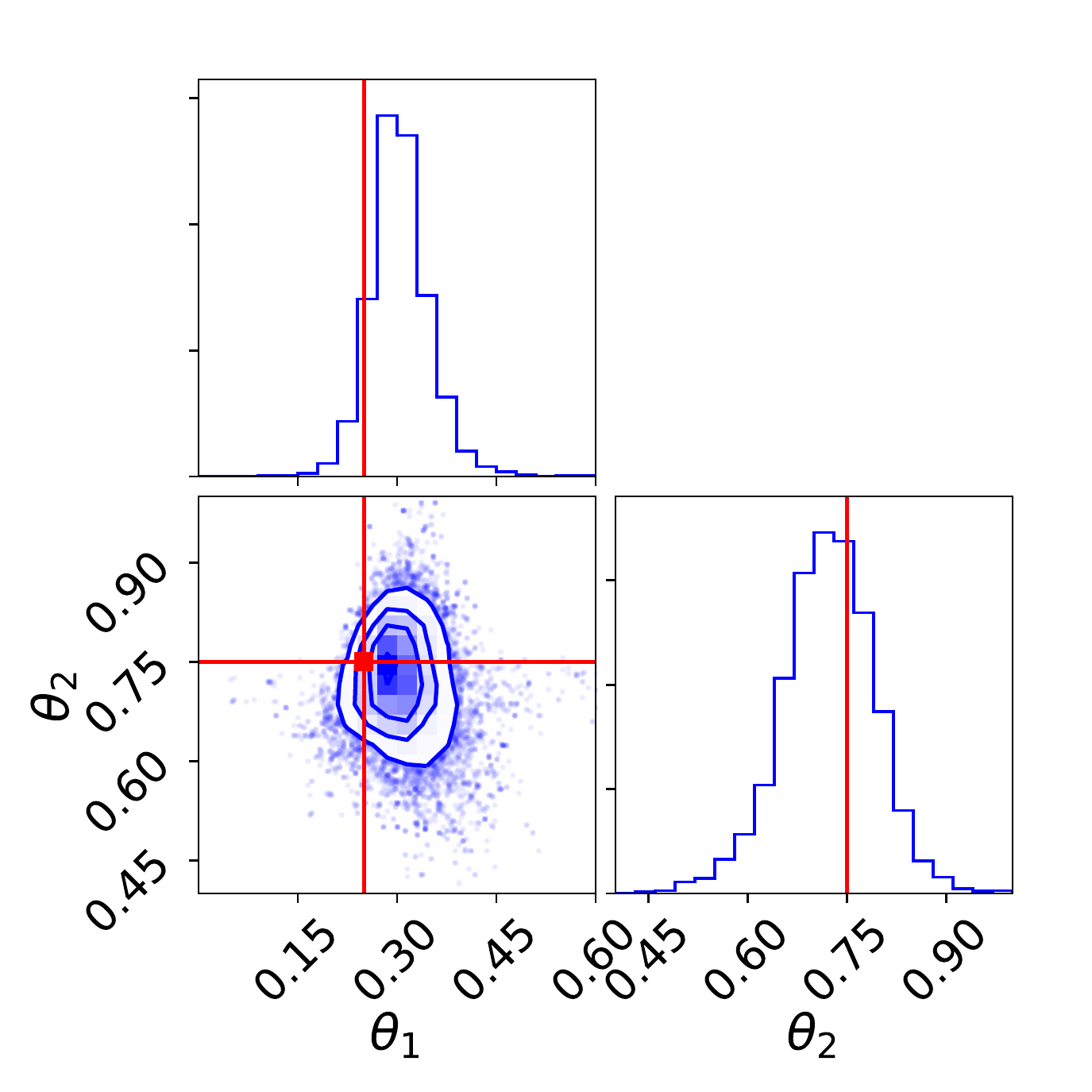}
	\vspace{-1.4\baselineskip}
        \caption{\small{GP (adaptive).}}
        \label{fig:heat_post_gp_adaptive}
    \end{subfigure}%
    \\
    \begin{subfigure}[t]{0.45\textwidth}
        \centering
        \includegraphics[width=\textwidth]{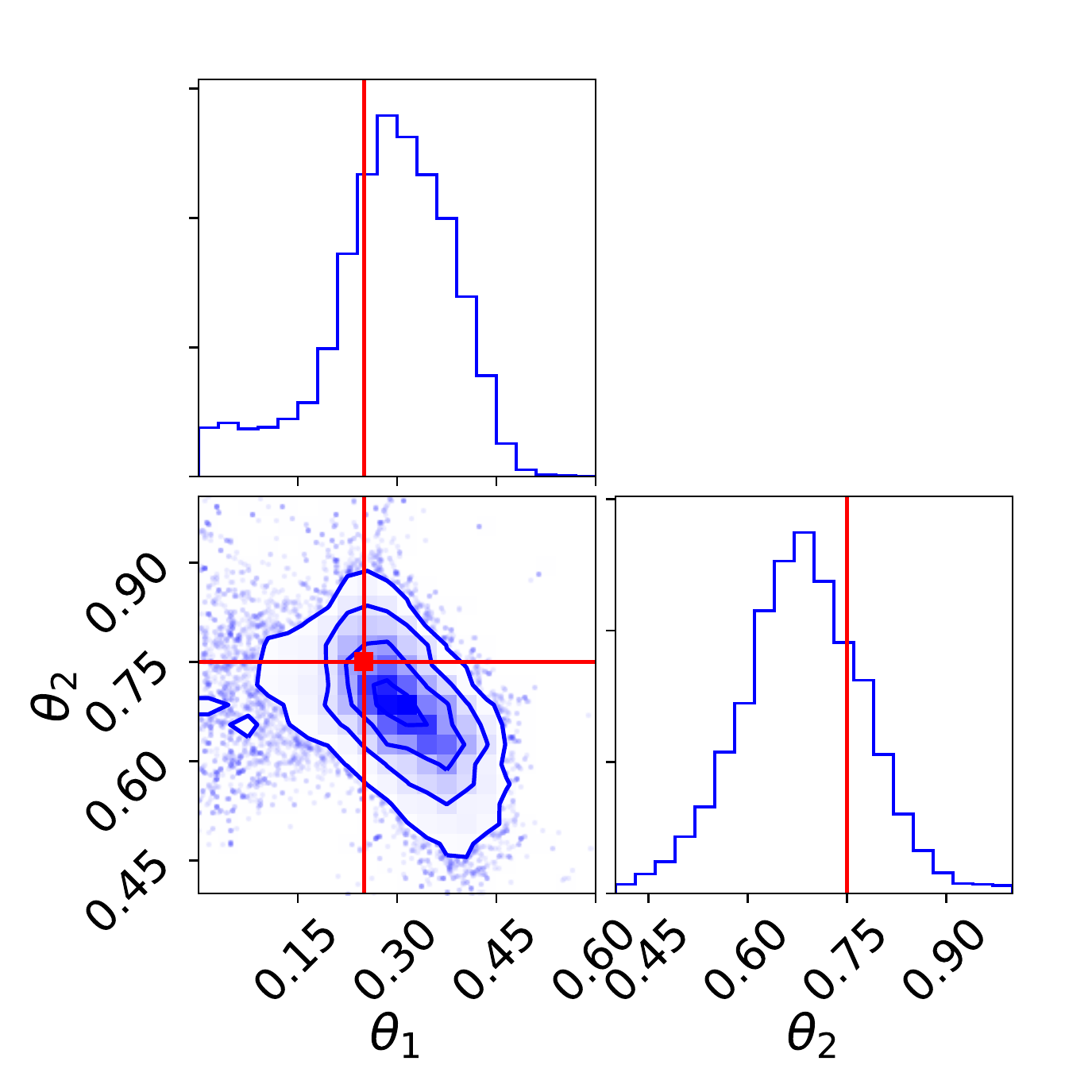}
	\vspace{-1.4\baselineskip}
        \caption{\small{GP (fixed).}}
        \label{fig:heat_post_gp_fixed}
    \end{subfigure}%
\caption{Corner plots of the posteriors estimated with the full model \protect\subref{fig:heat_lik_true}, the GP model with adaptive design \subref{fig:heat_lik_gp_adaptive}, and the GP model with a fixed design \subref{fig:heat_lik_gp_fixed} (each based on $2\times 10^4$ posterior samples). Red lines indicate the location of $\btheta_{true}$.}
\label{fig:heat_posteriors}
\end{figure}

Figure \ref{fig:heat_likelihoods} shows the contours of the normalized likelihoods---each likelihood function is evaluated on a grid of $625$ equidistant points in $\Omega$ and its values are divided by their sum. Figure \ref{fig:heat_lik_gp_adaptive} shows the $\cD$-restricted likelihood obtained with the adaptively constructed design in Figure \ref{fig:heat_design_final}. It appears to be very similar to the ``true'' likelihood in Figure \ref{fig:heat_lik_true}. By contrast, the $\cD$-restricted likelihood in Figure \ref{fig:heat_lik_gp_fixed}  that is based on the Latin hypercube design in Figure \ref{fig:heat_design_LH} deviates from the truth considerably and covers a larger region of the parameter space.

Posteriors estimated with each likelihood function in Figure \ref{fig:heat_likelihoods} are shown in Figure \ref{fig:heat_posteriors}. These figures are generated with $2\times 10^4$ posterior samples obtained with the MCMC sampler initialized using a uniform prior $p(\btheta)=\cU([0,1]^2)$ (see \ref{sec:implementation} for details).
%and likelihoods shown in Figures \ref{fig:heat_likelihoods}.
% by running \textit{emcee} with $100$ walkers for $400$ steps and discarding the first $200$ samples of each chain. The walkers are initialized by random draws from the uniform prior $p(\btheta)=\cU([0,1]^2)$. 
The red lines in each figure show the location of $\btheta_{true}$. We observe that the posterior obtained with the adaptively built GP model, see Figure \ref{fig:heat_post_gp_adaptive}, gives good estimates of the several important characteristics of the true posterior shown in Figure \ref{fig:heat_post_true}, such as its mode, its highest posterior density region (discussed below), and one-dimensional marginals. This cannot be said about the posterior obtained with  the GP model based on the non-adaptive fixed design shown in Figure \ref{fig:heat_post_gp_fixed}.

As a summary statistic to compare the quality of the obtained posteriors, we use the Highest Posterior Density (HPD) region defined as follows.

\begin{definition}\label{def:hpd_region}
A $100(1-\alpha)\%$ HPD region for $\btheta$ is a subset $\cH_\theta\subset\cB_\theta$ defined by
    $\cH_\theta = \{\btheta \in \cB_{\theta} \,:\, p(\btheta \,|\, \bz) \geq t\}$,
where $t$ is the largest number such that $\int\nolimits_{\ensuremath{\boldsymbol{\theta}} \,:\, p(\ensuremath{\boldsymbol{\theta}} | \bz) \geq t} p(\btheta \,|\, \bz) d\btheta = 1 - \alpha$.
\end{definition}

% \fbox{\parbox{0.9\textwidth}{
% 	A $100(1-\alpha)\%$ HPD region for $\btheta$ is a subset $\cH_\theta\in\cB_\theta$ {\color{blue}{``$\subset$'' instead of ``$\in$''?}} defined by
%     $\cH_\theta = \{\btheta \in \cB_{\theta} \,:\, p(\btheta \,|\, \bz) \geq t\}$,
% where $t$ is the largest number such that $\int\nolimits_{\ensuremath{\boldsymbol{\theta}} \,:\, p(\ensuremath{\boldsymbol{\theta}} | \bz) \geq t} p(\btheta \,|\, \bz) d\btheta = 1 - \alpha$. {\color{blue}{can we have this in sone sort of "definition environment" rather than a box?}}
% }}
% \newline

In short, HPD is the smallest region enclosing $(1-\alpha)\%$ of the posterior mass and is a form of Bayesian credibility region. Note that other commonly used summary statistics, such as K-L divergence, might not be appropriate in our case. Since we cannot guarantee that the support of the GP-based posterior will include that of the ``true'' posterior, or vice versa, K-L divergence would not be well-defined.

\noindent
For the posteriors in Figures \ref{fig:heat_post_true} and \ref{fig:heat_post_gp_adaptive}, the $95\%$ HPD regions are respectively:
\begin{align*}
	\cH_\theta &= [0.19, 0.38]\times[0.61, 0.83] \quad \text{ for the full model,} \\
    \cH_\theta &= [0.22, 0.39]\times[0.57, 0.83] \quad \text{ for the adaptive GP model.}
\end{align*}
Note, that in the case of $\cD$-restricted posterior we substitute $p(\btheta | \bz)$ with $p(\btheta | \bz, \cD)$ in  Definition \ref{def:hpd_region}.
For the fixed-design GP posterior in Figure \ref{fig:heat_post_gp_fixed} the $95\%$ HPD region is
\[
	\cH_\theta = [0.04, 0.45]\times[0.47, 0.86] \quad \text{ for the fixed GP model.}
\]

\begin{figure}[!htb]
    \centering
    \begin{subfigure}[t]{0.45\textwidth}
        \centering
        \includegraphics[width=\textwidth]{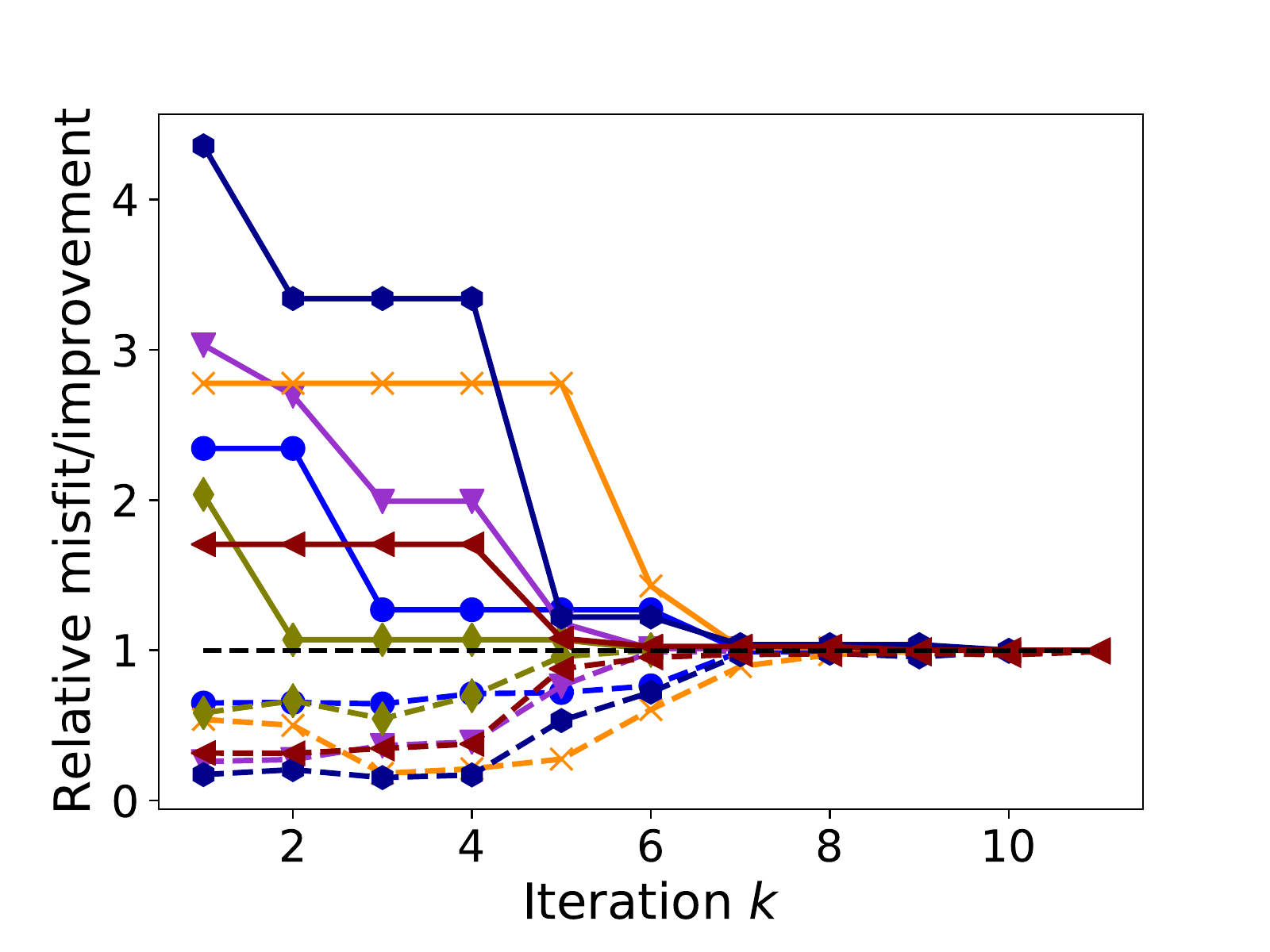}
	\vspace{-1.4\baselineskip}
        \caption{\small{Cases that terminated with $k<n_{max}$.}}
        \label{fig:heat_misfits_converged}
    \end{subfigure}%
   ~
    \begin{subfigure}[t]{0.45\textwidth}
        \centering
        \includegraphics[width=\textwidth]{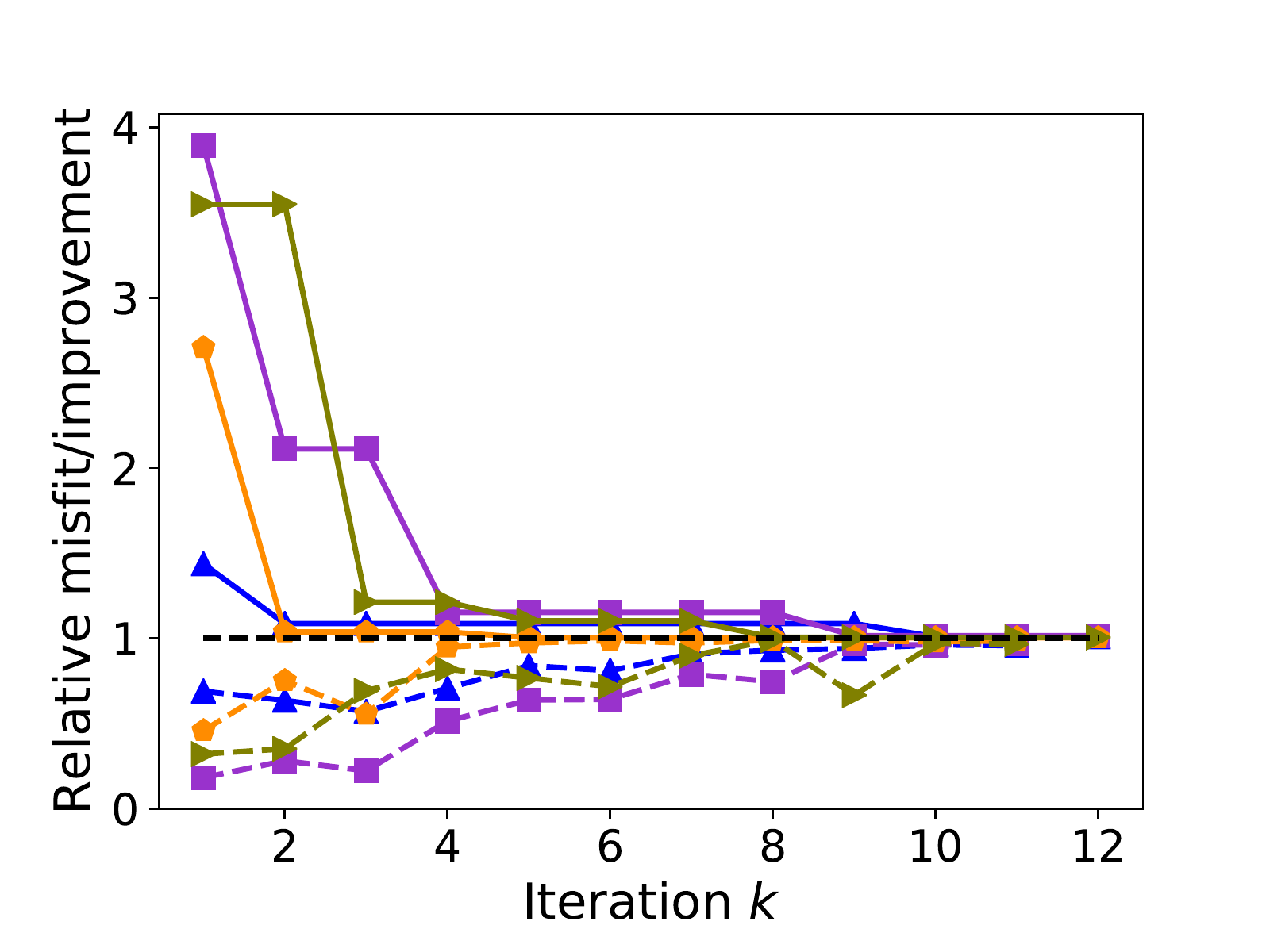}
	\vspace{-1.4\baselineskip}
        \caption{\small{Cases that terminated with $k=n_{max}$.}}
        \label{fig:heat_misfits_not_converged}
    \end{subfigure}%
\caption{Iteration histories for $10$ runs of Algorithm \ref{algo:adaptive_GP} for the Problem \eqref{eq:heat}. \protect\subref{fig:heat_misfits_converged} In $6$ cases, the algorithm terminated before reaching $n_{max}$ iterations.  \protect\subref{fig:heat_misfits_not_converged} In $4$ cases, the  maximum number of iterations was reached. Solid lines correspond to $g_{min}^k/g_{min}^*$, and dashed lines correspond to $1-\cI(\btheta^{(k)})/g_{min}^k$. Here $g_{min}^*=15.015$.}
\label{fig:heat_misfits_multiple_runs}
\end{figure}

The results reported above are based on a single run of  Algorithm \ref{algo:adaptive_GP}. Since the algorithm is stochastic due to the randomness in the initial design $\cD$, multiple runs are required to draw meaningful conclusions. Figure \ref{fig:heat_misfits_multiple_runs} reports the iteration histories for  $10$ random starts of Algorithm \ref{algo:adaptive_GP}, each with $4$ initial inputs arranged in randomized Latin hypercube designs. As before, we set the total allowed number of iterations $n_{max}=11$ so that the final designs have at most $15$ inputs. Figure \ref{fig:heat_misfits_converged} shows the iteration histories for the cases that satisfied the termination condition $\cI(\btheta^{(k)}) \leq \epsilon_{thresh}\cdot g_{min}$ without exceeding $n_{max}$ iterations. The behavior of the algorithm in these cases is similar to that in Figure \ref{fig:heat_misfits}: in the first few iterations, the relative expected improvement is between $35$-$85\%$, and the $g_{min}$ value either stays the same or is reduced slowly---the algorithm is in the exploration stage. A small value of $g_{min}$ is achieved by iterations $5$-$6$ at which point the relative expected improvement drops to less than $30\%$, and the final iterations are spent on further reducing the uncertainty in the model by exploiting the found $g_{min}$ value. In the cases for which the maximum number of iterations $n_{max}$ was reached before the relative expected improvement value reached  the specified threshold $\epsilon_{thresh}=0.01$, see Figure \ref{fig:heat_misfits_not_converged}, the behavior of the algorithm is slightly different: the small $g_{min}$ value is found relatively fast by iterations $3$-$4$, and the remaining iterations are spent on reducing the uncertainty in the GP model. Typically, the inputs added to the training set during these iterations lie on the boundaries of the parameter domain $\cB_\theta$. The relative expected improvement is slowly reduced and by the last iteration is very close to the threshold value---about $2$-$4\%$ instead of the desired $1\%$---which suggests that only a few more iterations would be required to satisfy the desired threshold condition. A summary of the results for all $10$ random starts is given in Table \ref{tbl:results_heat}.

\begin{table}[!htb]
\caption{\small{Results from 10 runs of Algorithm \ref{algo:adaptive_GP} for the Problem \eqref{eq:heat}.}}
\label{tbl:results_heat}
\centering
\begin{tabular}{ c | c | c | c | c }
Run & Final $n_{train}$ & Final $g_{min}^k$ & Final $\cI(\btheta^{(k)})/g_{min}^k$ & Less $\epsilon_{thresh}$ \\[0.5ex]
  \hline			
	1 & 11 & 15.016 & $7\times 10^{-3}$ & yes\\
    2 & 10 & 15.177 & $6\times 10^{-3}$ & yes\\
    3 & 15 & 15.135 & $4\times 10^{-2}$ & no\\
    4 & 12 & 15.036 & $6\times 10^{-3}$ & yes\\
    5 & 9 & 15.114 & 0 & yes\\
    6 & 13 & 15.019 & $7\times 10^{-3}$ & yes\\
    7 & 15 & 15.234 & $3\times 10^{-2}$ & no\\
    8 & 15 & 15.067 & $1.4\times 10^{-2}$ & no\\
    9 & 14 & 15.023 & $9\times 10^{-3}$ & yes\\
    10 & 15 & 15.093 & $3\times 10^{-2}$ & no\\
  \hline  
\end{tabular}
\end{table}

\begin{figure}[!htb]
    \centering
    \begin{subfigure}[t]{0.45\textwidth}
        \centering
        \includegraphics[width=\textwidth]{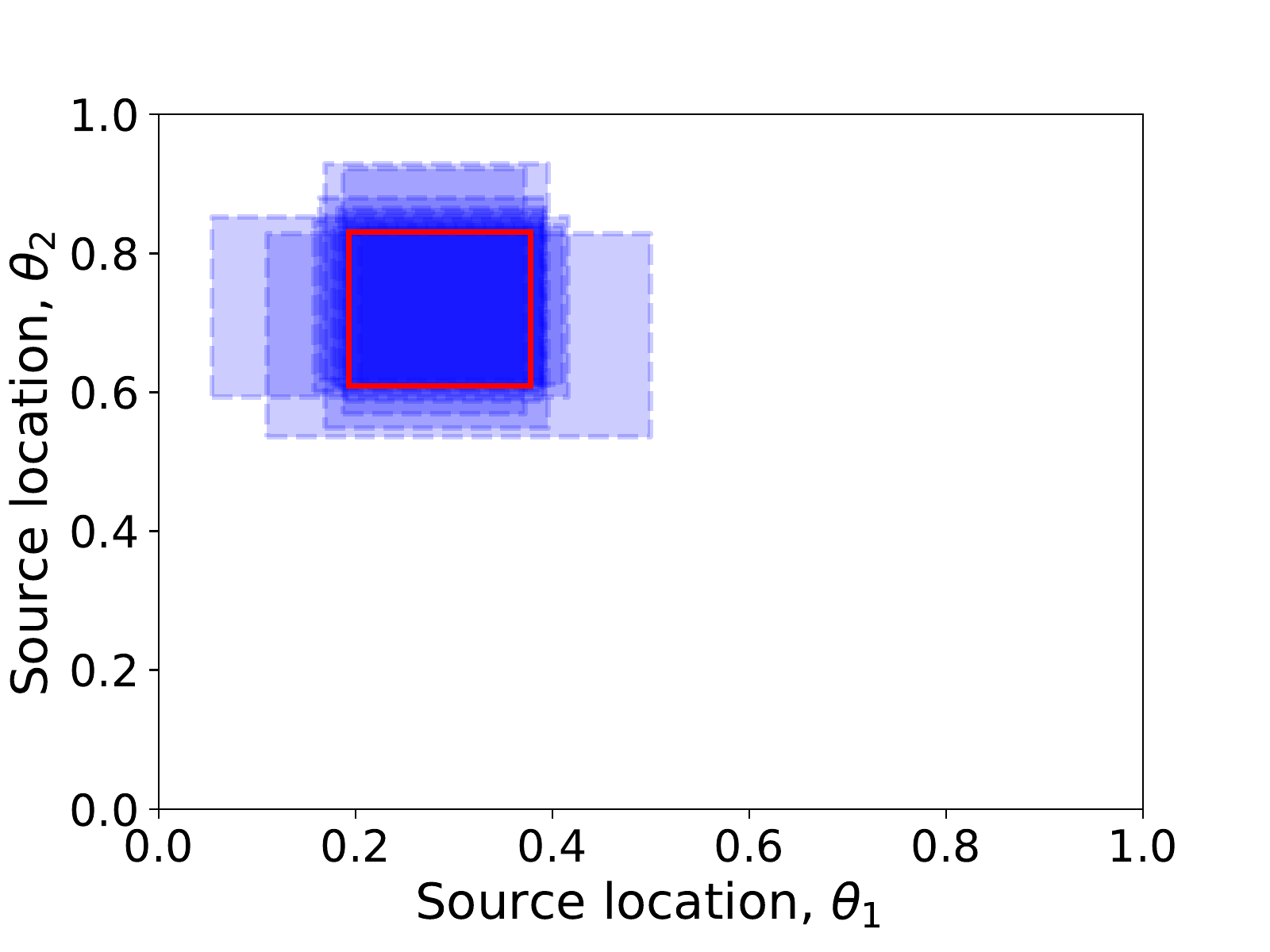}
	\vspace{-1.4\baselineskip}
        \caption{\small{GP adaptive (all cases).}}
        \label{fig:heat_hpd_gp_adaptive}
    \end{subfigure}%
   ~
    \begin{subfigure}[t]{0.45\textwidth}
        \centering
        \includegraphics[width=\textwidth]{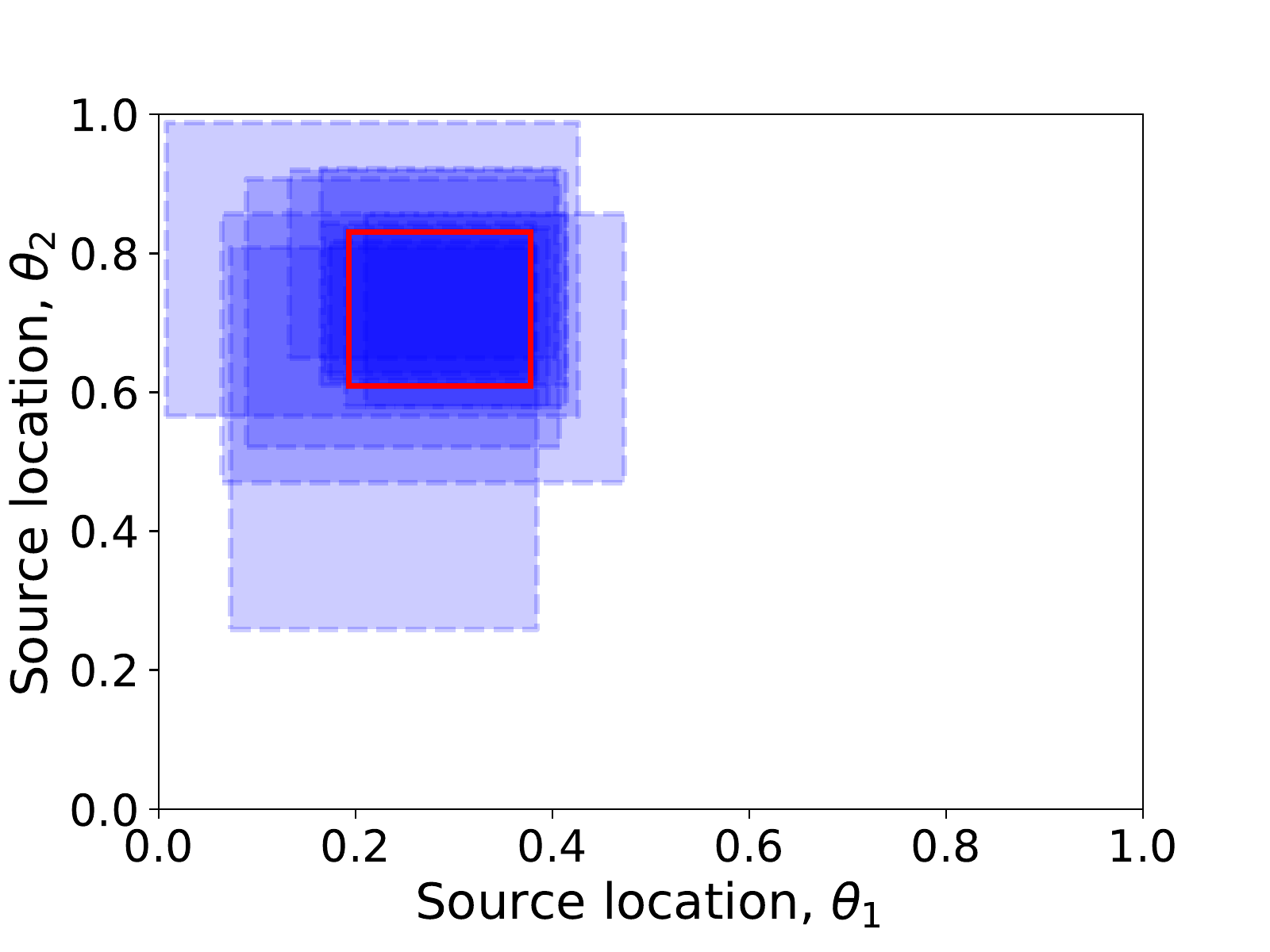}
	\vspace{-1.4\baselineskip}
        \caption{\small{GP fixed.}}
        \label{fig:heat_hpd_gp_LH}
    \end{subfigure}%
    \\
    \begin{subfigure}[t]{0.45\textwidth}
        \centering
        \includegraphics[width=\textwidth]{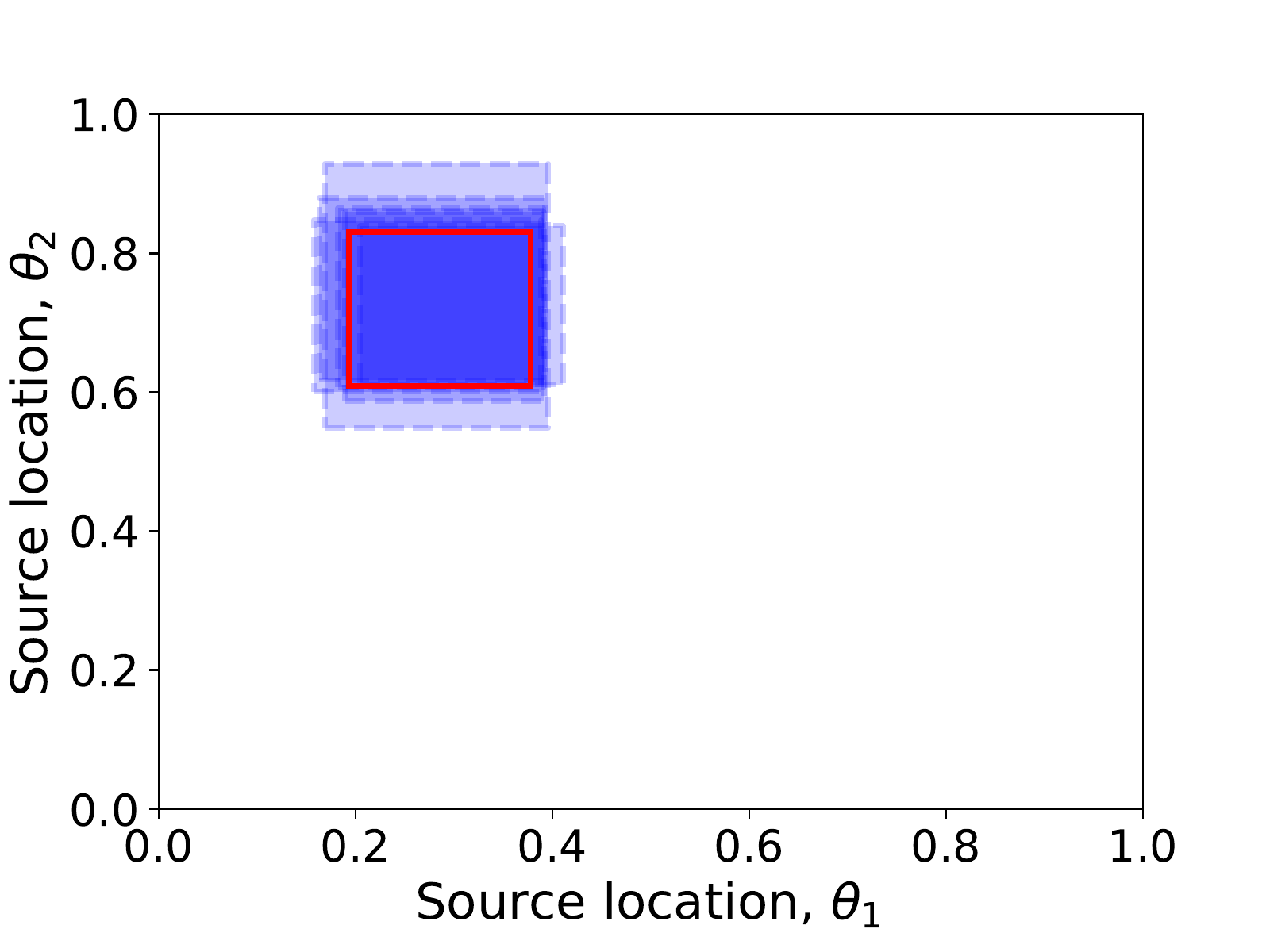}
	\vspace{-1.4\baselineskip}
        \caption{\small{GP adaptive (cases with final $k<n_{max}$).}}
        \label{fig:heat_hpd_gp_adaptive_converged}
    \end{subfigure}%
\caption{The $95\%$ HPD regions for $10$ runs of Algorithm \ref{algo:adaptive_GP} \protect\subref{fig:heat_hpd_gp_adaptive} and for $10$ Latin hypercube designs (each with $15$ points) \protect\subref{fig:heat_hpd_gp_LH} shown as blue rectangles with dashed borders. The $95\%$ HPD region for the true posterior is shown as red rectangle with solid border. \protect\subref{fig:heat_hpd_gp_adaptive_converged} corresponds to the 6 cases from \protect\subref{fig:heat_hpd_gp_adaptive} for which Algorithm \ref{algo:adaptive_GP} terminated with $k<n_{max}$.}
\end{figure}

We plot the $95\%$ HPD regions for all 10 runs of Algorithm \ref{algo:adaptive_GP} in Figure \ref{fig:heat_hpd_gp_adaptive}. These regions correspond well to the $95\%$ HPD region of the true posterior shown as red rectangle. Discarding the cases that reached the maximum number of iterations produces an even more convincing picture, see Figure \ref{fig:heat_hpd_gp_adaptive_converged}. To contrast our adaptive designs with randomized fixed designs, we also perform the posterior estimation with $10$ randomized Latin hypercube designs each containing $15$ points. As Figure \ref{fig:heat_hpd_gp_LH} demonstrates, the resulting posteriors vary  significantly and, in general, are of much worse quality than those obtained with the adaptive designs. Note further that the average number of inputs in the final designs (see second column of Table \ref{tbl:results_heat}) for all adaptive cases is $12.9$, and for the cases with final $k<n_{max}$ depicted in Figure \ref{fig:heat_hpd_gp_adaptive_converged} it is $11.5$. Thus, we achieve consistently better results with adaptive designs than with fixed designs at $76$-$86\%$ of the  cost (measured in the number of forward model evaluations). For our target applications (e.g., cosmology simulations), the reduction in the number of forward model evaluations without sacrificing the accuracy of the inference is critical due to the high computational demands of simulations.
% {\color{blue} This is significant why? -- make a note as to how a reduction of evaluations and generation of better results is particularly important to simulations that take huge amounts of time, like a million CPU hours -- for these cases, it makes a significant difference if you evaluate 10 times or 30 times}

%
%!TEX root = mainAdaptiveGP.tex

\section{Discussion and conclusions}\label{sec:conclusions}

We presented  a novel approach to the adaptive construction of Gaussian process surrogates for the solution of inverse problems in the Bayesian framework. Our approach builds upon the Bayesian surrogate framework of \cite{IBilionis_NZabaras_2014a} and utilizes the expected improvement acquisition function from Bayesian optimization \cite{DRJones_MSchonlau_WJWelch_1998a} in order to sequentially and adaptively select training inputs. We optimize the \textit{expected improvement in fit} function that takes into account measurement noise as well as uncertainty of the GP surrogate. At each step of the algorithm, we add its maximizers to the training set and re-evaluate the parameters of the surrogate model. In this way, we build a hierarchical Bayesian model that adjusts to the obtained simulation data.

The low-dimensional numerical examples demonstrate the effectiveness of our method compared to a fixed design approach based on Latin hypercube sampling. While a formal analysis remains a difficult task, our empirical results show that our adaptive method achieves significantly better results than the non-adaptive method in terms of estimating the parameter posteriors and the computational  cost that is associated with  the posterior estimation using the forward model. Results for a problem with $9$ parameters that can be found in the supplement \ref{sec:permeability} suggest that the method is effective in a higher-dimensional setting as well.

% {\color{blue} A recent work \cite{MSinsbeck_WNowak_2017a} also considers a sequential design strategy for the solution of inverse problems with GP emulators. In this work, however, the true model is assumed to be a realization of the GP model which is known completely (i.e., hyperparameters of the covariance function are fixed). -- can this go into the introduction?} {\color{red} Our setting is much more general, in that we build a hierarchical Bayesian model that adjusts to the obtained simulation data. Furthermore, our criterion for selecting design points is much more tractable than the one considered by the authors. -- and can this (most of it, slight reformulation) go into the previous paragraph}

% {\color{blue}Additional numerical experiments on a higher-dimensional problem are reported in the supplementary material. The results suggest that the method is effective in higher-dimensional setting as well. --this feels a bit out of place. any chance you can stick it somewhere into the previous section? maybe between section 5 and subsection 5.1 -- put a short note of what's to come in this section (implementation details, and the 2  examples, and then make the reference to the additional 9-d result that can be found in the supplement. }

Some comments on the limitations of the methodology are in order. Firstly, the method of Algorithm \ref{algo:adaptive_GP} has a ``potential for deception'' as it relies on the estimates of the prediction error of the Gaussian process model which might considerably underestimate the true error at untested inputs. As noted by Jones \cite[Section~7]{DRJones_2001a}, in some clinical cases, the expected improvement criterion might completely fail if the training sample misleads the construction of the GP surrogate. Furthermore, the greedy and myopic strategy of Algorithm \ref{algo:adaptive_GP} means that if the misfit function has multiple minima with about the same value, only one of them will most likely be explored. The first problem is somewhat unavoidable, but unlikely in practice. The second problem, however,  can be alleviated by modeling. One could  further restrict the search region $\cB_\theta$ or one could tweak the algorithm parameters, for example, by adding several maximizers of the expected improvement in fit function in each iteration and by decreasing the threshold parameter of the stopping criterion. Note, however, that once the $g_{min}$ value is sufficiently close to the global minimum $g_{min}^*$ of the true misfit function, the expected improvement function becomes mostly zero with sharp peaks (see also Figure \ref{fig:1D_exp_imp_step9}). Therefore, it becomes increasingly difficult to find its maxima and  a reduction of the threshold parameter might have no effect.

Further analysis and development of the method are a matter of future work. Several potential extensions of the presented method are: adaptation to a case of correlated outputs using the ideas in \cite{SConti_AOHagan_2010a}; selection of multiple new training inputs at a time in cases when the forward model evaluations can be efficiently performed in parallel; and incorporation of several levels of fidelity of the simulation code as in the autoregressive setting of \cite{MCKennedy_AOHagan_2000a}.

\begin{appendix}
%!TEX root = mainAdaptiveGP.tex

\section{Implementation details} \label{sec:implementation}

\textbf{GP training.} We specify an ``uninformative'' prior on the hyperparameters $\bpsi$ by only specifying their ranges, i.e., we assume uniform priors on both the variance parameter $\sigma_c^2$ and on the characteristic length-scales $\ell_i$, and take their product as the prior $p(\bpsi)$:
\[
	p(\bpsi) = \cU(10^{-8},\sigma_c^U) \times \prod_{i=1}^p \cU(10^{-8}, \ell_i^U),
\]
where the choice of the upper bounds $\sigma_c^U$, $\ell_i^U$ depends on the problem at hand.
The posterior $p(\bpsi | \cD)$ is then obtained by MCMC methods. We utilize the Python library \textit{gptools}  \cite{MAChilenski_MGreenwald_YMMarzouk_NTHoward_AEWhite_JERice_JRWalk_2015a} that uses the the affine-invariant ensemble sampler \cite{JGoodman_JWeare_2010a} known as \textit{emcee} \cite{DForemanMackey_DWHogg_DLang_JGoodman_2013a}. In order to perform sampling, we initialize an ensemble of walkers (typically, $200$) using the prior $p(\bpsi)$, run the parallelized \textit{emcee} sampler for $400$ steps, and use the final states of each walker's chain as posterior samples $\{\bpsi\}_{j=1}^{n_\psi}$ with $n_\psi$ equal to the number of walkers. 
%The convergence of the chains is assessed visually by looking at the traces. 
Evaluation of the posterior means and variances is also performed in parallel.

% \subsection{Solving the expected improvement in fit problem \eqref{eq:expected_improvement_max_problem}}

\textbf{Solving \eqref{eq:expected_improvement_max_problem}.} 
Our strategy for solving \eqref{eq:expected_improvement_max_problem} is to first apply smoothing to the positive part function $[\,\cdot\,]^+$, and then to apply a gradient-based optimization method to find its maxima. Since the expected improvement function $\cI(\btheta)$ is multi-modal, we employ a multi-start strategy to find its multiple local maxima, and we choose the best one as our solution.

For convenience and to make the following derivations simpler, we assume a diagonal noise covariance $\BSigma_E$. 
%In the special case of a diagonal measurement noise matrix $\BSigma_E$, we have
This allows us to write
\[
	g(\btheta; \cD, \bpsi) = \sum_{i=1}^q g(\btheta; \cD_i, \bpsi),
\]
where $g(\btheta; \cD_i, \bpsi)$ is a misfit for the $i$-th measurement (recall \eqref{eq:gp_vec_pred_mean} and \eqref{eq:gp_vec_pred_cov}):
\[
	g(\btheta; \cD_i, \bpsi) = \frac{(z_i - {\V_i}^{1/2}m(\btheta; \cD_i, \bpsi) - m_i)^2}{\sigma_i^2 + \V_i\V(\btheta; \cD, \bpsi)}.
\]

In order to use gradient-based algorithms for solving  \eqref{eq:expected_improvement_max_problem}, we use a smoothed positive part function $[\,\cdot\,]^+_\eta$ that depends on a smoothing parameter $\eta$. Specifically, we use the following twice continuously differentiable function from \cite{DPKouri_TMSurowiec_2016a}:
\[
	[x]^+_\eta = \begin{cases}
		0, \quad &\text{if } x \leq 0,\\
		\left(\frac{x^3}{\eta^2} - \frac{x^4}{2\eta^3}\right), \quad &\text{if } x\in(0,\eta),\\
		x - \frac{\eta}{2}, \quad &\text{if } x\geq \eta.
\end{cases}
\] 
For this function $[x]^+_\eta \leq [x]^+ \leq [x]^+_\eta + 0.5\eta$. We set $\eta=10^{-4}$ and in the following treat $[x]^+_\eta$ as a function of $x$ only. With $[\,\cdot\,]^+$ substituted by $[\,\cdot\,]^+_{\eta}$, the problem \eqref{eq:expected_improvement_max_problem} is substituted by
\begin{equation}\label{eq:expected_improvement_max_smooth_problem}
	\max_{\boldsymbol{\theta}\in\cB_\theta} \cI_{\eta}(\btheta) \coloneqq \frac{1}{n_\psi} \sum_{j=1}^{n_\psi} \left[ g_{min} - g\big(\btheta; \cD, \bpsi^{(j)}\big)\right]^+_\eta.
\end{equation}
The gradient of the objective $\cI_\eta(\btheta)$ can be computed as follows:
\[
	\nabla_{\boldsymbol{\theta}} \cI_\eta(\btheta) = -\frac{1}{n_\psi} \sum_{j=1}^{n_\psi} \left(\left[ g_{min} - g\big(\btheta; \cD, \bpsi^{(j)}\big) \right]^{+}_\eta \right)' \nabla_{\boldsymbol{\theta}} g\big(\btheta; \cD, \bpsi^{(j)}\big),
\]
where the gradient of the misfit function is given by
\begin{align*}
	\nabla_{\boldsymbol{\theta}} g(\btheta; \cD, \bpsi) = -\sum_{i=1}^q &\bigg[ \frac{2{\V_i}^{1/2}(z_i - {\V_i}^{1/2}m(\btheta; \cD_i, \bpsi) - m_i)}{\sigma_i^2 + \V_i\V(\btheta; \cD, \bpsi)} \nabla_{\boldsymbol{\theta}} m(\btheta; \cD_i, \bpsi) \\
	&+ \V\nolimits_i\frac{(z_i - {\V_i}^{1/2}m(\btheta; \cD_i, \bpsi) - m_i)^2}{(\sigma_i^2 + \V\nolimits_i\V(\btheta; \cD, \bpsi))^2} \nabla_{\boldsymbol{\theta}}\V(\btheta; \cD, \bpsi) \bigg].
\end{align*}
Recall that the predictive mean for the $i$-th output has the form
\[
	m(\btheta; \cD_i, \bpsi) = \bc_\psi^T(\BC_\psi)^{-1} \by_i = \sum_{j=1}^{n_{train}} c\big(\btheta, \btheta^{(j)}_{train} ; \bpsi \big) v_j^{(i)}
\]
with $(v_j^{(i)}, \dots, v_{n_{train}}^{(i)})^T = \bv_i \coloneqq (\BC_\psi)^{-1}\by_i$. The gradient of the covariance between $\btheta$ and a point $\btheta^{(j)}_{train}$ in the training set is
\[
	\nabla_{\boldsymbol{\theta}} c\big(\btheta, \btheta^{(j)}_{train}; \bpsi\big) = -\BLambda^{-1}\big(\btheta - \btheta^{(j)}_{train}\big) c\big(\btheta, \btheta^{(j)}_{train}; \bpsi\big),
\]
where $\BLambda \coloneqq \text{diag}[\ell_1^2, \dots, \ell_p^2]$. Thus, the gradient of the predictive mean with respect to $\btheta$ is
\[
	\nabla_{\boldsymbol{\theta}} m(\btheta; \cD_i, \bpsi) = - \sum_{j=1}^{n_{train}} \BLambda^{-1} \big(\btheta - \btheta^{(j)}_{train}\big) c\big(\btheta, \btheta^{(j)}_{train}; \bpsi\big) v_j^{(i)},
\]
or in a more compact form
\[
	\nabla_{\boldsymbol{\theta}} m(\btheta; \cD_i, \bpsi) = -\BLambda^{-1}(\btheta - \btheta_{train}) (\bc_\psi \cdot \bv_i),
\]
where $\btheta - \btheta_{train} = \bigg[\btheta - \btheta^{(1)}_{train}, \dots, \btheta - \btheta^{(n_{train})}_{train} \bigg] \in \real^{p\times n_{train}}$, and $\ba\cdot\bb$ means element-wise product of the vectors $\ba$ and $\bb$.

\noindent
For the predictive variance
\[
	\V(\btheta; \cD, \bpsi) = c(\btheta, \btheta; \bpsi) - \bc_\psi^T(\BC_\psi)^{-1} \bc_\psi = \sigma_c^2 - \bc_\psi^T(\BC_\psi)^{-1} \bc_\psi,
\]
we get
\[
	\nabla_{\boldsymbol{\theta}} \V(\btheta; \cD, \bpsi) = 2\BLambda^{-1}(\btheta-\btheta_{train})(\bc_\psi \cdot (\BC_\psi)^{-1} \bc_\psi ).
\]
% Further details, such as the choice of the optimization method and its initialization, are discussed in Section \ref{sec:implementation}.
Next we discuss the choice of the optimization method for solving \eqref{eq:expected_improvement_max_smooth_problem} and the choice of initial points.

We utilize the truncated Newton method \cite{SGNash_1984a} with bound constraints. Its implementation is available through the Python function \textit{scipy.optimize.fmin\_tnc}. The bounding box $\cB_{\theta}$ corresponds to the domain of the uniform prior on the parameters $\btheta$. We initialize the solver from multiple initial locations chosen in $\cB_\theta$ either on a grid, or according to a quasi-random design (Sobol sequence \cite[Section~5.6.4]{TJSantner_BJWilliams_WINotz_2003a}). The number of initializations is dictated by the dimensionality of the parameter space and by computational time considerations. Since evaluation of the objective function is already performed in parallel (with respect to the hyperparameter samples), we perform the optimizations sequentially for the multiple starts. As convergence criteria we use the norm of the projected gradient, the absolute difference in the consecutive function values, and the norm of the difference in the consecutive iterate values. From the set of converged results we select the one corresponding to the highest optimal objective value as the solution of \eqref{eq:expected_improvement_max_smooth_problem}.

\textbf{Posterior estimation.} In the numerical examples estimation of the posterior is also performed with \textit{emcee}. It requires providing the log-probability function that is a product of the log-prior and the log-likelihood.
For the ``true'' likelihood $L(\btheta | \bz)$, the log-likelihood (assuming diagonal noise covariance $\BSigma_{E}$) is computed as follows:
\[
	\log L(\btheta | \bz) = -\frac{1}{2} \sum_{i=1}^q \bigg[ \frac{(z_i - f_i(\btheta))^2}{2\pi \sigma_i^2} + \log(2\pi \sigma_i^2) \bigg].
\]
For the GP-based likelihood $L(\btheta | \bz, \cD)$, the log-likelihood is approximated as:
\[
	\log L(\btheta | \bz, \cD) \approx \log\bigg( \sum_{j=1}^{n_\psi} \frac{k^{(j)}}{n_\psi} \exp\left[-\frac{1}{2} g\big(\btheta; \cD, \bpsi^{(j)}\big) \right] \bigg)
\]
with $k^{(j)} = \big( \prod_{i=1}^q 2\pi\big(\sigma_i^2 + \V_i\cdot\V\big(\btheta; \cD, \bpsi^{(j)}\big) \big)^{-1/2}$.
In order to avoid an underflow when computing this approximation, we use \textit{scipy.misc.logsumexp} in Python that is based on the following formulation:
\[
	\log L(\btheta | \bz, \cD) \approx -\frac{1}{2}g^* + \log\bigg( \sum_{j=1}^{n_\psi} \frac{k^{(j)}}{n_\psi} \exp\left[-\frac{1}{2} \big(g\big(\btheta; \cD, \bpsi^{(j)}\big) - g^*\big) \right] \bigg),
\]
where
\[
	g^* =\min\{g\big(\btheta; \cD, \bpsi^{(j)}\big) \,|\, j=1,\dots,n_{train}\}.
\]
Finally, as previously mentioned, the prior $p(\btheta)$ is taken to be uniform $\cU(\btheta^L, \btheta^U)$.
The posterior plots are generated using the Python library \textit{corner} \cite{DForemanMackey_2016a}.
%
%!TEX root = mainAdaptiveGP.tex

\section{Inversion of permeability field}\label{sec:permeability}
We use a test problem motivated by steady flow in porous media considered in \cite{CuiT_YMMarzouk_KEWillcox_2015a}. The governing equations are given by
\begin{subequations}\label{eq:laplace}
\begin{align}
	- \nabla\cdot(\kappa(\bx; \btheta) \nabla u(\bx)) &= q(\bx) &\bx\in \Omega\coloneqq[0,1]^2 \\
	\kappa(\bx; \btheta)\nabla u(\bx) \cdot \mathbf{n}(\bx) &= 0 &\bx\in\partial \Omega \\
	\int_\Omega u(\bx) d\bx &= 0.
\end{align}
\end{subequations}
The source term $q(\bx)$ is defined by the mixture of four weighted Gaussians with standard deviations of $0.05$, centered at $(0.3, 0.3)$, $(0.7, 0.3)$, $(0.7, 0.7)$, $(0.3, 0.7)$, and with weights $\{+2, -3, +3, -2\}$. Equations \eqref{eq:laplace} are solved in FEniCS \cite{HPLangtangen_ALogg_2017a} using a $32\times 32$ uniform finite element mesh with piecewise-linear Lagrange elements. 

The permeability field $\kappa(\bx; \btheta)$ is defined as a weighted sum of $p=9$ radial basis functions with the weights being the parameters of interest:
\begin{equation}\label{eq:permeability_field}
	\kappa(\bx; \btheta) = \sum_{i=1}^p \theta_i b_i(\bx),
\end{equation}
where
\[
	b_i(\bx) = \exp\left[ -\frac{\|\bx - \bc_i\|^2}{2(0.15)^2} \right]
\]
with centers $\bc_i$ given by $(0.5, 0.5)$, $(0.25, 0.25)$, $(0.75, 0.25)$, $(0.75, 0.75)$, $(0.25, 0.75)$, $(0, 0.5)$, $(0.5, 0)$, $(1., 0.5)$, $(0.5, 1.)$.

\begin{figure}[!htb]
    \centering
    \begin{subfigure}[t]{0.46\textwidth}
        \centering
        \includegraphics[width=\textwidth]{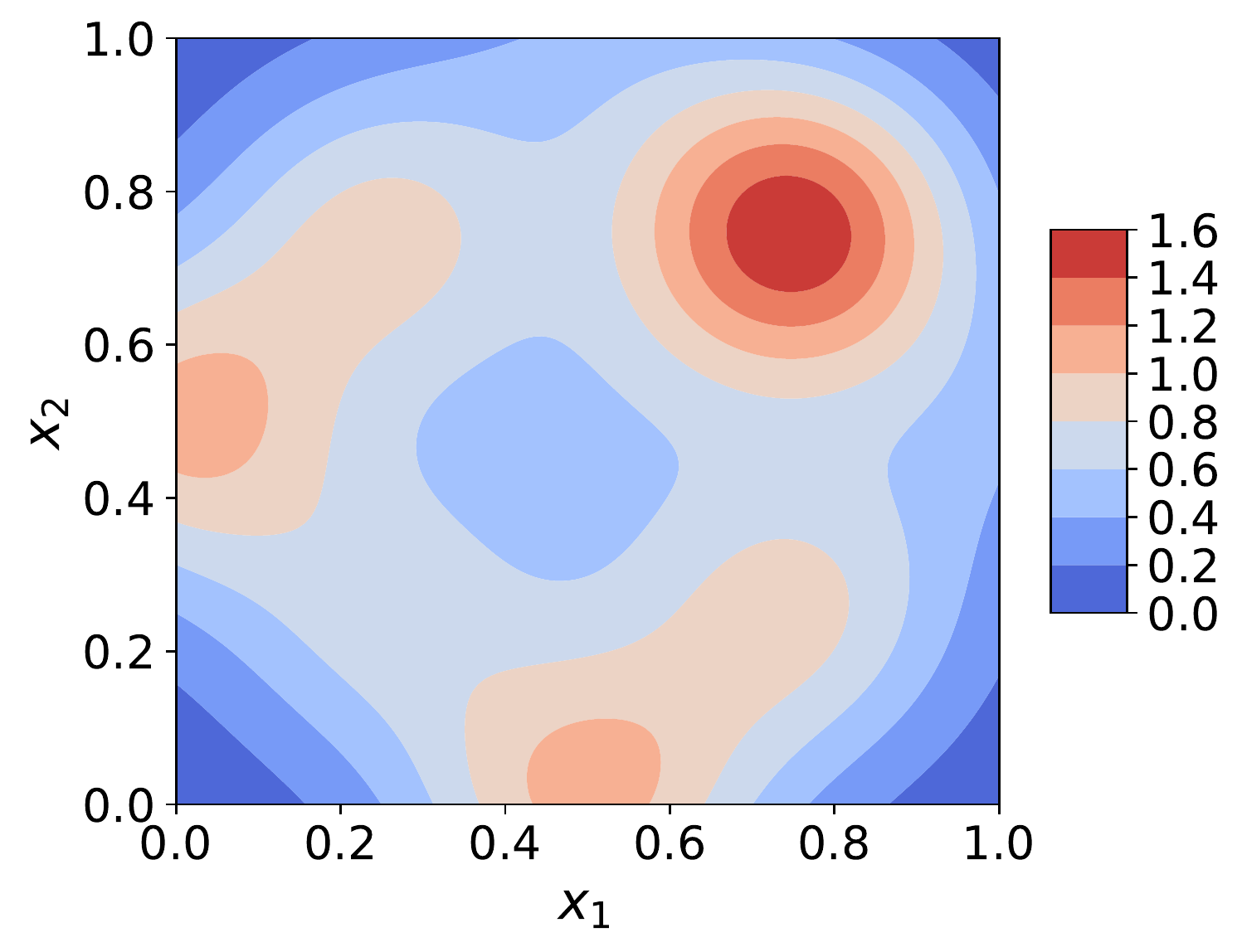}
	\vspace{-1.4\baselineskip}
        \caption{\small{Permeability field $\kappa(\bx; \btheta_{true})$.}}
        \label{fig:perm_field}
    \end{subfigure}%
   ~
    \begin{subfigure}[t]{0.46\textwidth}
        \centering
        \includegraphics[width=\textwidth]{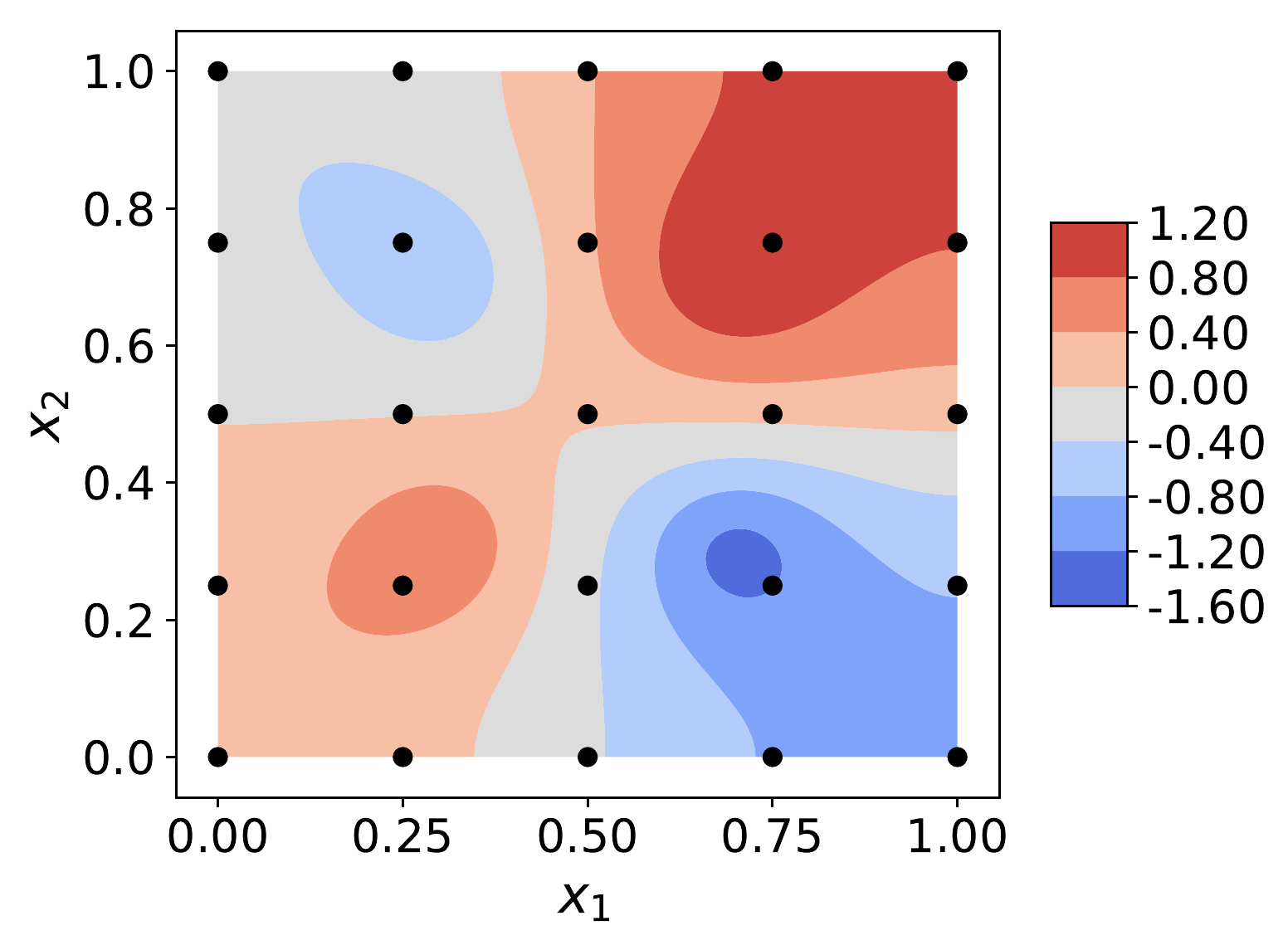}
	\vspace{-1.4\baselineskip}
        \caption{\small{Solution $u(\bx)$ corresponding to $\btheta_{true}$.}}
        \label{fig:perm_solution}
    \end{subfigure}%
\caption{Setup of the Problem \eqref{eq:laplace}. Black dots in \protect\subref{fig:perm_solution} indicate measurement locations.}
\end{figure}

The measurements are taken on a uniform $5\times 5$ grid covering $\Omega$ resulting in a total of $25$ measurements. The measurement data $\bz$ is generated with the following $\btheta_{true}=(0.3, 0.6, 0.8, 1.5, 0.8, 1.0, 1.0, 0.3, 0.3)^T$ by solving \eqref{eq:laplace} on a $128\times 128$ finite element mesh and adding zero-mean measurement noise with $\sigma_i=0.01$, $i=1,\dots, 25$. The permeability field corresponding to $\btheta_{true}$ is shown in Figure \ref{fig:perm_field}, and the corresponding solution $u(\bx)$ of the problem \eqref{eq:laplace} is shown in Figure \ref{fig:perm_solution}.

We assume uniform priors on the parameters $\btheta$, $p(\btheta)=\cU(\cB_\theta)$, where $\cB(\theta) = [0, 1]\times[0, 1]\times[0, 1]\times[0.8, 1.8]\times[0, 1]\times[0.5, 1.5]\times[0.6, 1.6]\times[0, 1]\times[0, 1]$.

The initial design $\cD$ contains $18$ training inputs arranged in a Latin hypercube design. We run Algorithm \ref{algo:adaptive_GP} (described in the main document) with $n_{max}=20$, $\epsilon_{thresh}=0.01$, and we use $500$ $9$-dimensional Sobol sequence points as initial guesses for the multi-start optimization problem \eqref{eq:expected_improvement_max_problem}. We use the hyperparameter prior  $p(\bpsi)=\cU([0,4]^{p+1})$. We obtain $n_{\psi}=200$ hyperparameter posterior samples with \textit{emcee} using the likelihood function \eqref{eq:evidence_mult_out} with normalized outputs \eqref{eq:scaled_outputs_mult_case}.

The iteration history of  Algorithm \ref{algo:adaptive_GP} for the Problem \eqref{eq:laplace} is presented in Figure \ref{fig:perm_misfits}. The reference value $g_{min}^*=12.62$ is computed by minimizing the true misfit function $g(\btheta)$. After $16$ iterations of  Algorithm \ref{algo:adaptive_GP}, the achieved value of $g_{min}$ is $56.09$ which is relatively large compared to $g_{min}^*$. At iteration $17$, however, all $500$ local optimizations return zero expected improvement objective value and the algorithm exits. Without performing additional searches, we use the obtained GP model based on a total of $34$ forward model evaluations to estimate the parameter posteriors. 

\begin{figure}[!htb]
    \centering
    \begin{subfigure}[t]{0.46\textwidth}
        \centering
        \includegraphics[width=\textwidth]{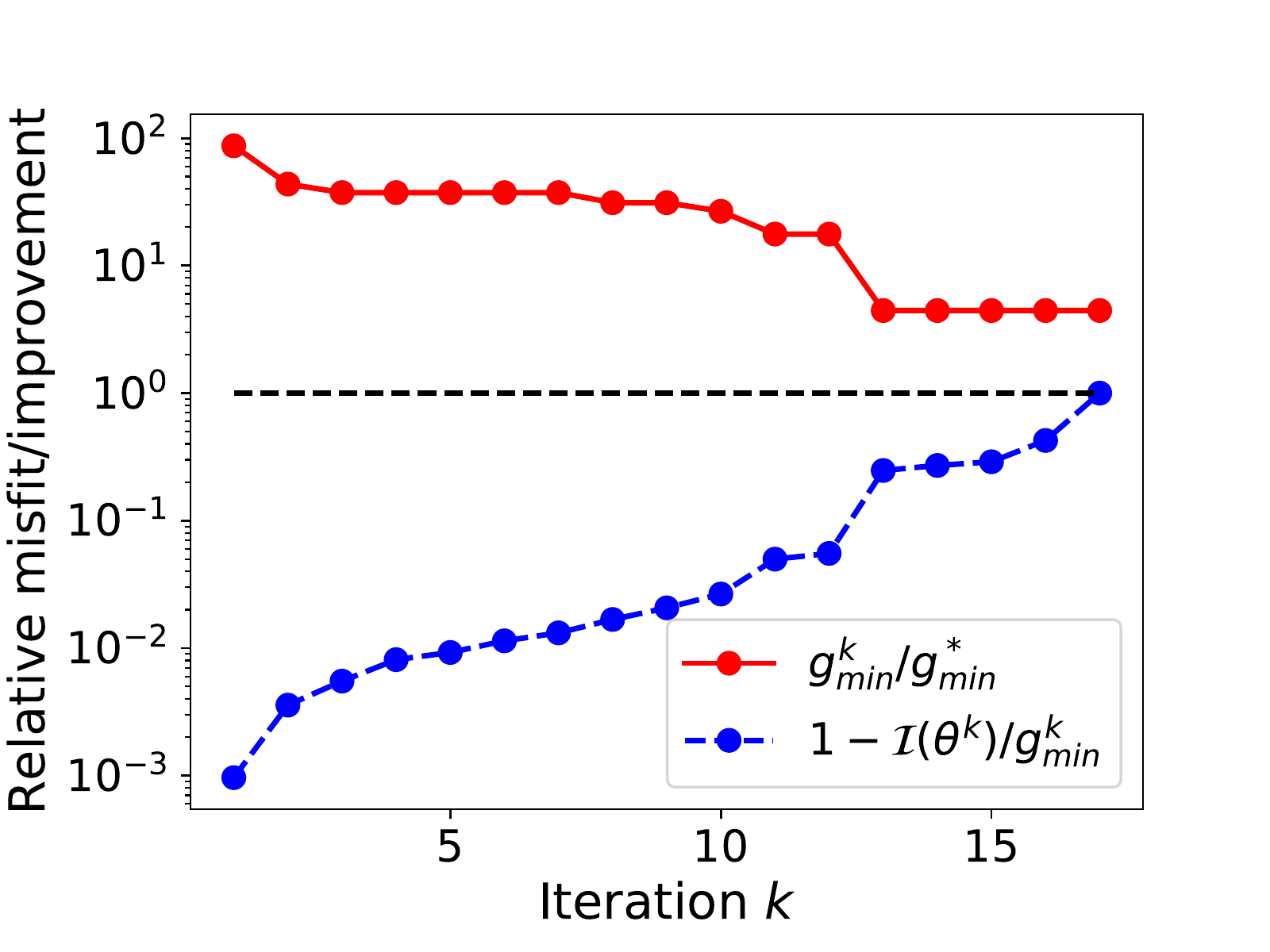}
	\vspace{-1.4\baselineskip}  
    \end{subfigure}%
    \caption{Iteration history of Algorithm \ref{algo:adaptive_GP} for the Problem \eqref{eq:laplace}. Here $g_{min}^*=12.62$.}
    \label{fig:perm_misfits}
\end{figure}

The result of estimating $p(\btheta | \bz)$ with the full forward model is presented in Figure \ref{fig:perm_post_true}, and the result of estimating $p(\btheta | \bz, \cD)$ with the adaptively constructed GP model is shown  in Figure \ref{fig:perm_post_gp_adaptive}. In both cases, the plots are generated with $2\times 10^4$ posterior samples. Even though our algorithm did not find the global minimum of the misfit function $g(\btheta)$, the obtained posterior agrees with the ``true'' one for most of the parameters judging by the two- and one-dimensional marginals. The two slightly misspecified parameters appear to be $\theta_3$ and $\theta_4$.

\noindent
The $95\%$ high probability density (HPD) region for the true posterior $p(\btheta | \bz)$ is given by
\begin{align*}
	\cH_\theta^{true} &= [0.17, 0.48]\times[0.52, 0.66]\times[0.75, 0.88]\times[1.35, 1.53]\times[0.66, 0.83]\\
    &\times[0.89, 1.19]\times[0.86, 1.04]\times[0.22, 0.38]\times[0.25, 0.42].
\end{align*}
The $95\%$ HPD region for the GP-based posterior $p(\btheta | \bz, \cD)$ is given by
\begin{align*}
	\cH_\theta^{GP} &= [0.17, 0.49]\times[0.53, 0.66]\times[0.71, 0.83]\times[1.34, 1.49]\times[0.66, 0.79] \\
 &\times[0.92, 1.21]\times[0.86, 1.03]\times[0.26, 0.39]\times[0.28, 0.42].
\end{align*}

To illustrate the inference results with the full model and the GP model, we plot the recovered permeability fields corresponding to the values of $\btheta$ fixed at the medians of the one-dimensional marginals of the posteriors $p(\btheta | \bz)$ (Figure \ref{fig:perm_field_posterior_median_fm}) and $p(\btheta | \bz, \cD)$ (Figure \ref{fig:perm_field_posterior_median_gp}). As Figures \ref{fig:perm_post_true} and \ref{fig:perm_post_gp_adaptive} suggest, the medians of the one-dimensional marginals provide good estimates of the modes of both posteriors. Thus, the permeability fields in Figures \ref{fig:perm_field_posterior_median_fm} and \ref{fig:perm_field_posterior_median_gp} can be considered as point estimates of $\kappa(\bx; \btheta_{true})$ shown in Figure \ref{fig:perm_field} given the measurements $\bz$.

\begin{figure}[!htb]
    \centering
    \begin{subfigure}[t]{0.46\textwidth}
        \centering
        \includegraphics[width=\textwidth]{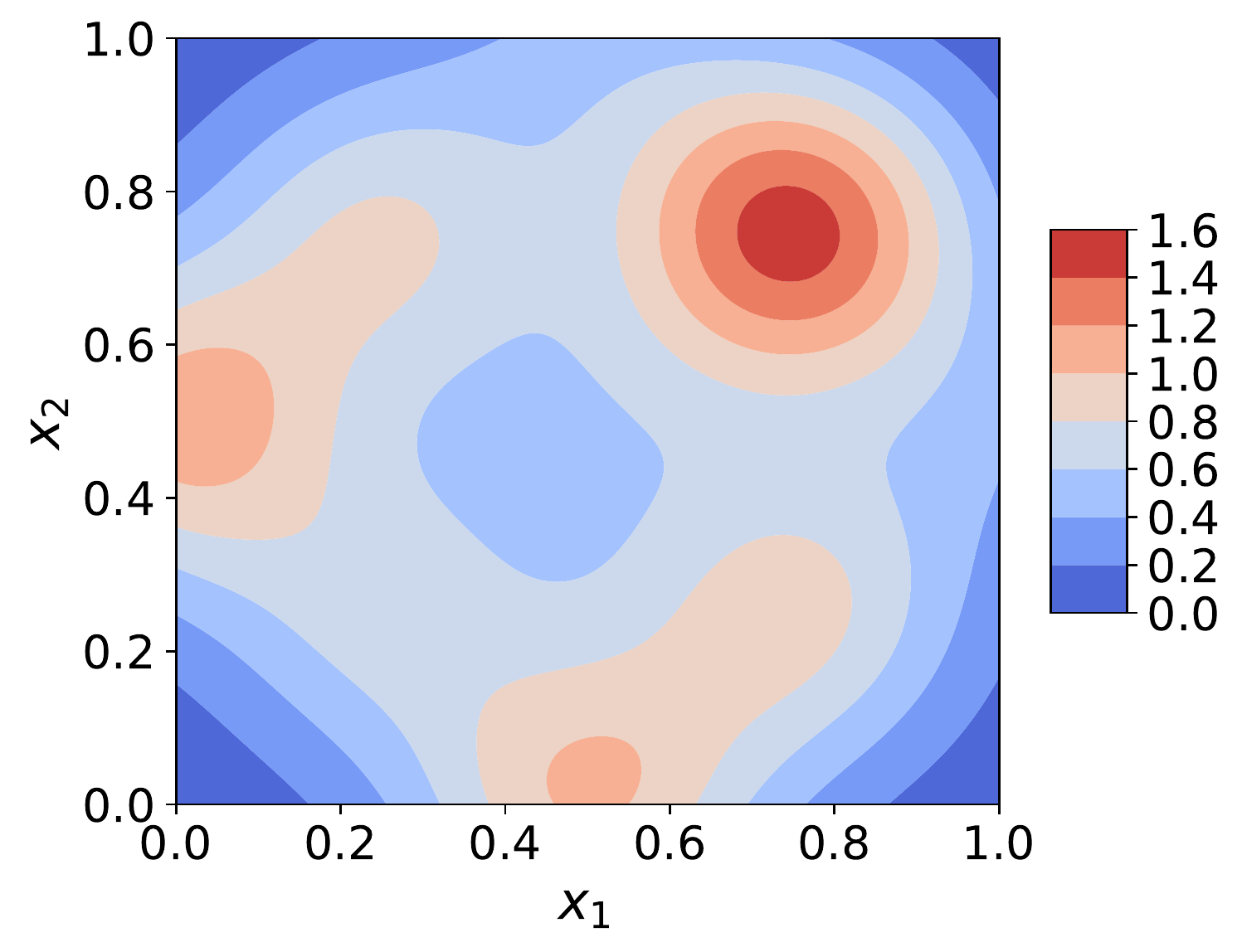}
	\vspace{-1.4\baselineskip}
        \caption{\small{Permeability field $\kappa(\bx; \btheta_{median}^{full})$.}}
        \label{fig:perm_field_posterior_median_fm}
    \end{subfigure}%
   ~
    \begin{subfigure}[t]{0.46\textwidth}
        \centering
        \includegraphics[width=\textwidth]{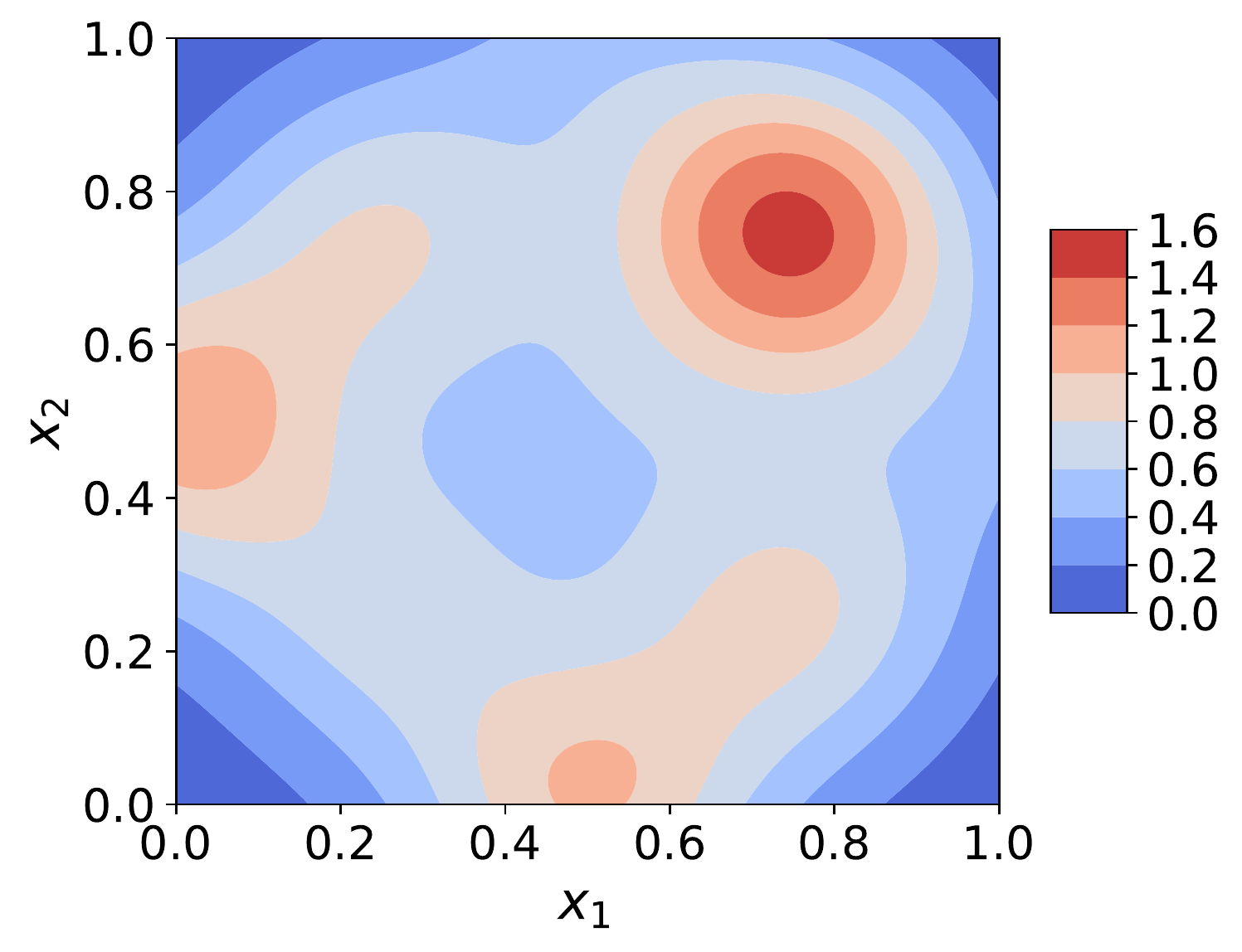}
	\vspace{-1.4\baselineskip}
        \caption{\small{Permeability field $\kappa(\bx; \btheta_{median}^{GP})$.}}
        \label{fig:perm_field_posterior_median_gp}
    \end{subfigure}%
\caption{Recovered permeability fields corresponding to $\btheta$ fixed at the medians of the one-dimensional marginals of the posteriors $p(\btheta | \bz)$ \protect\subref{fig:perm_field_posterior_median_fm} and $p(\btheta | \bz, \cD)$ \protect\subref{fig:perm_field_posterior_median_gp}.}
\end{figure}

\begin{figure}[!htb]
    \centering
    \begin{subfigure}[t]{\textwidth}
        \centering
        \includegraphics[width=\textwidth]{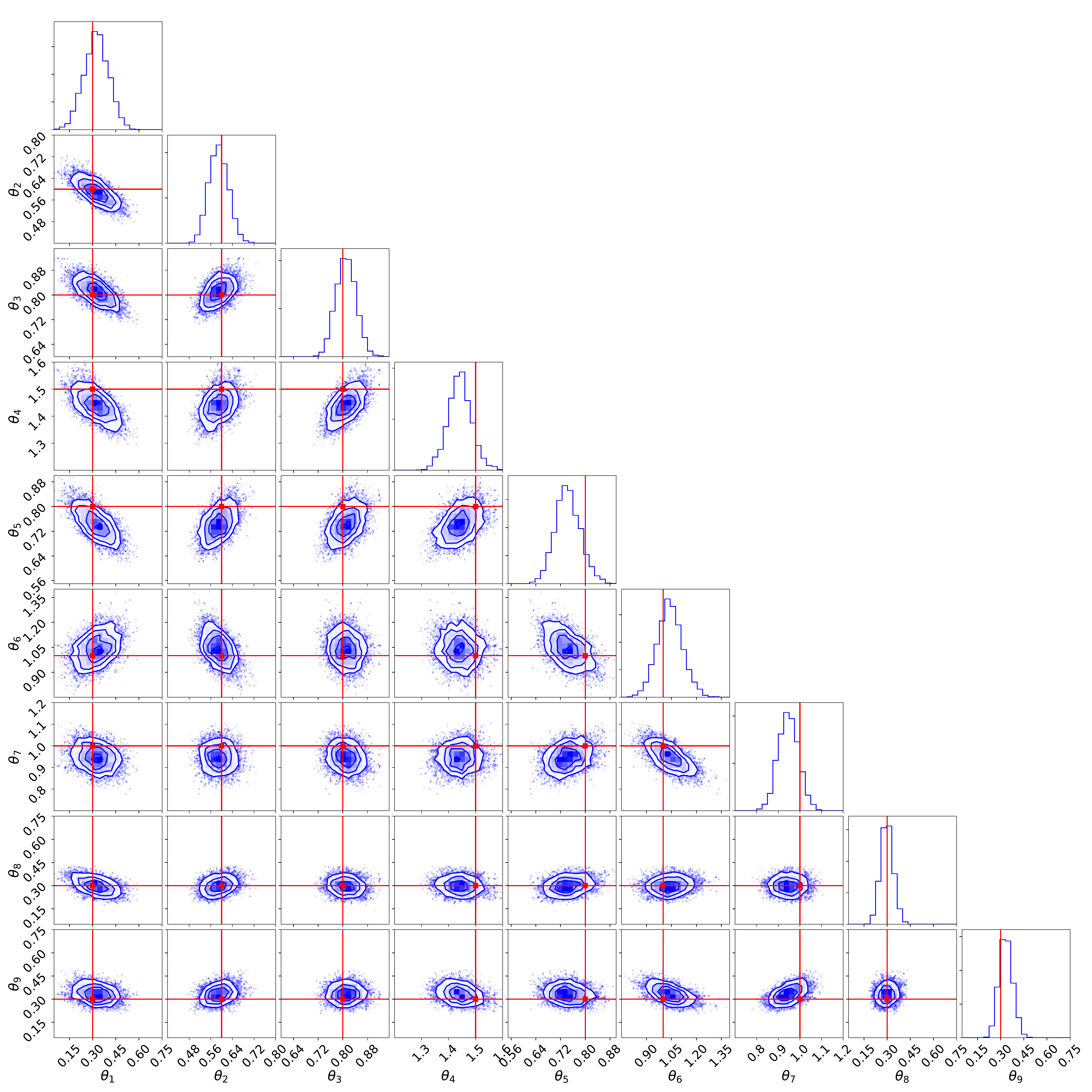}
    \end{subfigure}%
    %\vspace{-1.4\baselineskip}
    \caption{\small{Posterior $p(\btheta | \bz)$ estimated with forward model (based on $2\times 10^4$ samples).}}
        \label{fig:perm_post_true}
\end{figure}

\begin{figure}[!htb]
\centering
    \begin{subfigure}[t]{\textwidth}
        \centering
        \includegraphics[width=\textwidth]{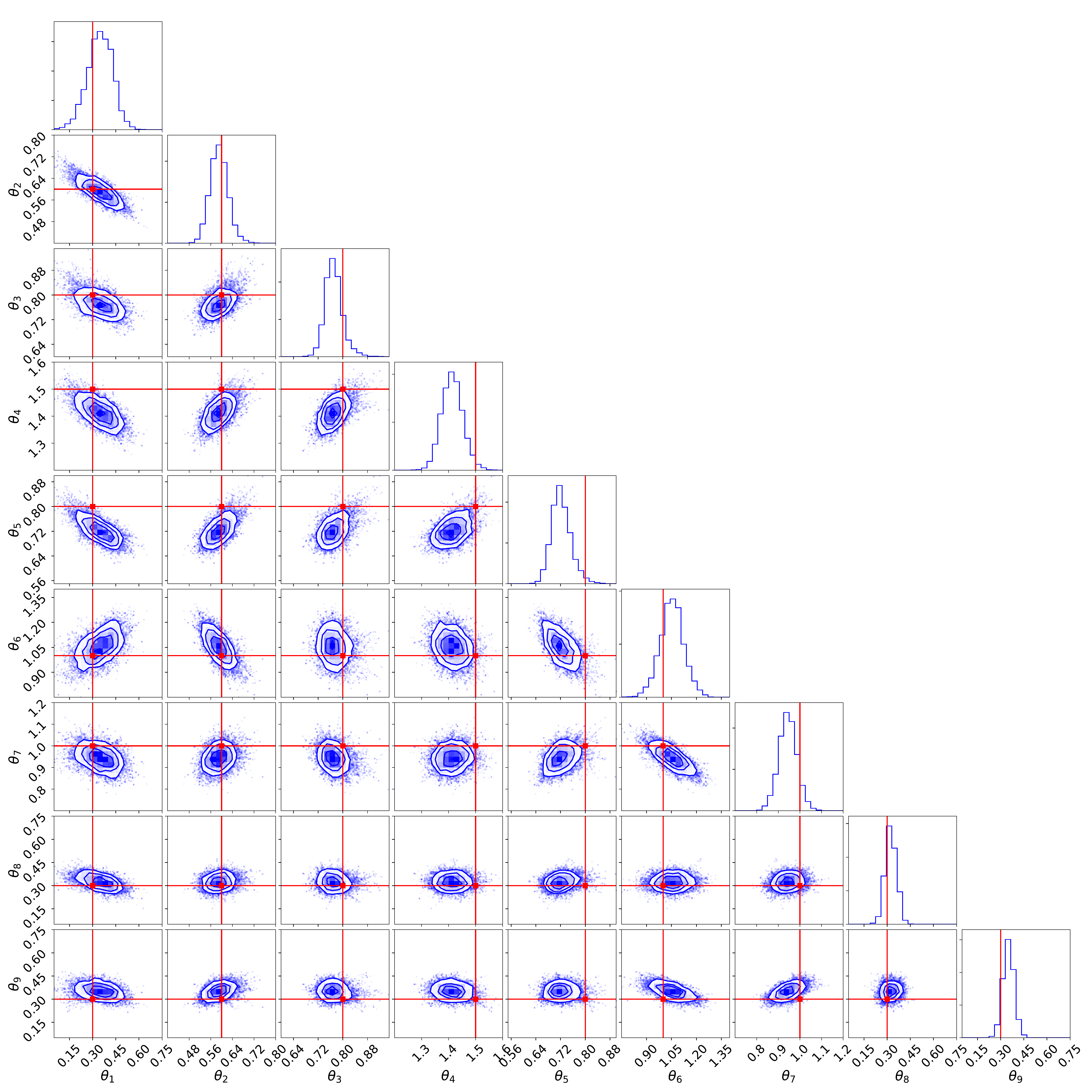}
    \end{subfigure}%
    %\vspace{-1.4\baselineskip}
        \caption{\small{Posterior $p(\btheta | \bz, \cD)$ estimated with adaptive GP model built using a total of $34$ forward model evaluations (based on $2\times 10^4$ samples).}}
        \label{fig:perm_post_gp_adaptive}
\end{figure}

\end{appendix}

\clearpage
\bibliography{references}
\bibliographystyle{plain}

% \newpage
% \begin{appendix}
% \input{appendix_permeability}
% \end{appendix}

\end{document}